\documentclass{article}

\usepackage{PRIMEarxiv}

\usepackage[utf8]{inputenc} % allow utf-8 input
\usepackage[T1]{fontenc}    % use 8-bit T1 fonts
\usepackage{hyperref}       % hyperlinks
\usepackage{url}            % simple URL typesetting
\usepackage{booktabs}       % professional-quality tables
\usepackage{amsfonts}       % blackboard math symbols
\usepackage{nicefrac}       % compact symbols for 1/2, etc.
\usepackage{microtype}      % microtypography
\usepackage{fancyhdr}       % header
\usepackage{graphicx}       % graphics
\graphicspath{{media/}}     % organize your images and other figures under media/ folder

\usepackage{amssymb}
\usepackage[authoryear,longnamesfirst]{natbib}
\usepackage{bbm}
\usepackage{extarrows}
\usepackage{algorithmicx,algorithm}
\usepackage[noend]{algpseudocode}
\usepackage{diagbox}
\usepackage{cleveref,tikz}
\usepackage{amsthm}
\usepackage{color}
\usepackage{subfigure}
\usepackage{amsmath}
\usepackage{xspace}

\def\tsc#1{\csdef{#1}{\textsc{\lowercase{#1}}\xspace}}
\tsc{WGM}
\tsc{QE}

\newtheorem{theorem}{Theorem}
\newtheorem{lemma}[theorem]{Lemma}
\newtheorem{definition}{Definition}
\newtheorem{example}{Example}
\newtheorem{assumption}{Assumption}
\theoremstyle{remark}

\newcommand{\Halmos}{\hfill \ensuremath{\Box}}
\newcommand{\ie}{i.e.,\@\xspace}

\makeatletter
\newenvironment{breakablealgorithm}
{% \begin{breakablealgorithm}
  \refstepcounter{algorithm}% New algorithm
  \hrule height.8pt depth0pt \kern2pt% \@fs@pre for \@fs@ruled
  \renewcommand{\caption}[2][\relax]{% Make a new \caption
    \ifx\relax##1\relax % #1 is \relax
    \addcontentsline{loa}{algorithm}{\protect\numberline{\thealgorithm}##2}%
    \else % #1 is not \relax
    \addcontentsline{loa}{algorithm}{\protect\numberline{\thealgorithm}##1}%
    \fi
    \noindent\textbf{Algorithm~\thealgorithm} ##2\par% Modified line for left alignment
    \kern2pt\hrule\kern2pt
  }
}{% \end{breakablealgorithm}
  \kern2pt\hrule\relax% \@fs@post for \@fs@ruled
}
\makeatother

\title{General Formulation and PCL-Analysis for Restless Bandits with Limited Observability}

\author{Keqin Liu$^{1,*}$ and Qizhen Jia$^{1}$}

\begin{document}
\maketitle

\begin{center}
$^{1}$School of Mathematics and Physics, Xi'an Jiaotong-Liverpool University, Suzhou, 215123, Jiangsu, China.\\
\vspace{1em}
$^{*}$Corresponding author. E-mail(s): Keqin.Liu@xjtlu.edu.cn;
%Contributing authors: Qizhen.Jia21@student.xjtlu.edu.cn; z929663047@gmail.com
\end{center}
\maketitle
\vspace{1.5cm}
\begin{abstract}
In this paper, we consider a general observation model for restless multi-armed bandit problems. The operation of the player is based on the past observation history that is limited (partial) and error-prone due to resource constraints or environmental or intrinsic noises. By establishing a general probabilistic model for dynamics of the observation process, we formulate the problem as a restless bandit with an infinite high-dimensional belief state space. We apply the achievable region method with partial conservation law (PCL) to the infinite-state problem and analyze its indexability and priority index (Whittle index). Finally, we propose an approximation process to transform the problem into which the AG algorithm of Niño-Mora (2001) for finite-state problems can be applied.  Numerical experiments show that our algorithm has excellent performance.
\end{abstract}

% Keywords

\maketitle

% Main text
\section{Introduction}

Multi-armed bandit (MAB) is a classic operations research problem that involves a learner making choices among actions with uncertain/random rewards in order to maximize the expected return in long-run. The MAB problem is often associated with the important exploration-exploitation tradeoff and was initially proposed by \citet{robbins1952some}. Starting from the classic stochastic scheduling problem, the ongoing development of various MAB models makes it applicable to a wide range of practical fields, including clinical trials, recommendation systems, cognitive communications, and financial investments \citep{gittins1979bandit, berry1985bandit, press2009bandit, farias2011irrevocable, hoffman2011portfolio, shen2015portfolio}.

In the classic Bayesian MAB problem, the bandit machine has a total of $N$ arms and one player pulling an arm in each decision epoch. After a player selected one of these arms and activated it, a random reward is accrued depending on the activated arm and its current state. In this process, all states are visible, and only the state of the arm being activated changes according to a Markov chain. The goal of the MAB problem is to search a policy that maximizes the long-term cumulative reward. \citet{gittins1979bandit} proved that the classic MAB problem can be solved optimally by an index policy, referred to as Gittins index, and the player only needs to activate the arm with the largest index at each moment. The index policy significantly reduces the complexity of the problem from being exponential with the number of arms to being linear. After that, \citet{whittle1988restless} extended the MAB problem to the restless MAB (RMAB) model, in which $K$ arms can be activated at each moment, and the arms not chosen may also undergo state transitions over time. However, for the general RMAB problem with a finite state space, \citet{papadimitriou1994complexity} showed that the computational complexity of finding the optimal policy is already PSPACE-hard. Fortunately, \citet{whittle1988restless} generalized Gittins index to Whittle index that provides a possible solution to RMAB by considering Lagrangian relaxation and duality. Accordingly, Whittle index is optimal under a relaxed constraint that the {\em time-average} number of activated arms is~$K$. For the original problem, the Whittle index policy has been proven to be asymptotically optimal per-arm wise as the number of arms goes to infinity under certain conditions \citep{weber1990index,weber1991index}. Nonetheless, not every RMAB is indexable, \ie Whittle index may not exist. And the establishment of indexability is itself a difficult problem. Even if indexability is proven for a particular RMAB, analytical solutions of the Whittle index function can still be hard to obtain. For some important categories of RMAB models, Whittle indexability and its strong performance have been demonstrated in the literature, such as the dual-speed bandit problem and partially observable RMAB \citep{glazebrook2002index, ahmad2009optimality, liu2010indexability, gittins2011multi, larranaga2016rmab, AyestaEtal2021MMOR, LiuEtal24MMOR, fu2025rmab, liu2021index}.

In this paper, we extend the partially observable RMAB considered in \citet{liu2010dynamic,LiuEtal24MMOR,liu2021index} to general observation models. The previous work only considers special classes of observation errors and noises within the framework of partially observable Markov decision processes (POMDP) \citep{zhao2007survey,liu2010dynamic,LiuEtal24MMOR}. But general algorithms for POMDP often suffer from curse of dimensionality and become cumbersome when the value functions for dynamic programming are too complex to be precisely solved or well approximated \citep{smallwood1973optimal,sondik1978optimal}. For special cases, the state transition dynamics yield a simple structure (e.g., threshold type) of the optimal policy such that the value functions can be efficiently evaluated or even solved in closed-form \citep{gittins2011multi,liu2010dynamic,LiuEtal24MMOR,liu2021index}. In order to deal with general observation models, we need an alternative methodology to design efficient algorithms for systems with large sets of parameters. In the interdisciplinary field of operations research and stochastic optimization, an analytical method called ``Achievable Region'' emerged in the 1990s. Interestingly, this method transforms a time-series optimization problem into linear programming (LP) problems that can be solved efficiently \citep{coffman1980characterization,federgruen1988characterization,federgruen1988m,gelenbe2010analysis}. The challenge of using this method is mainly on proving the feasibility of such transformation \citep{bertsimas1995achievable}. Referred to as the general conservation law (GCL), \citet{bertsimas1996conservation} proposed the required structure of performance measures for the achievable region to be a polyhedron called extended polymatroid. Based on GCL, they extended Klimov's algorithm to several stochastic scheduling problems, including the classic MAB problem \citep{klimov1975time,thomas1991region}. By relaxing the GCL restrictions, \citet{nino2001restless} proposed the partial conservation law (PCL) and specified conditions for an RMAB to satisfy PCL and subsequently lead to the numerical calculation of Whittle index. Later on, \citet{nino2002dynamic} offered an economic explanation of PCL-indexability and its relation to Whittle indexability where Whittle index can be interpreted as the optimal marginal cost rate. A more comprehensive summary of this work can be found in \cite{nino2007dynamic}. However, all such work considers finite state spaces except for \citet{frostig1999four} who extended the GCL framework for classic MAB to countable state spaces. As detailed in \citet{nino2001restless}, the PCL framework is a significant generalization of GCL and works for a much broader class of Whittle-indexable RMAB compared to the GCL. Therefore, the complexity in building the PCL framework is much higher than that of GCL, as well as its extension to infinite state spaces. For example, GCL has a strong requirement that the full state combination set must satisfy a conservation law for the existence and optimality of the index policy; while PCL only requires a state combination subset satisfying a generalized conservation law. This significantly broadened the applicability of LP approach for solving bandit problems. Nevertheless, why the relaxed condition by PCL still ensures an optimal index policy is a deep theorem \citep{nino2001restless}, especially in the case of infinite states while state subsets become too diverse and also infinite. 

The main contribution of this paper is to establish the validity of PCL-framework for the general infinite-state RMAB model by carefully dealing with the limiting scenarios and fully exploring the intrinsic rich structure in the analysis of this high-dimensional infinite-state LP problem. The main challenge of this extension is due to the high-dimensional probability state space (belief states) with complex dynamics for transitions over time. By establishing the weak duality between LP formulations, we build the PCL framework for analyzing our RMAB and subsequently design an efficient algorithm for calculating Whittle index when PCL-indexability is satisfied. Finally, we demonstrate the superior performance of our algorithm by numerical simulations. We point out that \citet{nino2006restless} and \citet{nino2020mor} also considered PCL for special instances of RMAB with infinite state spaces. Specifically, \citet{nino2006restless} formulated the $M/G/1$ queuing control problem as an RMAB with a countable state space equipped with a simple state transition rule and a convex holding cost function where the family of threshold policies was focused on to analyze the indexability.  \citet{nino2020mor} extended the work of \citet{liu2010dynamic} for partially but error-free observable RMAB to the real-interval state space, i.e., the state space has dimension~$1$. Furthermore, only the family of threshold policies was considered but rather than proving the optimality of threshold policies before establishing indexability as in \citet{liu2010dynamic},  \citet{nino2020mor} gave sufficient conditions to obtain both the optimality of threshold policies and indexability at one time. \citet{AyestaEtal2021MMOR} adopted the time-average objective function with perfect observations. Under the ergodicity condition, the authors utilized the state transition dynamics to propose an improved exhaustive-search algorithm to check indexability and compute Whittle index. However, the algorithm works only for the case where the threshold policy is optimal in the single-armed problem with subsidy. Furthermore, the complexity of the algorithm is very high due to the exhaustive search in state combinations and numerical implementation of the algorithm still requires an approximate finite state space. In this paper, we consider the high-dimensional RMAB without any constraint on state dimension or transition or policy class and use linear programming to efficiently finish indexability verification and computation, as the first modeling and PCL analysis of RMAB with general observability.

\begin{table*}[!t]
\renewcommand{\arraystretch}{1.3}
\caption{Main Notations}\label{table:notations}
\label{notation}
\centering
\begin{tabular}{c|c}
\hline
$N$ & number of arms \\
\hline
$n$ & arm index \\
\hline
$\mathcal{M}_{n}$ & arm state space \\ 
\hline
$\mathcal{M}$ & joint state space of all arms\\ 
\hline
$K$ & number of arms to select at each decision epoch\\ 
\hline
$t$ & decision epoch index\\ 
\hline
$S$ & state of an arm \\
\hline
$O$ & observed state of an arm \\
\hline
$p_{ij}$ & transition probability between arm states \\
\hline
$\epsilon_{ij}$ & probability of observed state to be~$j$ given the true state being~$i$\\
\hline
$r_{ij}$ & immediate reward when the observed state is~$j$ and the true state is~$i$\\
\hline
$\omega_{i}$ & probability the arm is in state~$i$ given the past observations\\
\hline
$F$ & feedback state\\
\hline
$\rho_{ij}$ & probability of getting immediate reward~$j$ given the true state is~$i$\\
\hline
$\beta$ & discount factor\\
\hline
$\lambda$ & subsidy for passivity in the single-arm problem\\
\hline
$\mathcal{P}(\lambda)$ & passive set under subsidy~$\lambda$\\
\hline
$W(\cdot)$ & Whittle index function\\
\hline
$\mathcal{B}$ & belief update operator\\
\hline
$p_{ij}^k$ & probability that belief state transits to $\boldsymbol{\omega}_j$ given $\boldsymbol{\omega}_i$ and action~$a=k$ ($k=0$ for passive and $k=1$ for active)\\
\hline
$x_{\boldsymbol{\omega}}^{a}(\pi,\boldsymbol{\omega_{i}})$ & expected total discounted time that $\boldsymbol{\omega}$ is applied with action~$a$ under policy~$\pi$ and initial belief $\boldsymbol{\omega_{i}}$\\
\hline
$\Omega(\boldsymbol{\omega_{0}})$ & set of all possible belief states given initial belief $\boldsymbol{\omega_{0}}$\\
\hline
$\tilde{\Omega}$ & feasible set family\\
\hline
$\Omega_{\boldsymbol{\omega_{i}}}^{C}$ & an element (belief state subset) of full subset chain $C$ indexed by $\boldsymbol{\omega_{i}}$\\
\hline
\end{tabular}
\end{table*}

\section{Main Results}\label{sec:MainResult}
\subsection{Model Formulation}\label{sec:formulation}
In this section, we begin with the formulation of the RMAB problem with general observation models. The main notations adopted in this paper can be found in Table~\ref{table:notations}.
Assume the system has $N$ arms and the state space of the $n$-th arm is $\mathcal{M}_{n}$. Therefore the joint state space of all arms is $\mathcal{M}=\prod_{n=1}^{N}\mathcal{M}_{n}$. At each moment, the actual state of each arm undergoes a state transition based on its own Markov probability transition matrix, and we will select $K$ arms to activate. For those activated arms, we can observe their states (with errors) and accrue some reward dependent on the observed states and the actual states. For those arms that are not activated, we can neither observe their states nor obtain reward from them, and their states still transit over time. This assumption was motivated by practical scenarios where not selected arm does not provide reward or incur cost. For example, in the opportunistic spectrum access problem, if the user does not select an arm (channel) to sense and transmit data, then no observation or reward is received from that channel (see below and Sec.~\ref{sec:approxState} for details). Relaxing this assumption will be considered in the future work. Let $S_{n}(t), O_{n}(t)\in\mathbb{Z}^+$ be respectively the true and observed states of arm $n$ in decision epoch~$t$ and $A(t)$ the set of activated arms in $t$. First consider a single-arm process ({\em i.e.} $N=1$ and $\mathcal{M}=\mathcal{M}_n$) and drop the subscript~$n$ (sometimes the time index $t$ is also dropped if referring to the same time index) for convenience. Suppose the state transition matrix, error matrix and reward matrix are $P=\{p_{ij}\}_{\mathcal{M}\times\mathcal{M}}$, $E=\{\varepsilon_{ij}\}_{\mathcal{M}\times\mathcal{M}}$ and $R=\{r_{ij}\}_{\mathcal{M}\times\mathcal{M}}$ respectively. For error matrix, $\varepsilon_{ij}$ represents the probability that the observed state is $j$ when the true state is $i$, that is, $\varepsilon_{ij}=P(O=j\vert S=i)$, while $r_{ij}$ represents the reward obtained when the true state is $i$ and the observed state is $j$ (different observations may lead to different sub-actions and thus different rewards following). Denote by ${\Omega}_a=\{(\omega_{1},\cdots,\omega_{M})\vert\sum_{i=1}^{M}\omega_{i}=1,\omega_{i}\geq 0,1\leq i\leq M:=|\mathcal{M}|\}$ the belief space, for each $\boldsymbol{\omega}(t)=(\omega_{1}(t),\cdots,\omega_{M}(t))\in{\Omega_a}$, $\omega_{i}(t)$ represents the conditional probability that the true state is~$i$ in epoch~$t$ (based on past observations). If the current belief state is $(\omega_{1},\cdots,\omega_{M})$, the current expected reward is $\sum_{i=1}^{M}\omega_{i}(\sum_{j=1}^{M}\varepsilon_{ij}r_{ij})$. In addition, we may sometimes receive additional feedback related to the true and observed states that helps us better estimate the actual arm state ({\em e.g.,} ACK/NAK in a communication channel to indicate whether the data was successfully transmitted~\citep{LiuEtal24MMOR}). Assume that there are $L$ feedback states (positive integers dependent only on the true state and the observed state) in total that encompass {\em all} possible observations, and denote the feedback state at time~$t$ by $F(t)~(L\leq M^{2})$. For example, in the opportunistic spectrum access (OSA) problem, the user first chooses an arm (channel) to sense. If the channel is observed to be available ($O(t)=1$), then the user sends a data packet to the receiver. If the channel is truly available ($S(t)=1$), then the receiver successfully receives this data packet and sends back an acknowledgment ($F(t)=1$) in the end of this decision epoch. In this case, the user knows the channel state ($S(t)=1$). However, due to sensor errors, the user may sense the channel and find it to be unavailable ($O(t)=0$) while the channel is actually available ($S(t)=1$). Then the user will not send any data packet and certainly no acknowledgment will be obtained ($F(t)=0$). In this case, the user never knows the true state of the channel in this decision epoch. By Bayes rule, we have
\begin{equation}
    P(S(t+1)=j\vert F(t)=i)=\frac{P(S(t+1)=j,F(t)=i)}{P(F(t)=i)}=\frac{\sum_{k=1}^{M}p_{kj}\rho_{ki}\omega_{k}}{\sum_{k=1}^{M}\rho_{ki}\omega_{k}}.
\end{equation}
Thus the rule of belief update is
\begin{equation}
    \boldsymbol{\omega}(t+1)=\begin{cases}
    (\frac{\sum_{i=1}^{M}p_{i1}\rho_{ij}\omega_{i}}{\sum_{i=1}^{M}\rho_{ij}\omega_{i}},\cdots,\frac{\sum_{i=1}^{M}p_{iM}\rho_{ij}\omega_{i}}{\sum_{i=1}^{M}\rho_{ij}\omega_{i}}), &\text{if active and } F(t)=j\\
    \boldsymbol{\omega}(t)P, &\text{if passive}
    \end{cases}.
\end{equation}
In a simpler scenario where there is no additional feedback but only the observed states (e.g., no acknowledgment is sent back in the OSA problem), the update rule becomes
\begin{equation}
    \boldsymbol{\omega}(t+1)=\begin{cases}
    (\frac{\sum_{i=1}^{M}p_{i1}\varepsilon_{ij}\omega_{i}}{\sum_{i=1}^{M}\varepsilon_{ij}\omega_{i}},\cdots,\frac{\sum_{i=1}^{M}p_{iM}\varepsilon_{ij}\omega_{i}}{\sum_{i=1}^{M}\varepsilon_{ij}\omega_{i}}), &\text{if active and } O(t)=j\\
    \boldsymbol{\omega}(t)P, &\text{if passive}
    \end{cases}.
\end{equation}
In the extreme case when we cannot observe any state after activating an arm, we simply treat the obtained reward (might be~$0$) as the feedback. In general, the feedback state informs the player about the immediate reward obtained. By normalizing the immediate reward $r(t)$ to take only positive integer values ($r_{ij}$'s are assumed to be rational for all $i,~j$ since~$Q$ is dense in~$R$ and provides arbitrary precision in practical scenarios, Theorem~1.20 (b) in \citet{rudin1976principles}), we have
\begin{equation}
    \begin{aligned}
        \rho_{ij}
        :=&P(r(t)=j\vert S(t)=i)\notag\\
        =&\frac{P(r(t)=j,S(t)=i)}{P(S(t)=i)}\notag\\
        =&\frac{\sum_{k=1}^{M}P(r(t)=j,O(t)=k,S(t)=i)}{\omega_{i}}\notag\\
        =&\frac{\sum_{k=1}^{M}P(r(t)=j\vert O(t)=k,S(t)=i)P(O(t)=k\vert S(t)=i)\omega_{i}}{\omega_{i}}\notag\\
        =&\sum_{k=1}^{M}\mathbbm{1}(r_{ik}=j)\varepsilon_{ik}
    \end{aligned}
\end{equation}
or equivalently $\rho_{ir}=\sum_{j=1}^{M}\mathbbm{1}(r_{ij}=r)\varepsilon_{ij}$. Thus we have
\begin{align}\label{eqn:BeliefUpdateNumerator}
    \sum_{i=1}^{M}p_{i1}\rho_{ir}\omega_{i}
    ={}&\sum_{i=1}^{M}p_{i1}(\sum_{j=1}^{M}\mathbbm{1}(r_{ij}=r)\varepsilon_{ij}\omega_{i})\notag\\
    ={}&\sum_{i=1}^{M}\sum_{j=1}^{M}\mathbbm{1}(r_{ij}=r)p_{i1}\varepsilon_{ij}\omega_{i}\notag\\
    ={}&\sum_{r_{ij}=r}p_{i1}\varepsilon_{ij}\omega_{i}.
\end{align}
Similarly,
\begin{equation}\label{eqn:BeliefUpdateDenominator}
    \sum_{i=1}^{M}\rho_{ir}\omega_{i}=\sum_{r_{ij}=r}\varepsilon_{ij}\omega_{i}.
\end{equation}
Combining (\ref{eqn:BeliefUpdateNumerator}) and (\ref{eqn:BeliefUpdateDenominator}), the belief state update rule is
\begin{equation}
    \boldsymbol{\omega}(t+1)=\begin{cases}
    (\frac{\sum_{r_{ij}=r}p_{i1}\varepsilon_{ij}\omega_{i}}{\sum_{r_{ij}=r}\varepsilon_{ij}\omega_{i}},\cdots,\frac{\sum_{r_{ij}=r}p_{iM}\varepsilon_{ij}\omega_{i}}{\sum_{r_{ij}=r}\varepsilon_{ij}\omega_{i}}), &\text{if active and } r(t)=r\\
    \boldsymbol{\omega}(t)P, &\text{if passive}
    \end{cases}.
\end{equation}

A more practical situation is that we can jointly obtain information from the observed state and the obtained reward. In this case, we can similarly give the belief update rule:
\begin{equation}
    \boldsymbol{\omega}(t+1)=\begin{cases}
    (\frac{\sum_{i=1}^{M}p_{i1}\varepsilon_{ij}\omega_{i}\mathbbm{1}(r_{ij}=r)}{\sum_{i=1}^{M}\varepsilon_{ij}\omega_{i}\mathbbm{1}(r_{ij}=r)},\cdots,\frac{\sum_{i=1}^{M}p_{iM}\varepsilon_{ij}\omega_{i}\mathbbm{1}(r_{ij}=r)}{\sum_{i=1}^{M}\varepsilon_{ij}\omega_{i}\mathbbm{1}(r_{ij}=r)}), &\text{if active and } O(t)=j,\\&r(t)=r~(\exists i~s.t.~ r_{ij}=r)\\
    \boldsymbol{\omega}(t)P, &\text{if passive}
    \end{cases}.
\end{equation}

For the restless multi-armed bandit model, the goal is to find a policy $\pi$ that maps the belief states of all arms into an active set $A(t)$ in epoch~$t$ that maximize the long-term expected discounted reward. In other words, if we denote the reward obtained by the $n$-th arm in epoch~$t$ by $r_{n}(t)$, then our objective is
\begin{align}
    \max_{\pi}\mathbb{E}_{\pi}&\left[\sum_{t=1}^{+\infty}\beta^{t-1}\sum_{n=1}^{N}r_{n}(t)\bigg\vert\boldsymbol{\omega}_{1}(1),\cdots,\boldsymbol{\omega}_{n}(1)\right],\\
    &s.t. ~|A(t)|=K, t\geq 1,
\end{align}
where $0\leq\beta<1$ is the discount factor and $\boldsymbol{\omega}_{n}(t)$ is the belief state (vector) of arm $n$ in epoch~$t$. We point out that the initial belief state can be set arbitrarily but after it is given, the system state space is countable since the belief update happens only at discrete time indices. In this problem, the diversity of states, choices, errors makes the problem highly complex. In RMAB problems, searching for an easily computable priority index policy is the mainstream. The core idea is to assign an index (a real number) to the current state of each arm, and then activate those arms with top~$K$ large indexes. The goal of this paper is to theoretically characterize the conditions when such an index policy with near-optimal performance exists and provide a detailed algorithm for efficiently computing the index function (if it exists).

\subsection{Whittle Index}
\citet{whittle1988restless} relaxed the constraint on the {\em exact} number of arms activated in each decision epoch, requiring only that the {\it expected} number of arms activated per epoch on average (in the sense of discounted time) is $K$, {\em i.e.}
\begin{equation}\label{eqn:lagRelaxA}
    \mathbb{E}_{\pi}\left[\sum_{t=1}^{+\infty}\beta^{t-1}\sum_{n=1}^{N}\mathbbm{1}(n\in A(t))\bigg\vert\boldsymbol{\Omega}(1)\right]=\frac{K}{1-\beta}
\end{equation}
or
\begin{equation}
    \mathbb{E}_{\pi}\left[\sum_{t=1}^{+\infty}\beta^{t-1}\sum_{n=1}^{N}\mathbbm{1}(n\notin A(t))\bigg\vert\boldsymbol{\Omega}(1)\right]=\frac{N-K}{1-\beta},
\end{equation}
where $\boldsymbol{\Omega}(t)=(\boldsymbol{\omega}_{1}(t),\cdots,\boldsymbol{\omega}_{N}(t))$ with $\boldsymbol{\omega}_{n}(t)$ as the belief state (vector) for arm~$n$ at time~$t$. Thus the Lagrange relaxation of~\eqref{eqn:lagRelaxA} can be written as
\begin{equation}
    \max_{\pi}\mathbb{E}_{\pi}\left[\sum_{t=1}^{+\infty}\beta^{t-1}\sum_{n=1}^{N}(\mathbbm{1}(n\in A(t))r_{n}(t)+\lambda\mathbbm{1}(n\notin A(t)))\bigg\vert\boldsymbol{\Omega}(1)\right].
\end{equation}
Note that only the class of admissible (measurable) policies are considered, i.e., the action can only depend on past observations instead of future or unobserved information. The above problem can be decomposed into $N$ independent subproblems, that is, for any $1\leq n\leq N$,
\begin{equation}
    \max_{\pi}\mathbb{E}_{\pi}\left[\sum_{t=1}^{+\infty}\beta^{t-1}(\mathbbm{1}(n\in A(t))r_{n}(t)+\lambda\mathbbm{1}(n\notin A(t)))\bigg\vert\boldsymbol{\omega}_{n}(1)\right].
\end{equation}
The explanation of the optimization problem above is that for each arm, when it is not selected (made passive), we will receive a subsidy $\lambda$. Since the problems above are independent, we just need to consider the single-arm case. In other words, the Lagrange relaxation decouples an $N$-dimensional problem into $N$ $1$-dimensional problems, reducing the complexity with~$N$ from exponential to linear! From this point on, we will focus on the single-arm problem with subsidy and drop the arm index~$n$. For a given arm, the optimal policy for the relaxed optimization problem divides the arm state space into two subsets (here the arm state space is the arm belief state space): active set $\mathcal{A}(\lambda)$ and passive set $\mathcal{P}(\lambda)$. Specifically, $\mathcal{P}(\lambda)$ contains all belief states in which the optimal choice is passive when the subsidy is $\lambda$. In particular, for a certain state $s$, if both active and passive actions are optimal, we include it in $\mathcal{P}(\lambda)$, and $\mathcal{A}(\lambda)$ is just the complement of $\mathcal{P}(\lambda)$ in the entire state space. Under the concept of passive set, Whittle indexability can be stated as follows:
\begin{definition}[Whittle Indexability]\label{def:Indexability}
    An arm is indexable if the passive set $\mathcal{P}(\lambda)$ monotonically increases from $\emptyset$ to the whole state space as the subsidy $\lambda$ increases from $-\infty$ to $+\infty$. The RMAB problem is indexable if every decoupled problem is indexable.
\end{definition}

Indexability states that, once an arm is made passive with subsidy $\lambda$, it should also be passive with any $\lambda'$ larger than $\lambda$. If the problem satisfies indexability, its Whittle index is defined as follows:
\begin{definition}[Whittle Index]\label{def:WhittleIndex}
    If an arm is indexable, the Whittle index $W(s)$ of a state $s$ is the infimum subsidy $\lambda$ that keeps~$s$ in the passive set $\mathcal{P}(\lambda)$. That is,
    \begin{equation*}
        W(s):=\inf\{\lambda:s\in\mathcal{P}(\lambda)\}.
    \end{equation*}
\end{definition}
By continuity, $W(s)$ is the infimum subsidy that makes it equivalent to be active or passive at state~$s$. Under the definitions of indexability and Whittle index, the Whittle index policy for the original multi-armed bandit problem is simply to activate the $K$ arms whose states offer the largest Whittle indices. In fact, for the classic MAB problem where passive arms do not change states, the Whittle indexability is always satisfied and the Whittle index is reduced to the Gittins index.

\subsection{Belief State Space}\label{sec:CountableStateSpace}
Different from the perfect observation case, the update of the belief state is nonlinear for the general observation model. This yields much difficulty for us to use value functions and dynamic programming methods to analyze the problem. Given an initial belief state, it appears that the size of the belief state space will grow exponentially over time as all possible realizations of observations/rewards/feedback are traversed and updated. Fortunately, a large number of numerical experiments have shown that in the sense of Euclidean norm approximation, the approximate state space exhibits stability after a process of iterative calculations (see Sec.~\ref{sec:approximatePCL} for details). Note that the initial belief state is $\boldsymbol{\omega}$, $\{\mathcal{B}_{h}\}$ is a complete list of operators defined by observed and feedback states and state update rules. %, that is, for $1\leq l\leq L$, $\mathcal{B}_{l}\boldsymbol{\omega}$ is the updated state caused by the $l$-th feedback state under the active action, and $\mathcal{B}_{0}$ is the state update operator under the passive action. In this paper, $\mathcal{B}_{0}\boldsymbol{\omega}=\boldsymbol{\omega}P$.
Set $H:=|\{\mathcal{B}_{h}\}|$. Define the $T$-step state space recursively as follows:
\begin{definition}[$T$-step state space]\label{def:TStepStateSpace}
    Define
    \begin{equation*}
        \Omega(T\vert\boldsymbol{\omega})=\begin{cases}
        \{\mathcal{B}_{h_{1}}\cdots\mathcal{B}_{h_{T}}\boldsymbol{\omega}:1\leq h_{1},\cdots,h_{T}\leq H\}\cup\Omega(T-1\vert\boldsymbol{\omega}),& T\geq2\\
        \{\mathcal{B}_{h}\boldsymbol{\omega}:1\leq h\leq H\}\cup\{\boldsymbol{\omega}\},& T=1
        \end{cases}.
    \end{equation*}
    We call $\Omega(T\vert\boldsymbol{\omega})$ the $T$-step state space under the initial belief state $\boldsymbol{\omega}$.
\end{definition}

Under the definition of the $T$-step state space, $\Omega(+\infty\vert\boldsymbol{\omega_{0}})$ is the belief space obtained by traversing all possible state update rules starting from the initial state $\boldsymbol{\omega_{0}}$. Clearly, $\Omega(+\infty\vert\boldsymbol{\omega_{0}})$ is countable (since the belief update happens only at discrete time indices, see Sec.~\ref{sec:formulation}) and we denote it by $\Omega(\boldsymbol{\omega_{0}})$. Under a policy $\pi$, let $p_{ij}(\pi):=P(\boldsymbol{\omega}(t+1)=\boldsymbol{\omega_{j}}\vert\boldsymbol{\omega}(t)=\boldsymbol{\omega_{i}})$ be the conditional probability of transition from belief state $\boldsymbol{\omega_{i}}$ to $\boldsymbol{\omega_{j}}$. Due to different actions of activation and passivity affecting the state update, decomposition of $p_{ij}(\pi)$ should be carried out. Let $a(t)$ be the indicator function of whether the arm is activated in epoch~$t$, then we define the transition probability of the belief state as follows:
\begin{align}
    p_{ij}^{0}=P(\boldsymbol{\omega}(t+1)=\boldsymbol{\omega}_{j}\vert\boldsymbol{\omega}(t)=\boldsymbol{\omega}_{i},a(t)=0)\label{eqn:ProbPassive},\\
    p_{ij}^{1}=P(\boldsymbol{\omega}(t+1)=\boldsymbol{\omega}_{j}\vert\boldsymbol{\omega}(t)=\boldsymbol{\omega}_{i},a(t)=1)\label{eqn:ProbActive}.
\end{align}
Clearly, $ p_{ij}^{0}$ and $ p_{ij}^{1}$ do not depend on~$\pi$ since the action is already specified in the condition. By (\ref{eqn:ProbPassive})-(\ref{eqn:ProbActive}) and the law of total probability, we have
\begin{equation}
    \begin{aligned}
        p_{ij}(\pi)&=\sum\limits_{a=0}^{1}P_{\pi}(a(t)=a\vert\boldsymbol{\omega}(t)=\boldsymbol{\omega}_{i})p_{ij}^{a}.
        % &=P(\boldsymbol{\omega}(t+1)=\boldsymbol{\omega}_{j}\vert\boldsymbol{\omega}(t)=\boldsymbol{\omega}_{i})\\
        % &=\frac{P(\boldsymbol{\omega}(t+1)=\boldsymbol{\omega}_{j},\boldsymbol{\omega}(t)=\boldsymbol{\omega}_{i})}{P(\boldsymbol{\omega}(t)=\boldsymbol{\omega}_{i})}\\
        % &=\frac{\sum\limits_{a=0}^{1}P(\boldsymbol{\omega}(t+1)=\boldsymbol{\omega}_{j},\boldsymbol{\omega}(t)=\boldsymbol{\omega}_{i},a(t)=a)}{P(\boldsymbol{\omega}(t)=\boldsymbol{\omega}_{i})}\\
        % &=\frac{\sum\limits_{a=0}^{1}P(\boldsymbol{\omega}(t)=\boldsymbol{\omega}_{i},a(t)=a)p_{ij}^{a}}{P(\boldsymbol{\omega}(t)=\boldsymbol{\omega}_{i})}\\
        % &=\sum\limits_{a=0}^{1}P(a(t)=a\vert\boldsymbol{\omega}(t)=\boldsymbol{\omega}_{i})p_{ij}^{a}.
    \end{aligned}
\end{equation}
Note that $P_{\pi}(a(t)=a\vert\boldsymbol{\omega}(t)=\boldsymbol{\omega}_{i})$ depends on~$\pi$ since the policy specifies which action to take given the current belief state. The elements in the probability transition matrix under passive and active actions have the following more specific expressions:
\begin{equation}\label{eqn:pij0}
    p_{ij}^{0}=\begin{cases}
    1, &\text{if $\boldsymbol{\omega}_{j}=\boldsymbol{\omega}_{i}P$}\\
    0, &\text{otherwise}
    \end{cases}
\end{equation}
and
\begin{equation}\label{eqn:pij1}
    \begin{aligned}
        p_{ij}^{1}&=P(\boldsymbol{\omega}(t+1)=\boldsymbol{\omega}_j\vert\text{active and }\boldsymbol{\omega}(t)=\boldsymbol{\omega}_{i})\\
        &=\sum\limits_{\{h:\mathcal{B}_{h}\boldsymbol{\omega}_{i}=\boldsymbol{\omega}_{j},1\leq h\leq H\}}\mathcal{B}_{h}\boldsymbol{\omega}_{i}P(h\text{-th belief update rule}\vert\text{active and }\boldsymbol{\omega}(t)=\boldsymbol{\omega}_{i}).
    \end{aligned}
\end{equation}
In the following sections, these two probability transition matrices will play an important role in the calculation of Whittle indices.

\subsection{Two-Arm Problem and Achievable Region}
With the rapid development at the junction of operations research, stochastic optimization and reinforcement learning, quite a few effective methods for finite-state problems have been found, among which achievable region with conservation laws is a great example. \citet{bertsimas1996conservation} and \citet{nino2001restless} adopted the analytical framework of generalized conservation laws and partial conservation laws for the classic MAB and RMAB problems with finite state spaces, respectively, and provided efficient index algorithms for the corresponding problems. In the rest of the subsections, we will apply the PCL framework to the analysis of our RMAB problem with an infinite state space, and theoretically build the foundation for the construction of an efficient index policy. 

Consider the single-arm process with $M$ states discussed in the previous section. For an initial belief state $\boldsymbol{\omega_{0}}$, let $\Omega(\boldsymbol{\omega_{0}})$ be the countable belief state space generated by iteratively updating $\boldsymbol{\omega_{0}}$ through state transitions. During the process of making the arm active or passive, since all belief states fall within $\Omega(\boldsymbol{\omega_{0}})$, the entire time period can be completely partitioned by the time segments in which each state is located. In this scenario, define the performance indicator variables for each belief state as follows:
\begin{equation*}
    I_{\boldsymbol{\omega}}^{1}(t)=\begin{cases}
    1, &\text{if the arm is active in epoch~$t$ and its belief state is $\boldsymbol{\omega}$}\\
    0, &\text{otherwise}
    \end{cases}
\end{equation*}
and
\begin{equation*}
    I_{\boldsymbol{\omega}}^{0}(t)=\begin{cases}
    1, &\text{if the arm is passive in epoch~$t$ and its belief state is $\boldsymbol{\omega}$}\\
    0, &\text{otherwise}
    \end{cases}.
\end{equation*}
Furthermore, define the performance measures of belief state $\boldsymbol{\omega}$ as follows:
\begin{equation}
    x_{\boldsymbol{\omega}}^{a}(\pi)=\mathbb{E}_{\pi}\left[\sum_{t=0}^{+\infty}\beta^{t}I_{\boldsymbol{\omega}}^{a}(t)\right],\quad a=0,1
\end{equation}
\begin{equation}
    x_{\boldsymbol{\omega}}^{a}(\pi,\boldsymbol{\omega_{i}})=\mathbb{E}_{\pi}\left[\sum_{t=0}^{+\infty}\beta^{t}I_{\boldsymbol{\omega}}^{a}(t)\bigg\vert\boldsymbol{\omega}(0)=\boldsymbol{\omega}_{i}\right],\quad a=0,1
\end{equation}
where $\boldsymbol{\omega}(t)$ is the belief state of the arm in epoch~$t$. In Whittle's relaxation, the Lagrangian multiplier $\lambda$ can be regarded as a subsidy when the arm is made passive.  Thus the optimization objective of the partially observable RMAB problem with subsidy can be written as follows:
\begin{equation}\label{ObjectiveOfRMAB}
    \begin{aligned}
        \max&\sum\limits_{\boldsymbol{\omega}\in\Omega(\boldsymbol{\omega_{0}})}R_{\boldsymbol{\omega}}x_{\boldsymbol{\omega}}^{1}+\lambda\sum\limits_{\boldsymbol{\omega}\in\Omega(\boldsymbol{\omega_{0}})}x_{\boldsymbol{\omega}}^{0} \\
        \text{subject to\quad}&x_{\boldsymbol{\omega}_{i}}^{1}+x_{\boldsymbol{\omega_{i}}}^{0}=e_{\boldsymbol{\omega}_{i}}+\beta\left(\sum\limits_{\boldsymbol{\omega_{j}}\in\Omega(\boldsymbol{\omega_{0}})}p_{ji}^{1}x_{\boldsymbol{\omega}_{j}}^{1}+\sum\limits_{\boldsymbol{\omega_{j}}\in\Omega(\boldsymbol{\omega_{0}})}p_{ji}^{0}x_{\boldsymbol{\omega}_{j}}^{0}\right),\quad\forall\boldsymbol{\omega_{i}}\in\Omega(\boldsymbol{\omega_{0}})\\
        &x_{\boldsymbol{\omega}_{i}}^{1},x_{\boldsymbol{\omega}_{i}}^{0}\geq 0,\quad\forall\boldsymbol{\omega_{i}}\in\Omega(\boldsymbol{\omega_{0}})
    \end{aligned}
\end{equation}
where $R_{\boldsymbol{\omega}}=\sum\limits_{i=1}^{M}\omega_{i}R_{i},\boldsymbol{\omega}=(\omega_{1},\cdots,\omega_{M})$, $p_{ij}^{a}$ are given by (\ref{eqn:pij0})-(\ref{eqn:pij1}), $e_{\boldsymbol{\omega}_{i}}$ is the indicator whether the initial belief state is $\boldsymbol{\omega}_{i}$, $R_i$ is the known expected immediate reward conditional on $S(t)=i$ with the active action. Note that $R_{\boldsymbol{\omega}}$ is uniformly bounded over $\Omega(\boldsymbol{\omega_{0}})$ since $|R_{\boldsymbol{\omega}}|\le\max\{|R_i|\}$. As we will see in Sec.~\ref{sec:pcl}, both $\sum\limits_{\boldsymbol{\omega}\in\Omega(\boldsymbol{\omega_{0}})}x_{\boldsymbol{\omega}}^{1}$ and $\sum\limits_{\boldsymbol{\omega}\in\Omega(\boldsymbol{\omega_{0}})}x_{\boldsymbol{\omega}}^{0}$ are absolutely convergent so the objective function in~\eqref{ObjectiveOfRMAB} is well defined. The $\max$ operator  is justified by observing that the objective function is upper bounded by $\frac{\max\{|R_i|\}+|\lambda|}{1-\beta}$ and our following analysis aims at finding the optimal stationary policy $\pi$ to maximize the objective function (i.e., the expected total discounted reward). Together with the constraint on each variable, our formulation falls into the standard form in~ \citet{anderson1987linear}. We point out the LP problem with a countably infinite number of both variables and constraints is significantly more complex than a finite-dimensional one and we must use the infinite-dimensional LP theory for investigating the duality and optimality, as detailed below.

For the equality constraints of this optimization problem, it can be understood that the left side of the equation is the direct performance measure for state $\boldsymbol{\omega}_{i}$, while the right side of the equation is a decomposition of its occupied time over previous states transitioning into it. Therefore, the performance measure for state $\boldsymbol{\omega}_{i}$ can be represented by a combination of performance measures of other states. This model with subsidy can be explained more intuitively through a two-arm system. In this system, the first arm is the original arm while the second arm (auxiliary arm) has only one state~$0$. In every epoch we choose one of the two arms to activate. When the auxiliary arm is activated, we obtain a fixed reward $\lambda$. Our goal is to find a policy to maximize the long-term expected discounted reward by deciding which arm to choose in each epoch. The benefit of using the two-arm model for the single-arm bandit with subsidy is that the second arm never changes its state and only the selected arm yields reward. Specifically, the objective function can be simply written as
\begin{equation*}
    \max\limits\sum\limits_{\boldsymbol{\omega}\in\Omega(\boldsymbol{\omega_{0}})}R_{\boldsymbol{\omega}}x_{\boldsymbol{\omega}}^{1}(\pi)+\lambda x_{0}^{1}(\pi),
\end{equation*}
where $x_{0}^{1}(\pi)$ is the performance measure of state~0 being activated under policy~$\pi$. Let $X$ be the set of all elements $(x_{0}^{1}(\pi),x_{\boldsymbol{\omega}}^{1}(\pi))_{\boldsymbol{\omega}\in\Omega(\boldsymbol{\omega_{0}})}$ as~$\pi$ traverses the admissible (feasible) policy set~$\Pi$. We call $X$ the achievable region. Under this definition, the optimization objective function can be written as follows:
\begin{equation}\label{OPT}
    \quad\max_{(x_{0}^{1},x_{\boldsymbol{\omega}}^{1})\in X}\sum\limits_{\boldsymbol{\omega}\in\Omega(\boldsymbol{\omega_{0}})}R_{\boldsymbol{\omega}}x_{\boldsymbol{\omega}}^{1}+\lambda x_{0}^{1}\tag{OPT}.
\end{equation}
The core of solving this optimization problem is to mathematically characterize the achievable region $X$. The so-called conservation law refers to the fact that for any $(x_{0}^{1},x_{\boldsymbol{\omega}}^{1})\in X$, these components may satisfy certain equality or inequality constraints. For the above model, a trivial equality constraint is $x_{0}^{1}+\sum\limits_{\boldsymbol{\omega}\in\Omega(\boldsymbol{\omega_{0}})}x_{\boldsymbol{\omega}}^{1}=\frac{1}{1-\beta}$. This is because the RMAB system has exactly one state in the active phase at each moment (time-conservation). In the following subsections, we will further explore the rich mathematical structure of performance measures for the two-arm system with countable states under the concept of achievable region.

\subsection{Partial Conservation Law}\label{sec:pcl}
In this section, we elaborate the relationship between performance measures for each state in the two-arm system. We first introduce some important variables following a similar line of \citet{nino2001restless,nino2002dynamic}. To measure the performance of the original arm under policy $\pi$, we define the following variables:
\begin{align}
    T_{\boldsymbol{\omega}_{i}}^{\pi}&:=\mathbb{E}_{\pi}\left[\sum\limits_{t=0}^{+\infty}\beta^{t}\mathbbm{1}(a(t)=1)\bigg\vert\boldsymbol{\omega}(0)=\boldsymbol{\omega}_{i}\right],\quad \boldsymbol{\omega_{i}}\in\Omega(\boldsymbol{\omega_{0}}),\\
    R_{\boldsymbol{\omega}_{i}}^{\pi}&:=\mathbb{E}_{\pi}\left[\sum\limits_{t=0}^{+\infty}\beta^{t}R_{\boldsymbol{\omega}(t)}\mathbbm{1}(a(t)=1)\bigg\vert\boldsymbol{\omega}(0)=\boldsymbol{\omega}_{i}\right],\quad \boldsymbol{\omega_{i}}\in\Omega(\boldsymbol{\omega_{0}}).
\end{align}
Here $T_{\boldsymbol{\omega}_{i}}^{\pi}$ describes the expected total discounted time for the original arm to be activated in the two-arm system under policy $\pi$. We require that the policy $\pi$ only depends on the original arm. Similarly, $R_{\boldsymbol{\omega}_{i}}^{\Omega}$ represents the expected total discounted reward obtained from the original arm under policy $\pi$. Therefore, $T_{\boldsymbol{\omega}_{i}}^{\pi}$ and $R_{\boldsymbol{\omega}_{i}}^{\pi}$ also have the following expressions:
\begin{align}
    T_{\boldsymbol{\omega}_{i}}^{\pi}&=\sum\limits_{\boldsymbol{\omega}\in\Omega(\boldsymbol{\omega}_{0})}x_{\boldsymbol{\omega}}^{1}(\pi,\boldsymbol{\omega}_{i}),\\
    R_{\boldsymbol{\omega}_{i}}^{\pi}&=\sum\limits_{\boldsymbol{\omega}\in\Omega(\boldsymbol{\omega}_{0})}R_{\boldsymbol{\omega}}x_{\boldsymbol{\omega}}^{1}(\pi,\boldsymbol{\omega}_{i}).
\end{align}
Clearly, the optimal policy partitions the belief state space~$\Omega(\boldsymbol{\omega_{0}})$ of the original arm into two sets where the optimal actions should be active and passive respectively. For any $\Omega\subset\Omega(\boldsymbol{\omega_{0}})$, we call $\pi_{\Omega}$ an $\Omega$-priority policy if the original arm is activated when its current belief state falls into $\Omega$ and made passive when the state falls into $\Omega^{c}$. For $\Omega$-priority policy, we have
\begin{align}
    T_{\boldsymbol{\omega}_{i}}^{\Omega}&=\mathbb{E}_{\pi_{\Omega}}\left[\sum\limits_{t=0}^{+\infty}\beta^{t}\mathbbm{1}(\boldsymbol{\omega}(t)\in\Omega)\bigg\vert\boldsymbol{\omega}(0)=\boldsymbol{\omega}_{i}\right],\quad \boldsymbol{\omega_{i}}\in\Omega(\boldsymbol{\omega_{0}}),\\
    R_{\boldsymbol{\omega}_{i}}^{\Omega}&=\mathbb{E}_{\pi_{\Omega}}\left[\sum\limits_{t=0}^{+\infty}\beta^{t}R_{\boldsymbol{\omega}(t)}\mathbbm{1}(\boldsymbol{\omega}(t)\in\Omega)\bigg\vert\boldsymbol{\omega}(0)=\boldsymbol{\omega}_{i}\right],\quad \boldsymbol{\omega_{i}}\in\Omega(\boldsymbol{\omega_{0}}).
\end{align}
From the above definition, we can easily give the following dynamic programming equations for these two variables:
\begin{equation}
    T_{\boldsymbol{\omega}_{i}}^{\Omega}=\begin{cases}
    1+\beta\sum\limits_{\boldsymbol{\omega_{j}}\in\Omega(\boldsymbol{\omega_{0}})}p_{ij}^{1}T_{\boldsymbol{\omega}_{j}}^{\Omega}, &\boldsymbol{\omega}_{i}\in\Omega\\
    \beta\sum\limits_{\boldsymbol{\omega_{j}}\in\Omega(\boldsymbol{\omega_{0}})}p_{ij}^{0}T_{\boldsymbol{\omega}_{j}}^{\Omega} &\boldsymbol{\omega}_{i}\notin\Omega
    \end{cases}
\end{equation}
and
\begin{equation}
    R_{\boldsymbol{\omega}_{i}}^{\Omega}=\begin{cases}
    R_{\boldsymbol{\omega}_{i}}+\beta\sum\limits_{\boldsymbol{\omega_{j}}\in\Omega(\boldsymbol{\omega_{0}})}p_{ij}^{1}R_{\boldsymbol{\omega}_{j}}^{\Omega}, &\boldsymbol{\omega}_{i}\in\Omega\\
    \beta\sum\limits_{\boldsymbol{\omega_{j}}\in\Omega(\boldsymbol{\omega_{0}})}p_{ij}^{0}R_{\boldsymbol{\omega}_{j}}^{\Omega} &\boldsymbol{\omega}_{i}\notin\Omega
    \end{cases}.
\end{equation}

If $\Omega(\boldsymbol{\omega_{0}})$ is a finite state space, we can directly solve for $T_{\boldsymbol{\omega_{i}}}$ and $R_{\boldsymbol{\omega_{i}}}$ from the two equation sets above. %Specifically, if the auxiliary arm in the two-arm system is considered for illustration, then the arm has only one state 0 with $R_{0}=\lambda,p_{00}^{1}=p_{00}^{0}=1$. In this case, we have $T_{0}^{\{0\}}=\frac{1}{1-\beta}$ and $R_{0}^{\{0\}}=\frac{\lambda}{1-\beta}$. 
However, a direct solution is not available for countable state spaces. To further investigate the properties of these variables, we define $V_{\boldsymbol{\omega_{i}}}^{\pi}$ as the expected total discounted reward starting from state~$\boldsymbol{\omega_{i}}$ under policy~$\pi$. We have
\begin{equation}
V_{\boldsymbol{\omega_{i}}}^{\pi}=R_{\boldsymbol{\omega_{i}}}^{\pi}+\lambda\left(\frac{1}{1-\beta}-T_{\boldsymbol{\omega_{i}}}^{\pi}\right).
\end{equation}
Let $V_{\boldsymbol{\omega_{i}},0}^{\pi}$ and $V_{\boldsymbol{\omega_{i}},1}^{\pi}$ be the expected total discounted rewards with the initial state $\boldsymbol{\omega_{i}}$ when taking the passive and active actions respectively, then
\begin{align}
V_{\boldsymbol{\omega_{i}},0}^{\pi}=\lambda+\beta\sum\limits_{\boldsymbol{\omega_{j}}\in\Omega(\boldsymbol{\omega_{0}})}p_{ij}^{0}V_{\boldsymbol{\omega_{j}}}^{\pi},\\
V_{\boldsymbol{\omega_{i}},1}^{\pi}=R_{\boldsymbol{\omega_{i}}}+\beta\sum\limits_{\boldsymbol{\omega_{j}}\in\Omega(\boldsymbol{\omega_{0}})}p_{ij}^{1}V_{\boldsymbol{\omega_{j}}}^{\pi}.
\end{align}
If, under policy $\pi$, the expected total discounted reward obtained by active and passive actions are the same for the initial state $\boldsymbol{\omega_{i}}$, i.e., $V_{\boldsymbol{\omega_{i},0}}^{\pi}=V_{\boldsymbol{\omega_{i},1}}^{\pi}$, then we can solve for~$\lambda$:
\begin{equation}\label{eqn:lambdaExpress}
\lambda=\frac{R_{\boldsymbol{\omega_{i}}}+\beta\sum\limits_{\boldsymbol{\omega_{j}}\in\Omega(\boldsymbol{\omega_{0}})}(p_{ij}^{1}-p_{ij}^{0})R_{\boldsymbol{\omega_{j}}}^{\pi}}{1+\beta\sum\limits_{\boldsymbol{\omega_{j}}\in\Omega(\boldsymbol{\omega_{0}})}(p_{ij}^{1}-p_{ij}^{0})T_{\boldsymbol{\omega_{j}}}^{\pi}}.
\end{equation}
%The computation result of $\lambda$ gives us an intuitive construction form similar to the Whittle index.
Now define
\begin{align}
    A_{\boldsymbol{\omega}_{i}}^{\Omega}&:=1+\beta\sum\limits_{\boldsymbol{\omega_{j}}\in\Omega(\boldsymbol{\omega_{0}})}(p_{ij}^{1}-p_{ij}^{0})T_{\boldsymbol{\omega}_{j}}^{\Omega^{c}},\\
    W_{\boldsymbol{\omega}_{i}}^{\Omega}&:=R_{\boldsymbol{\omega}_{i}}+\beta\sum\limits_{\boldsymbol{\omega_{j}}\in\Omega(\boldsymbol{\omega_{0}})}(p_{ij}^{1}-p_{ij}^{0})R_{\boldsymbol{\omega}_{j}}^{\Omega^{c}}.
\end{align}
%Since $T_{\boldsymbol{\omega_{j}}}^{\Omega^{c}}$ and $R_{\boldsymbol{\omega_{j}}}^{\Omega^{c}}$ are bounded above, $A_{\boldsymbol{\omega_{i}}}$ and $W_{\boldsymbol{\omega_{i}}}$ are well defined. 
For the auxiliary arm, we have $A_{0}^{\{0\}}=1$ and $W_{0}^{\{0\}}=\lambda$. Generally, we can extend these variables to the multi-arm case. Assume that there are two arms and their state spaces are $\Omega(\boldsymbol{\omega})$ and $\Omega'(\boldsymbol{\omega}')$ with $\Omega(\boldsymbol{\omega})\cap\Omega'(\boldsymbol{\omega}')=\emptyset$. Then for $\Omega\subset\Omega(\boldsymbol{\omega})$,  $\Omega'\subset\Omega'(\boldsymbol{\omega}')$ and $\boldsymbol{\omega_{i}}\in\Omega(\boldsymbol{\omega})$, we define $T_{\boldsymbol{\omega_{i}}}^{\Omega\cup\Omega'}:=T_{\boldsymbol{\omega_{i}}}^{\Omega}, R_{\boldsymbol{\omega_{i}}}^{\Omega\cup\Omega'}:=R_{\boldsymbol{\omega_{i}}}^{\Omega}, A_{\boldsymbol{\omega_{i}}}^{\Omega\cup\Omega'}:=A_{\boldsymbol{\omega_{i}}}^{\Omega}$ and $W_{\boldsymbol{\omega_{i}}}^{\Omega\cup\Omega'}:=W_{\boldsymbol{\omega_{i}}}^{\Omega}$. Previously, Niño-Mora provided the relationships among $T_{\boldsymbol{\omega_{i}}}^{\Omega}$, $R_{\boldsymbol{\omega_{i}}}^{\Omega}$, $A_{\boldsymbol{\omega_{i}}}^{\Omega}$, and $W_{\boldsymbol{\omega_{i}}}^{\Omega}$ for finite-state problems in \citep{nino2001restless,nino2002dynamic}. After the above formulation, we prove the important result that these conservation relations also hold for general countable state spaces:
\begin{theorem}[Decomposition Law]\label{prop:DecompLaw}
For any $\boldsymbol{\omega_{j}}\in\Omega(\boldsymbol{\omega_{0}})$ and $\Omega\subset\Omega(\boldsymbol{\omega_{0}})$, the following two identities hold:
\begin{align}
    T_{\boldsymbol{\omega}_{j}}^{\pi}+\sum\limits_{\boldsymbol{\omega_{i}}\in\Omega^{c}}A_{\boldsymbol{\omega_{i}}}^{\Omega}x_{\boldsymbol{\omega_{i}}}^{0}(\pi,\boldsymbol{\omega}_{j})&=T_{\boldsymbol{\omega}_{j}}^{\Omega^{c}}+\sum\limits_{\boldsymbol{\omega_{i}}\in\Omega}A_{\boldsymbol{\omega_{i}}}^{\Omega}x_{\boldsymbol{\omega_{i}}}^{1}(\pi,\boldsymbol{\omega}_{j})\label{eqn:TDecompLaw},\\
    R_{\boldsymbol{\omega}_{j}}^{\pi}+\sum\limits_{\boldsymbol{\omega_{i}}\in\Omega^{c}}W_{\boldsymbol{\omega_{i}}}^{\Omega}x_{\boldsymbol{\omega_{i}}}^{0}(\pi,\boldsymbol{\omega}_{j})&=R_{\boldsymbol{\omega}_{j}}^{\Omega^{c}}+\sum\limits_{\boldsymbol{\omega_{i}}\in\Omega}W_{\boldsymbol{\omega_{i}}}^{\Omega}x_{\boldsymbol{\omega_{i}}}^{1}(\pi,\boldsymbol{\omega}_{j})\label{eqn:RDecompLaw}.
\end{align}
\end{theorem}
\proof
Assume the initial belief state is $\boldsymbol{\omega_{k}}$. For convenience, abbreviate $x_{\boldsymbol{\omega_{i}}}^{1}(\pi,\boldsymbol{\omega_{k}})$ and $x_{\boldsymbol{\omega_{i}}}^{0}(\pi,\boldsymbol{\omega_{k}})$ as $x_{\boldsymbol{\omega_{i}}}^{1}$ and $x_{\boldsymbol{\omega_{i}}}^{0}$. In this problem, we also have the constraint equation:
\begin{equation*}
    x_{\boldsymbol{\omega_{i}}}^{1}+x_{\boldsymbol{\omega_{i}}}^{0}=e_{\boldsymbol{\omega_{i}}}+\beta\sum\limits_{\boldsymbol{\omega_{j}}\in\Omega(\boldsymbol{\omega_{0}})}p_{ji}^{1}x_{\boldsymbol{\omega_{j}}}^{1}+\beta\sum\limits_{\boldsymbol{\omega_{j}}\in\Omega(\boldsymbol{\omega_{0}})}p_{ji}^{0}x_{\boldsymbol{\omega_{j}}}^{0},\quad \boldsymbol{\omega_{i}}\in\Omega(\boldsymbol{\omega_{0}}).
\end{equation*}
Since $\sum\limits_{\boldsymbol{\omega_{i}}\in\Omega(\boldsymbol{\omega_{0}})}(x_{\boldsymbol{\omega_{i}}}^{1}+x_{\boldsymbol{\omega_{i}}}^{0})=\frac{1}{1-\beta}$, and for any $\boldsymbol{\omega_{i}}\in\Omega(\boldsymbol{\omega_{0}})$, $x_{\boldsymbol{\omega_{i}}}^{1},x_{\boldsymbol{\omega_{i}}}^{0}\geq 0$, thus by the bounded convergence theorem for monotonic sequences (Theorem 3.14 in \citet{rudin1976principles}), $\sum\limits_{\boldsymbol{\omega_{i}}\in\Omega(\boldsymbol{\omega_{0}})}x_{\boldsymbol{\omega_{i}}}^{1}$ and $\sum\limits_{\boldsymbol{\omega_{i}}\in\Omega(\boldsymbol{\omega_{0}})}x_{\boldsymbol{\omega_{i}}}^{0}$ are both absolutely convergent. Due to the countability of $\Omega(\boldsymbol{\omega_{0}})$, summing the above constraint equations gives:
\begin{equation}
    \frac{1}{1-\beta}=\sum\limits_{\boldsymbol{\omega_{i}}\in\Omega(\boldsymbol{\omega_{0}})}\sum\limits_{\boldsymbol{\omega_{j}}\in\Omega(\boldsymbol{\omega_{0}})}p_{ji}^{1}x_{\boldsymbol{\omega_{j}}}^{1}+\sum\limits_{\boldsymbol{\omega_{i}}\in\Omega(\boldsymbol{\omega_{0}})}\sum\limits_{\boldsymbol{\omega_{j}}\in\Omega(\boldsymbol{\omega_{0}})}p_{ji}^{0}x_{\boldsymbol{\omega_{j}}}^{0}.
\end{equation}
Therefore, $\sum\limits_{\boldsymbol{\omega_{i}}\in\Omega(\boldsymbol{\omega_{0}})}\sum\limits_{\boldsymbol{\omega_{j}}\in\Omega(\boldsymbol{\omega_{0}})}p_{ji}^{1}x_{\boldsymbol{\omega_{j}}}^{1}$ and $\sum\limits_{\boldsymbol{\omega_{i}}\in\Omega(\boldsymbol{\omega_{0}})}\sum\limits_{\boldsymbol{\omega_{j}}\in\Omega(\boldsymbol{\omega_{0}})}p_{ji}^{0}x_{\boldsymbol{\omega_{j}}}^{0}$ are both absolutely convergent by the bounded convergence theorem for monotonic sequences and their summation orders are exchangeable (Theorem 8.3 in \citet{rudin1976principles}). By constraint equations and absolute convergence of series, we have
\begin{equation}
    \sum\limits_{\boldsymbol{\omega_{i}}\in\Omega(\boldsymbol{\omega_{0}})}(x_{\boldsymbol{\omega_{i}}}^{1}-\beta\sum\limits_{\boldsymbol{\omega_{j}}\in\Omega(\boldsymbol{\omega_{0}})}p_{ji}^{1}x_{\boldsymbol{\omega_{j}}}^{1})T_{\boldsymbol{\omega_{i}}}^{\Omega^{c}}+\sum\limits_{\boldsymbol{\omega_{i}}\in\Omega(\boldsymbol{\omega_{0}})}(x_{\boldsymbol{\omega_{i}}}^{0}-\beta\sum\limits_{\boldsymbol{\omega_{j}}\in\Omega(\boldsymbol{\omega_{0}})}p_{ji}^{0}x_{\boldsymbol{\omega_{j}}}^{0})T_{\boldsymbol{\omega_{i}}}^{\Omega^{c}}=T_{\boldsymbol{\omega_{k}}}^{\Omega^{c}}.
\end{equation}
Since $T_{\boldsymbol{\omega_{i}}}^{\Omega^{c}}$ has the upper bound $\frac{1}{1-\beta}$, the series in the above equation converge. In this equation,
\begin{align*}
    &\sum\limits_{\boldsymbol{\omega_{i}}\in\Omega(\boldsymbol{\omega_{0}})}(x_{\boldsymbol{\omega_{i}}}^{1}-\beta\sum\limits_{\boldsymbol{\omega_{j}}\in\Omega(\boldsymbol{\omega_{0}})}p_{ji}^{1}x_{\boldsymbol{\omega_{j}}}^{1})T_{\boldsymbol{\omega_{i}}}^{\Omega^{c}}\\
    =&\sum\limits_{\boldsymbol{\omega_{i}}\in\Omega(\boldsymbol{\omega_{0}})}x_{\boldsymbol{\omega_{i}}}^{1}T_{\boldsymbol{\omega_{i}}}^{\Omega^{c}}-\beta\sum\limits_{\boldsymbol{\omega_{i}}\in\Omega(\boldsymbol{\omega_{0}})}\sum\limits_{\boldsymbol{\omega_{j}}\in\Omega(\boldsymbol{\omega_{0}})}p_{ji}^{1}x_{\boldsymbol{\omega_{j}}}^{1}T_{\boldsymbol{\omega_{i}}}^{\Omega^{c}}\\
    =&\sum\limits_{\boldsymbol{\omega_{i}}\in\Omega(\boldsymbol{\omega_{0}})}x_{\boldsymbol{\omega_{i}}}^{1}T_{\boldsymbol{\omega_{i}}}^{\Omega^{c}}-\beta\sum\limits_{\boldsymbol{\omega_{j}}\in\Omega(\boldsymbol{\omega_{0}})}\sum\limits_{\boldsymbol{\omega_{i}}\in\Omega(\boldsymbol{\omega_{0}})}p_{ij}^{1}x_{\boldsymbol{\omega_{i}}}^{1}T_{\boldsymbol{\omega_{j}}}^{\Omega^{c}}\\
    =&\sum\limits_{\boldsymbol{\omega_{i}}\in\Omega(\boldsymbol{\omega_{0}})}x_{\boldsymbol{\omega_{i}}}^{1}T_{\boldsymbol{\omega_{i}}}^{\Omega^{c}}-\beta\sum\limits_{\boldsymbol{\omega_{i}}\in\Omega(\boldsymbol{\omega_{0}})}\sum\limits_{\boldsymbol{\omega_{j}}\in\Omega(\boldsymbol{\omega_{0}})}p_{ij}^{1}x_{\boldsymbol{\omega_{i}}}^{1}T_{\boldsymbol{\omega_{j}}}^{\Omega^{c}}\\
    =&\sum\limits_{\boldsymbol{\omega_{i}}\in\Omega(\boldsymbol{\omega_{0}})}x_{\boldsymbol{\omega_{i}}}^{1}(T_{\boldsymbol{\omega_{i}}}^{\Omega^{c}}-\beta\sum\limits_{\boldsymbol{\omega_{j}}\in\Omega(\boldsymbol{\omega_{0}})}p_{ij}^{1}T_{\boldsymbol{\omega_{j}}}^{\Omega^{c}})\\
    =&\sum\limits_{\boldsymbol{\omega_{i}}\in\Omega}x_{\boldsymbol{\omega_{i}}}^{1}(T_{\boldsymbol{\omega_{i}}}^{\Omega^{c}}-\beta\sum\limits_{\boldsymbol{\omega_{j}}\in\Omega(\boldsymbol{\omega_{0}})}p_{ij}^{1}T_{\boldsymbol{\omega_{j}}}^{\Omega^{c}})+\sum\limits_{\boldsymbol{\omega_{i}}\in\Omega^{c}}x_{\boldsymbol{\omega_{i}}}^{1}(T_{\boldsymbol{\omega_{i}}}^{\Omega^{c}}-\beta\sum\limits_{\boldsymbol{\omega_{j}}\in\Omega(\boldsymbol{\omega_{0}})}p_{ij}^{1}T_{\boldsymbol{\omega_{j}}}^{\Omega^{c}})\\
    =&\sum\limits_{\boldsymbol{\omega_{i}}\in\Omega}x_{\boldsymbol{\omega_{i}}}^{1}(1-A_{\boldsymbol{\omega_{i}}}^{\Omega})+\sum\limits_{\boldsymbol{\omega_{i}}\in\Omega^{c}}x_{\boldsymbol{\omega_{i}}}^{1}\\
    =&\sum\limits_{\boldsymbol{\omega_{i}}\in\Omega(\boldsymbol{\omega_{0}})}x_{\boldsymbol{\omega_{i}}}^{1}-\sum\limits_{\boldsymbol{\omega_{i}}\in\Omega}x_{\boldsymbol{\omega_{i}}}^{1}A_{\boldsymbol{\omega_{i}}}^{\Omega}.
\end{align*}
Here, the third equality exchanges the order of summation and the sixth equality uses the equations of $A_{\boldsymbol{\omega_{i}}}^{\Omega}$ and $W_{\boldsymbol{\omega_{i}}}^{\Omega}$. Similarly,
\begin{align*}
&\sum\limits_{\boldsymbol{\omega_{i}}\in\Omega(\boldsymbol{\omega_{0}})}(x_{\boldsymbol{\omega_{i}}}^{0}-\beta\sum\limits_{\boldsymbol{\omega_{j}}\in\Omega(\boldsymbol{\omega_{0}})}p_{ji}^{0}x_{\boldsymbol{\omega_{j}}}^{0})T_{\boldsymbol{\omega_{i}}}^{\Omega^{c}}\\
=&\sum\limits_{\boldsymbol{\omega_{i}}\in\Omega}x_{\boldsymbol{\omega_{i}}}^{0}(T_{\boldsymbol{\omega_{i}}}^{\Omega^{c}}-\beta\sum\limits_{\boldsymbol{\omega_{j}}\in\Omega(\boldsymbol{\omega_{0}})}p_{ij}^{0}T_{\boldsymbol{\omega_{j}}}^{\Omega^{c}})+\sum\limits_{\boldsymbol{\omega_{i}}\in\Omega^{c}}x_{\boldsymbol{\omega_{i}}}^{0}(T_{\boldsymbol{\omega_{i}}}^{\Omega^{c}}-\beta\sum\limits_{\boldsymbol{\omega_{j}}\in\Omega(\boldsymbol{\omega_{0}})}p_{ij}^{0}T_{\boldsymbol{\omega_{j}}}^{\Omega^{c}})\\
=&\sum\limits_{\boldsymbol{\omega_{i}}\in \Omega^{c}}x_{\boldsymbol{\omega_{i}}}^{0}A_{\boldsymbol{\omega_{i}}}^{\Omega}.
\end{align*}
Note that
\begin{equation}
    \sum\limits_{\boldsymbol{\omega_{i}}\in\Omega(\boldsymbol{\omega_{0}})}x_{\boldsymbol{\omega_{i}}}^{1}(\pi,\boldsymbol{\omega_{k}})=T_{\boldsymbol{\omega_{k}}}^{\pi}.
\end{equation}
We obtain
\begin{equation}
    T_{\boldsymbol{\omega_{k}}}^{\pi}+\sum\limits_{\boldsymbol{\omega_{i}}\in \Omega^{c}}A_{\boldsymbol{\omega_{i}}}^{\Omega}x_{\boldsymbol{\omega_{i}}}^{0}(\pi,\boldsymbol{\omega_{k}})=T_{\boldsymbol{\omega_{k}}}^{\Omega^{c}}+\sum\limits_{\boldsymbol{\omega_{i}}\in\Omega}A_{\boldsymbol{\omega_{i}}}^{\Omega}x_{\boldsymbol{\omega_{i}}}^{1}(\pi,\boldsymbol{\omega_{k}}).
\end{equation}
Similarly we have
\begin{equation}
    R_{\boldsymbol{\omega_{k}}}^{\pi}+\sum\limits_{\boldsymbol{\omega_{i}}\in\Omega^{c}}W_{\boldsymbol{\omega_{i}}}^{\Omega}x_{\boldsymbol{\omega_{i}}}^{0}(\pi,\boldsymbol{\omega_{k}})=R_{\boldsymbol{\omega_{k}}}^{\Omega^{c}}+\sum\limits_{\boldsymbol{\omega_{i}}\in\Omega}W_{\boldsymbol{\omega_{i}}}^{\Omega}x_{\boldsymbol{\omega_{i}}}^{1}(\pi,\boldsymbol{\omega_{k}}).
\end{equation}
\Halmos
\endproof

We can also perform certain transformations on two equations. Note that $x_{0}^{1}(\pi,\boldsymbol{\omega_{j}})=\frac{1}{1-\beta}-T_{\boldsymbol{\omega_{j}}}^{\pi}$, equation (\ref{eqn:TDecompLaw}) can be written as
\begin{equation}
x_{0}^{1}(\pi,\boldsymbol{\omega_{j}})+\sum\limits_{\boldsymbol{\omega_{i}}\in\Omega}A_{\boldsymbol{\omega_{i}}}^{\Omega}x_{\boldsymbol{\omega_{i}}}^{1}(\pi,\boldsymbol{\omega_{j}})=\frac{1}{1-\beta}-T_{\boldsymbol{\omega_{j}}}^{\Omega^{c}}+\sum\limits_{\boldsymbol{\omega_{i}}\in\Omega^{c}}A_{\boldsymbol{\omega_{i}}}^{\Omega}x_{\boldsymbol{\omega_{i}}}^{0}(\pi,\boldsymbol{\omega_{j}})\label{eqn:ADecompLaw1}.
\end{equation}
Similarly, we can convert (\ref{eqn:RDecompLaw}) to
\begin{equation}
\lambda x_{0}^{1}(\pi,\boldsymbol{\omega_{j}})+\sum\limits_{\boldsymbol{\omega_{i}}\in\Omega}W_{\boldsymbol{\omega_{i}}}^{\Omega}x_{\boldsymbol{\omega_{i}}}^{1}(\pi,\boldsymbol{\omega_{j}})=V_{\boldsymbol{\omega_{j}}}^{\pi}-R_{\boldsymbol{\omega_{j}}}^{\Omega^{c}}+\sum\limits_{\boldsymbol{\omega_{i}}\in\Omega^{c}}W_{\boldsymbol{\omega_{i}}}^{\Omega}x_{\boldsymbol{\omega_{i}}}^{0}(\pi,\boldsymbol{\omega_{j}})\label{eqn:WDecompLaw1}.
\end{equation}
If for any $\boldsymbol{\omega_{i}}\in\Omega^{c}$ we have $A_{\boldsymbol{\omega_{i}}}^{\Omega}\geq 0$, then
\begin{equation*}
x_{0}^{1}(\pi,\boldsymbol{\omega_{j}})+\sum\limits_{\boldsymbol{\omega_{i}}\in\Omega}A_{\boldsymbol{\omega_{i}}}^{\Omega}x_{\boldsymbol{\omega_{i}}}^{1}(\pi,\boldsymbol{\omega_{j}})\geq\frac{1}{1-\beta}-T_{\boldsymbol{\omega_{j}}}^{\Omega^{c}}.
\end{equation*}
The equality holds when policy $\pi$ gives higher priority to state 0 (the second arm) than states in $\Omega^{c}$. However, not all $\Omega\subset\Omega(\boldsymbol{\omega_{0}})$ have such property, so we need to consider a family of sets $\tilde{\Omega}\subset 2^{\Omega(\boldsymbol{\omega_{0}})}$, which satisfies $\Omega(\boldsymbol{\omega_{0}})\in\tilde{\Omega}$. For any $\Omega\in\tilde{\Omega}$, define $b(\Omega):=\inf\limits_{\pi}\{\sum\limits_{\boldsymbol{\omega}\in\Omega}A_{\boldsymbol{\omega}}^{\Omega}x_{\boldsymbol{\omega}}^{1}(\pi)\},b(\{0\}\cup\Omega)=\inf\limits_{\pi}\{x_{0}^{1}(\pi)+\sum\limits_{\boldsymbol{\omega}\in\Omega}A_{\boldsymbol{\omega}}^{\Omega}x_{\boldsymbol{\omega}}^{1}(\pi)\}$. Here we give the following definition of partial conservation law generalized from \citet{nino2001restless}:
\begin{definition}[Partial Conservation Law (PCL)]\label{def:PCL}
Let $\tilde{\Omega}$ be a feasible set family on the full belief-state space $\Omega(\boldsymbol{\omega}_{0})$ (i.e., $\emptyset,\ \Omega(\boldsymbol{\omega}_{0})\in\tilde{\Omega}$, and the complement of any element in $\tilde{\Omega}$ is taken with respect to $\Omega(\boldsymbol{\omega}_{0})$). The problem is said to satisfy the partial conservation law (PCL) with respect to $\tilde{\Omega}$ if, for any $\Omega\in\tilde{\Omega}$ and policy $\pi$, the following hold:
\begin{equation*}
    A_{\boldsymbol{\omega}}^{\Omega}>0, ~\forall\boldsymbol{\omega}\in\Omega,
\end{equation*}
\begin{equation*}
    x_{0}^{1}(\pi)+\sum\limits_{\boldsymbol{\omega}\in\Omega}A_{\boldsymbol{\omega}}^{\Omega}x_{\boldsymbol{\omega}}^{1}(\pi)\geq b(\{0\}\cup\Omega),
\end{equation*}
\begin{equation*}
    \sum\limits_{\boldsymbol{\omega}\in\Omega}A_{\boldsymbol{\omega}}^{\Omega}x_{\boldsymbol{\omega}}^{1}(\pi)\geq b(\Omega),
\end{equation*}
\begin{equation*}
    x_{0}^{1}(\pi)+\sum\limits_{\boldsymbol{\omega}\in\Omega}x_{\boldsymbol{\omega}}^{1}(\pi)=b(\{0\}\cup\Omega(\boldsymbol{\omega}_{0})).
\end{equation*}
The first inequality holds with equality if policy $\pi$ gives priority to 0 over $\Omega$, and the second inequality holds with equality if policy $\pi$ gives priority to $\Omega^{c}$ over 0.
\end{definition}

Conservation law limits the range of feasible region composed of performance measures of various states through a series of constraints. However, we are not yet ready to judge whether the objective function achieves the optimal value in this constrained region. In order to further characterize the indices of priority policies, we need to restrict the aforementioned $\tilde{\Omega}$. We  first give the following definition on a chain of subsets of $\tilde{\Omega}$:
\begin{definition}[Full subset chain of $\tilde{\Omega}$]\label{def:FullSubsetChain}
    We call a collection family $C=\{\Omega_{\boldsymbol{\omega_{i}}}^{C}\}_{\boldsymbol{\omega_{i}}\in\Omega(\boldsymbol{\omega_{0}})}\subset\tilde{\Omega}$ a full subset chain of $\tilde{\Omega}$, if it satisfies the following conditions:\\
    (i) $\forall\boldsymbol{\omega_{i}}\in\Omega(\boldsymbol{\omega_{0}}),\boldsymbol{\omega_{i}}\in\Omega_{\boldsymbol{\omega_{i}}}^{C}$;\\
    (ii) $\forall\boldsymbol{\omega_{i}},\boldsymbol{\omega_{j}}\in\Omega(\boldsymbol{\omega_{0}}), \boldsymbol{\omega_{i}}\neq\boldsymbol{\omega_{j}}$, there must be $\Omega_{\boldsymbol{\omega_{i}}}^{C}\subsetneq \Omega_{\boldsymbol{\omega_{j}}}^{C}$ or $\Omega_{\boldsymbol{\omega_{j}}}^{C}\subsetneq \Omega_{\boldsymbol{\omega_{i}}}^{C}$, and if $\Omega_{\boldsymbol{\omega_{i}}}^{C}\subsetneq \Omega_{\boldsymbol{\omega_{j}}}^{C}$, then $\boldsymbol{\omega_{j}}\notin\Omega_{\boldsymbol{\omega_{i}}}^{C}$.\\
    Furthermore, the entire full subset chain of $\tilde{\Omega}$ is denoted as $\mathcal{C}(\tilde{\Omega})$. 
\end{definition}
Definition~\ref{def:FullSubsetChain} aims to find an inclusive but separable relation between the subsets of belief states. Separability (if $\Omega_{\boldsymbol{\omega_{i}}}^{C}\subsetneq \Omega_{\boldsymbol{\omega_{j}}}^{C}$, then $\boldsymbol{\omega_{j}}\notin\Omega_{\boldsymbol{\omega_{i}}}^{C}$) ensures that different belief states have different Whittle indices while inclusiveness ($\forall\boldsymbol{\omega_{i}},\boldsymbol{\omega_{j}}\in\Omega(\boldsymbol{\omega_{0}}), \boldsymbol{\omega_{i}}\neq\boldsymbol{\omega_{j}}$, there must be $\Omega_{\boldsymbol{\omega_{i}}}^{C}\subsetneq \Omega_{\boldsymbol{\omega_{j}}}^{C}$ or $\Omega_{\boldsymbol{\omega_{j}}}^{C}\subsetneq \Omega_{\boldsymbol{\omega_{i}}}^{C}$) leads to a well-order by Whittle indices detailed below.

Now assume $\tilde{\Omega}$ satisfies $\mathcal{C}(\tilde{\Omega})\neq\emptyset$. For $C\in\mathcal{C}(\tilde{\Omega})$, we can use the equation (\ref{eqn:lambdaExpress}) to formally construct variable $\lambda_{\boldsymbol{\omega_{i}}}^{C}=\frac{W_{\boldsymbol{\omega_{i}}}^{\Omega_{\boldsymbol{\omega_{i}}}^{C}}}{A_{\boldsymbol{\omega_{i}}}^{\Omega_{\boldsymbol{\omega_{i}}}^{C}}}$. With the above concept, we provide the definition of PCL-indexability for countable state problems as follows:
\begin{definition}[PCL-indexability]\label{def:PCLIndexability}
    We call the problem PCL-indexable w.r.t. $\tilde{\Omega}$ if\\
    (i) The problem satisfies PCL (Definition~\ref{def:PCL}) w.r.t. $\tilde{\Omega}$;\\
    (ii) $\exists! C\in\mathcal{C}(\tilde{\Omega})$, s.t.$\forall \boldsymbol{\omega_{i}}\neq\boldsymbol{\omega_{j}}$ satisfying $\Omega_{\boldsymbol{\omega_{i}}}^{C}\subsetneq \Omega_{\boldsymbol{\omega_{j}}}^{C}$, there is $\lambda_{\boldsymbol{\omega_{i}}}^{C}<\lambda_{\boldsymbol{\omega_{j}}}^{C}$. ($\exists!$ denotes unique existence.)
\end{definition}

Condition~(i) in Definition~\ref{def:PCLIndexability} characterizes part of the boundary of the achievable region through the constraints between performance measures for an infinite dimensional linear programming problem, and condition~(ii) means that the objective function can take its maximum value at a unique vertex in the achievable region while the policy corresponding to this vertex is a state priority policy obtained by sorting according to the size of the index $\lambda_{\boldsymbol{\omega}}^{C}$. Next, we show that if the problem is PCL-indexable w.r.t. $\tilde{\Omega}$, then we can construct a unique full subset chain of $\tilde{\Omega}_{0}:=\tilde{\Omega}\cup\{\{0\}\cup\Omega,\Omega\in\tilde{\Omega}\}$. Assume $\exists! C\in\mathcal{C}(\tilde{\Omega})$, s.t.$\forall\boldsymbol{\omega_{i}}\neq\boldsymbol{\omega_{j}}$ with $\Omega_{\boldsymbol{\omega_{i}}}^{C}\subsetneq\Omega_{\boldsymbol{\omega_{j}}}^{C}$, we have $\lambda_{\boldsymbol{\omega_{i}}}^{C}<\lambda_{\boldsymbol{\omega_{j}}}^{C}$. We can construct the full subset chain $C'$ of $\tilde{\Omega}_{0}$ as follows:
\begin{equation*}
    \Omega_{\boldsymbol{\omega}}^{C'}=\begin{cases}
    \Omega_{\boldsymbol{\omega}}^{C}\cup\{0\},&\lambda_{\boldsymbol{\omega}}^{C}>\lambda\\
    \Omega_{\boldsymbol{\omega}}^{C},&\lambda_{\boldsymbol{\omega}}^{C}\leq\lambda
    \end{cases}
\end{equation*}
and
\begin{equation*}
    \Omega_{0}^{C'}=\{0\}\cup\bigcup\limits_{\boldsymbol{\omega}:\lambda_{\boldsymbol{\omega}}\leq\lambda}\Omega_{\boldsymbol{\omega}}^{C}.
\end{equation*}
Without loss of generality, here we can assume $\bigcup\limits_{\boldsymbol{\omega}:\lambda_{\boldsymbol{\omega}}\leq\lambda}\Omega_{\boldsymbol{\omega}}^{C}\in\tilde{\Omega}$, otherwise we just need to add it to $\tilde{\Omega}$. Under such a construction, it is easy to verify that $C'$ is a full subset chain of $\tilde{\Omega}_{0}$, and it satisfies the second condition of Definition~\ref{def:PCLIndexability}. Furthermore, we have the following characterization of the unique full subset chain of $\tilde{\Omega}_{0}$:
\begin{lemma}\label{prop:ChainDescribe}
Assume the problem is PCL-indexable w.r.t. $\tilde{\Omega}$, then the unique full subset chain constructed above has the following index characterization:
\begin{equation}
\Omega_{\boldsymbol{\omega_{i}}}=\{\boldsymbol{\omega_{k}}\in\Omega(\boldsymbol{\omega_{0}})\cup\{0\}:\lambda_{\boldsymbol{\omega_{k}}}\leq\lambda_{\boldsymbol{\omega_{i}}}\}\label{eqn:ChainDescribe}.
\end{equation}
\end{lemma}
\proof
If $\Omega_{\boldsymbol{\omega_{i}}}$ has the form of (\ref{eqn:ChainDescribe}), then it clearly satisfies the definition of the full subset chain. We will prove that the full subset chain that satisfies condition~(ii) in Definition~\ref{def:PCLIndexability} is exactly in the form of (\ref{eqn:ChainDescribe}). Assume that there exists $\boldsymbol{\omega_{j}}\in\Omega(\boldsymbol{\omega_{0}})$ such that $\lambda_{\boldsymbol{\omega_{j}}}>\lambda_{\boldsymbol{\omega_{i}}}$, so that $\boldsymbol{\omega_{j}}\in\Omega_{\boldsymbol{\omega_{i}}}$. Since $\boldsymbol{\omega_{j}}\in\Omega_{\boldsymbol{\omega_{j}}}$, by condition~(ii) in Definition~\ref{def:FullSubsetChain}, $\Omega_{\boldsymbol{\omega_{j}}}\subsetneq \Omega_{\boldsymbol{\omega_{i}}}$. By condition~(ii) in Definition~\ref{def:PCLIndexability}, we have $\lambda_{\boldsymbol{\omega_{j}}}<\lambda_{\boldsymbol{\omega_{i}}}$, this contradicts with $\lambda_{\boldsymbol{\omega_{j}}}>\lambda_{\boldsymbol{\omega_{i}}}$. Similarly, assume that there exists $\boldsymbol{\omega_{j}}\in\Omega(\boldsymbol{\omega_{0}})$ such that $\lambda_{\boldsymbol{\omega_{j}}}<\lambda_{\boldsymbol{\omega_{i}}}$, but $\boldsymbol{\omega_{j}}\notin \Omega_{\boldsymbol{\omega_{i}}}$, then $\boldsymbol{\omega_{j}}\in\Omega_{\boldsymbol{\omega_{j}}}$ implies that $\Omega_{\boldsymbol{\omega_{i}}}\subsetneq \Omega_{\boldsymbol{\omega_{j}}}$, which contradicts with the hypothesis.
\endproof
Intuitively, Lemma~\ref{prop:ChainDescribe} is just a formal statement that the full subset chain defined in Definition~\ref{def:FullSubsetChain} with a well-ordered index function defined in Definition~\ref{def:PCLIndexability} leads to the construction of a series of monotonically increasing state subsets such that each state has a well-defined priority index.

By combining Theorem~\ref{prop:DecompLaw} and Lemma~\ref{prop:ChainDescribe}, the following relationship can be established:
\begin{lemma}\label{prop:AWconnection}
    $\left(A_{\boldsymbol{\omega_{j}}}^{\Omega_{\boldsymbol{\omega_{i}}}}-A_{\boldsymbol{\omega_{j}}}^{\Omega_{\boldsymbol{\omega_{i}}}\backslash\{\boldsymbol{\omega_{i}}\}}\right)\frac{W_{\boldsymbol{\omega_{i}}}^{\Omega_{\boldsymbol{\omega_{i}}}}}{A_{\boldsymbol{\omega_{i}}}^{\Omega_{\boldsymbol{\omega_{i}}}}}=W_{\boldsymbol{\omega_{j}}}^{\Omega_{\boldsymbol{\omega_{i}}}}-W_{\boldsymbol{\omega_{j}}}^{\Omega_{\boldsymbol{\omega_{i}}}\backslash\{\boldsymbol{\omega_{i}}\}}$.
\end{lemma}
\proof
First we consider the case that $\boldsymbol{\omega_{i}}, \boldsymbol{\omega_{j}}\in\Omega(\boldsymbol{\omega_{0}})$. Let $\Omega=\Omega_{\boldsymbol{\omega_{i}}},\pi=\pi_{\Omega_{\boldsymbol{\omega_{i}}}^{c}\cup\{\boldsymbol{\omega_{i}}\}}$ in Theorem~\ref{prop:DecompLaw}, then
\begin{align}
    T_{\boldsymbol{\omega_{j}}}^{\Omega_{\boldsymbol{\omega_{i}}}^{c}\cup\{\boldsymbol{\omega_{i}}\}}&=T_{\boldsymbol{\omega_{j}}}^{\Omega_{\boldsymbol{\omega_{i}}}^{c}}+A_{\boldsymbol{\omega_{i}}}^{\Omega_{\boldsymbol{\omega_{i}}}}x_{\boldsymbol{\omega_{i}}}^{1}(\pi_{\Omega_{\boldsymbol{\omega_{i}}}^{c}\cup\{\boldsymbol{\omega_{i}}\}},\boldsymbol{\omega_{j}}),\\
    R_{\boldsymbol{\omega_{j}}}^{\Omega_{\boldsymbol{\omega_{i}}}^{c}\cup\{\boldsymbol{\omega_{i}}\}}&=R_{\boldsymbol{\omega_{j}}}^{\Omega_{\boldsymbol{\omega_{i}}}^{c}}+W_{\boldsymbol{\omega_{i}}}^{\Omega_{\boldsymbol{\omega_{i}}}}x_{\boldsymbol{\omega_{i}}}^{1}(\pi_{\Omega_{\boldsymbol{\omega_{i}}}^{c}\cup\{\boldsymbol{\omega_{i}}\}},\boldsymbol{\omega_{j}}).
\end{align}
Thus we have
\begin{equation}
    W_{\boldsymbol{\omega_{i}}}^{\Omega_{\boldsymbol{\omega_{i}}}}(T_{\boldsymbol{\omega_{j}}}^{\Omega_{\boldsymbol{\omega_{i}}}^{c}\cup\{\boldsymbol{\omega_{i}}\}}-T_{\boldsymbol{\omega_{j}}}^{\Omega_{\boldsymbol{\omega_{i}}}^{c}})=A_{\boldsymbol{\omega_{i}}}^{\Omega_{\boldsymbol{\omega_{i}}}}(R_{\boldsymbol{\omega_{j}}}^{\Omega_{\boldsymbol{\omega_{i}}}^{c}\cup\{\boldsymbol{\omega_{i}}\}}-R_{\boldsymbol{\omega_{j}}}^{\Omega_{\boldsymbol{\omega_{i}}}^{c}}).
\end{equation}
Since the above equation holds for any $\boldsymbol{\omega_{j}}\in\Omega(\boldsymbol{\omega_{0}})$, we have
\begin{equation*}
    W_{\boldsymbol{\omega_{i}}}^{\Omega_{\boldsymbol{\omega_{i}}}}\sum\limits_{\boldsymbol{\omega_{j}}\in\Omega(\boldsymbol{\omega_{0}})}(p_{ij}^{1}-p_{ij}^{0})(T_{\boldsymbol{\omega_{j}}}^{\Omega_{\boldsymbol{\omega_{i}}}^{c}\cup\{\boldsymbol{\omega_{i}}\}}-T_{\boldsymbol{\omega_{j}}}^{\Omega_{\boldsymbol{\omega_{i}}}^{c}})=A_{\boldsymbol{\omega_{i}}}^{\Omega_{\boldsymbol{\omega_{i}}}}\sum\limits_{\boldsymbol{\omega_{j}}\in\Omega(\boldsymbol{\omega_{0}})}(p_{ij}^{1}-p_{ij}^{0})(R_{\boldsymbol{\omega_{j}}}^{\Omega_{\boldsymbol{\omega_{i}}}^{c}\cup\{\boldsymbol{\omega_{i}}\}}-R_{\boldsymbol{\omega_{j}}}^{\Omega_{\boldsymbol{\omega_{i}}}^{c}}).
\end{equation*}
From the dynamic programming equations we obtain
\begin{equation*}
    W_{\boldsymbol{\omega_{i}}}^{\Omega_{\boldsymbol{\omega_{i}}}}(A_{\boldsymbol{\omega_{j}}}^{\Omega_{\boldsymbol{\omega_{i}}}}-A_{\boldsymbol{\omega_{j}}}^{\Omega_{\boldsymbol{\omega_{i}}}\backslash\{\boldsymbol{\omega_{i}}\}})=A_{\boldsymbol{\omega_{i}}}^{\Omega_{\boldsymbol{\omega_{i}}}}(W_{\boldsymbol{\omega_{j}}}^{\Omega_{\boldsymbol{\omega_{i}}}}-W_{\boldsymbol{\omega_{j}}}^{\Omega_{\boldsymbol{\omega_{i}}}\backslash\{\boldsymbol{\omega_{i}}\}}).
\end{equation*}
If $\boldsymbol{\omega_{j}}=0$, then we have $A_{\boldsymbol{\omega_{j}}}^{\Omega_{\boldsymbol{\omega_{i}}}}=A_{\boldsymbol{\omega_{j}}}^{\Omega_{\boldsymbol{\omega_{i}}}\backslash{\{\boldsymbol{\omega_{i}}\}}}=1$ and $W_{\boldsymbol{\omega_{j}}}^{\Omega_{\boldsymbol{\omega_{i}}}}=W_{\boldsymbol{\omega_{j}}}^{\Omega_{\boldsymbol{\omega_{i}}}\backslash{\{\boldsymbol{\omega_{i}}\}}}=\lambda$; if $\boldsymbol{\omega_{j}}\in\Omega(\boldsymbol{\omega_{0}})$ and $\boldsymbol{\omega_{i}}=0$, then by definition, $A_{\boldsymbol{\omega_{j}}}^{\Omega_{0}}=A_{\boldsymbol{\omega_{j}}}^{\Omega_{0}\backslash{\{0\}}}$ and $W_{\boldsymbol{\omega_{j}}}^{\Omega_{0}}=W_{\boldsymbol{\omega_{j}}}^{\Omega_{0}\backslash{\{0\}}}$. In both cases, the lemma is clearly true.
\endproof

\subsection{Optimality of Index Policy}
Now we address the main issue whether the sequence of $\{\lambda_{\boldsymbol{\omega}}\}$-priorities  when the problem satisfies PCL-indexability achieves optimality. Since the problem satisfies PCL, the partial conservation law inequality describes a region for an infinite-dimensional linear programming problem. First, we give a brief introduction to the basic theory of infinite-dimensional linear programming. Let $X$ and $Y$ be real linear vector spaces and the corresponding positive cones are $P_{X}$ and $P_{Y}$. Let $X$ be partially ordered by relation $\geq$, defined by $x\geq y$ if $x-y\in P_{X}$. Let $X^{*}$ be the dual space of $X$. For any $x^{*}\in X^{*}$, $x^{*}$ is a linear functional on $X$. Define operator $\langle x,x^{*}\rangle$ to be the image of $x\in X$ under $x^{*}\in X^{*}$. Similarly, space $Y$ for the dual problem also induces $Y^{*}$ and $P_{Y^{*}}$. Let $A$ be a linear map from $X$ to $Y$, $y\in Y$ and $x^{*}\in X^{*}$. The primal linear problem is
\begin{equation}\label{LP1}
    \begin{aligned}
        &\max_{x}\langle x, x^{*}\rangle\\
        s.t.\quad&Ax-y\in P_{Y}\\
        &x\in P_{X}
    \end{aligned}\tag{LP}
\end{equation}
Let $A^{*}$ be the dual linear map from $Y^{*}$ to $X^{*}$, then the dual linear program can be written as:
\begin{equation}\label{LD1}
    \begin{aligned}
        &\min_{y^{*}}\langle y, y^{*}\rangle\\
        s.t.\quad&A^{*}y^{*}-x^{*}\in P_{X^{*}}\\
        &-y^{*}\in P_{Y^{*}}
    \end{aligned}\tag{LD}
\end{equation}
Similar to the case of finite dimensions, a weak duality relation can be established between (\ref{LP1}) and (\ref{LD1}).
\begin{theorem}{\citep{anderson1987linear}}
Assume $x$ is feasible for (\ref{LP1}) and $y^{*}$ is feasible for (\ref{LD1}), then
\begin{equation}
    \langle y,y^{*}\rangle\geq\langle x,x^{*}\rangle.
\end{equation}
Equality holds if and only if
\begin{equation}
    \langle y-Ax,y^{*}\rangle=0,\quad\langle x,A^{*}y^{*}-x^{*}\rangle=0
\end{equation}
in which case we say that $x,y^{*}$ are complementary slack. If $x,y^{*}$ are feasible and complementary slack then $x$ and $y^{*}$ are optimal solutions to (\ref{LP1}) and (\ref{LD1}) respectively.
\end{theorem}

Next, we aim to prove that if the problem satisfies PCL-indexability w.r.t. $\tilde{\Omega}$, then the problem is also Whittle indexable, i.e., PCL-indexability is a sufficient condition for Whittle indexability. First, we obtain the optimality of the $\{\lambda_{\boldsymbol{\omega}}\}$-priority policy in the two-arm system. %This part is based on the discussion of optimality of Gittins index for countable state spaces in \citep{frostig1999four}. 
Assume the problem is PCL-indexable w.r.t. $\tilde{\Omega}$. For simplicity, we now denote the joint state space $\{0\}\cup\Omega(\boldsymbol{\omega_{0}})$ by $\Omega(\boldsymbol{\omega_{0}})$, i.e. we view state~0 as a part of the whole two-arm system's state space and denote $\tilde{\Omega}\cup\{\{0\}\cup\Omega:\Omega\in\tilde{\Omega}\}$ by $\tilde{\Omega}$. Consider the space $X=\{\{x_{\boldsymbol{\omega_{i}}}^{1}\}_{\boldsymbol{\omega_{i}}\in\Omega(\boldsymbol{\omega_{0}})}\vert\sum\limits_{\boldsymbol{\omega_{i}}\in\Omega(\boldsymbol{\omega_{0}})}\vert x_{\boldsymbol{\omega_{i}}}^{1}\vert<+\infty\}$. For any reward vector $\boldsymbol{R}=(R_{\boldsymbol{\omega_{i}}})_{\boldsymbol{\omega_{i}}\in\Omega(\boldsymbol{\omega_{0}})}$, define the inner product $\langle \boldsymbol{R},x^{1}\rangle:=\sum\limits_{\boldsymbol{\omega_{i}}\in\Omega(\boldsymbol{\omega_{0}})}x_{\boldsymbol{\omega_{i}}}^{1}R_{\boldsymbol{\omega_{i}}}$. Thus $\boldsymbol{R}\in X^{*}$. Furthermore, consider the variable space $Y$ of the dual problem, where $Y$ is a space of set functions on $\tilde{\Omega}$. Let $b\in Y$, for $\Omega\in\tilde{\Omega}$, define $b(\Omega)$ as
\begin{equation*}
    b(\Omega):=\inf\limits_{\pi}\{\sum_{\boldsymbol{\omega}\in\Omega}A_{\boldsymbol{\omega}}^{\Omega}x_{\boldsymbol{\omega}}^{1}(\pi):\pi\in\Pi\}.
\end{equation*}
For $A:X\rightarrow Y$, define $Ax^{1}(\Omega)=\sum_{\boldsymbol{\omega}\in \Omega}A_{\boldsymbol{\omega}}^{\Omega}x_{\boldsymbol{\omega}}^{1},\Omega\in\tilde{\Omega}$. Let $Y$ be the linear vector space spanned by $b$ and $Ax^{1},~x^{1}\in X$. The positive cone is $P_{Y}=\{y\in Y:y(\Omega)\geq 0,\Omega\subset\Omega(\boldsymbol{\omega_{0}}),y(\Omega(\boldsymbol{\omega_{0}}))=0\}$. By Lemma~\ref{prop:ChainDescribe}, $\Omega_{\boldsymbol{\omega_{i}}}=\{\boldsymbol{\omega_{j}}\in\Omega(\boldsymbol{\omega_{0}}):\lambda_{\boldsymbol{\omega_{j}}}\leq\lambda_{\boldsymbol{\omega_{i}}}\}$. Let $\Omega_{\boldsymbol{\omega_{i}}}^{-}:=\{\boldsymbol{\omega_{j}}\in\Omega(\boldsymbol{\omega_{0}}):\lambda_{\boldsymbol{\omega_{j}}}<\lambda_{\boldsymbol{\omega_{i}}}\}$. In infinite dimensional problems, for any $y\in Y$, define $\lambda^{*}\in Y^{*}$ as follows:
\begin{equation*}
    \langle y,\lambda^{*}\rangle:=\sum_{\boldsymbol{\omega_{i}}\in\Omega(\boldsymbol{\omega_{0}})}\lambda_{\boldsymbol{\omega_{i}}}\left[y(\Omega_{\boldsymbol{\omega_{i}}})-y(\Omega_{\boldsymbol{\omega_{i}}}^{-})\right].
\end{equation*}
The above~$\lambda^*$ is well defined since the right side of the equation is a linear functional on~$Y$ and two different linear functionals differ at least at one point. Let $x^{W}\in X$ be the performance measure given by the $\{\lambda_{\boldsymbol{\omega_{i}}}\}_{\boldsymbol{\omega_{i}}\in\Omega(\boldsymbol{\omega_{0}})}$-priority policy. Our goal is to prove that $x^{W}$ and $\lambda^{*}$ are optimal solutions to (\ref{LP1}) and (\ref{LD1}), respectively. Due to the assumption that the problem satisfies PCL-indexability w.r.t. $\tilde{\Omega}$, it satisfies the PCL condition and $x^{W}$ is a feasible solution of (\ref{LP1}). By the weak duality principle, we need to prove the following two statements:

\noindent
(i) $\lambda^{*}$ is a feasible solution to (\ref{LD1});\\
(ii) $\langle Ax^{W}-b,\lambda^{*}\rangle=0$ and $\langle x^{W},A^{*}\lambda^{*}-\boldsymbol{R}\rangle=0$.

For the first complementary slackness condition in~(ii), by the definition of $x^{W}$ and $\lambda^{*}$, we have
\begin{equation*}
    \begin{aligned}
        \langle Ax^{W},\lambda^{*}\rangle&=\sum_{\boldsymbol{\omega_{i}}\in\Omega(\boldsymbol{\omega_{0}})}\lambda_{\boldsymbol{\omega_{i}}}(Ax^{W}(\Omega_{\boldsymbol{\omega_{i}}})-Ax^{W}(\Omega_{\boldsymbol{\omega_{i}}}^{-}))\\
        &=\sum_{\boldsymbol{\omega_{i}}\in\Omega(\boldsymbol{\omega_{0}})}\lambda_{\boldsymbol{\omega_{i}}}\left[b(\Omega_{\boldsymbol{\omega_{i}}})-b(\Omega_{\boldsymbol{\omega_{i}}}^{-})\right]\\
        &=\langle b,\lambda^{*}\rangle.
    \end{aligned}
\end{equation*}
Thus $\langle Ax^{W}-b,\lambda^{*}\rangle=0$ holds. The establishment of the second complementary slackness condition requires an important equation that is trivially satisfied in finite-state problems. However, for the case of countable states considered here, the situation is much more complex and we prove it in the theorem below based on the following assumption.
%\textcolor{red}{
\begin{assumption}\label{assump:FiniteLimitPoint}
The sequence $\{\lambda_{\boldsymbol{\omega}}\}_{\boldsymbol{\omega}\in\Omega(\boldsymbol{\omega_{0}})\cup\{0\}}$ has finitely many limit points.
\end{assumption}
%In fact, since state update rules are themselves finite and the consideration of stationary policies is often sufficient for long-term optimization problems, the above assumption is naturally satisfied in practice. To further elucidate this assumption, it can be interpreted that the belief space can be decomposed into a finite number of equivalent classes, in which each class exhibits ergodic properties. Within each equivalent class, belief states are expected to display similar long-term statistical behavior. Based on this assumption, we give the following important equation.
%}
\begin{theorem}\label{prop:W-expansion}
    $R_{\boldsymbol{\omega_{j}}}=W_{\boldsymbol{\omega_{j}}}^{\Omega_{\boldsymbol{\omega_{j}}}}+\sum\limits_{\boldsymbol{\omega_{i}}:\lambda_{\boldsymbol{\omega_{i}}}>\lambda_{\boldsymbol{\omega_{j}}}}(W_{\boldsymbol{\omega_{j}}}^{\Omega_{\boldsymbol{\omega_{i}}}}-W_{\boldsymbol{\omega_{j}}}^{\Omega_{\boldsymbol{\omega_{i}}}^{-}})$.
\end{theorem}
In order to prove this theorem, we first establish an easy lemma. For simplicity, we denote the state space $\Omega_{\boldsymbol{\omega_{0}}}\cup\{0\}$ by $\mathcal{M}$ and state $\boldsymbol{\omega_{i}}$ by $i$.
\begin{lemma}\label{lem:RConverge}
    For $S\subset\mathcal{M}$, let $R_{T,i}^{S}$ denote the expected discounted reward obtained from initial belief state $i$ under the $S$-priority policy over the first $T$ time steps. Then we have $\lim\limits_{T\rightarrow+\infty}R_{T,i}^{S}=R_{i}^{S}$. In fact, this limit process is independent of $i$, i.e. $R_{T,i}^{S}\rightrightarrows R_{i}^{S}$ over~$i$ as $T\rightarrow+\infty$. ($\rightrightarrows$ denotes uniform convergence.)
\end{lemma}
\proof
    Let $p_{T,ij}^{S}$ be the probability of being in state $j$ after $T$ time steps starting from state $i$ under the $S$-priority policy. Then we have
    \begin{equation*}
        R_{i}^{S}=R_{T,i}^{S}+\beta^{T}\sum\limits_{j\in\mathcal{M}}p_{T,ij}^{S}R_{j}^{S}.
    \end{equation*}
    Since $\vert\sum\limits_{j\in\mathcal{M}}p_{T,ij}^{S}R_{j}^{S}\vert\leq\sum\limits_{j\in\mathcal{M}}p_{T,ij}^{S}\vert R_{j}^{S}\vert\leq\frac{C}{1-\beta}$, we have $\vert R_{i}^{S}-R_{T,i}^{S}\vert\leq\frac{\beta^{T}C}{1-\beta}\rightarrow 0(T\rightarrow+\infty)$.
\Halmos
\endproof
To prove Theorem~\ref{prop:W-expansion}, we only need to consider the case where the sequence $\{\lambda_{i}\}_{i\in\mathcal{M}}$ has one limit point. For notational convenience, we renumber the states in $\mathcal{M}$ as $\{*,1,2,\cdots\}$. The state indices in the sequence have the relationship as shown in Fig.~\ref{fig:limit_point}.

\begin{figure}[htbp]
\centering
\caption{Limit point}\label{fig:limit_point}
\begin{tikzpicture}
  \draw[->] (0,0) -- (8,0);
  \foreach \x in {1,2,3,4,5,6,7}
    \draw (\x,0) -- (\x,0.1);
  \foreach \x/\label in {1/$\lambda_{2}$,2/$\lambda_{4}$,3/$\cdots$,4/$\lambda_{*}$,5/$\cdots$,6/$\lambda_{3}$,7/$\lambda_{1}$}
    \node at (\x,-0.3) {\label};
    \draw[<-] (4.5,0.5) -- (6.5,0.5);
    \draw[->] (1.5,0.5) -- (3.5,0.5);
\end{tikzpicture}
\end{figure}

In Fig.~\ref{fig:limit_point}, $\lambda_{*}$ is the unique limit point of $\{\lambda_{i}\}_{i\in\mathcal{M}}$ and $\{\lambda_{2n-1}\}\searrow\lambda_{*}$ while $\{\lambda_{2n}\}\nearrow\lambda_{*}$. Recall that $\Omega_{i}=\{j\in\mathcal{M}:\lambda_{j}\leq\lambda_{i}\}$, we have the following lemma:
\begin{lemma}
$\lim\limits_{n\rightarrow+\infty}R_{i}^{\Omega_{2n-1}^{c}}=R_{i}^{\Omega_{*}^{c}},~ \lim\limits_{n\rightarrow+\infty}R_{i}^{\Omega_{2n}^{c}}=R_{i}^{\Omega_{*}^{-c}},~ \forall i\in\mathcal{M}$.
\end{lemma}
\proof
    We first prove that $\lim\limits_{n\rightarrow+\infty}R_{i}^{\Omega_{2n-1}^{c}}=R_{i}^{\Omega_{*}^{c}}$. We will use induction to show that for any $T$, $\lim\limits_{T\rightarrow+\infty}R_{T,i}^{\Omega_{2n-1}^{c}}=R_{T,i}^{\Omega_{*}^{c}}$. When $T=1$, if $i\notin\Omega_{*}^{c}$, then for any $n$, $R_{1,i}^{\Omega_{2n-1}^{c}}=R_{1,i}^{\Omega_{*}^{c}}=0$; If $i\in\Omega_{*}^{c}$, since $\lambda_{*}$ is the limit point, $\exists N$ s.t. $i\in\Omega_{2n-1}^{c}$ when $n>N$. In this case, $R_{1,i}^{\Omega_{2n-1}^{c}}=R_{1,i}^{\Omega_{*}^{c}}=R_{i}$. Thus $\lim\limits_{n\rightarrow+\infty}R_{1,i}^{\Omega_{2n-1}^{c}}=R_{1,i}^{\Omega_{*}^{c}}$ holds for every $i\in\mathcal{M}$. \\
    When $T\geq 1$, we assume $\lim\limits_{n\rightarrow+\infty}R_{T,i}^{\Omega_{2n-1}^{c}}=R_{T,i}^{\Omega_{*}^{c}}$ holds for every $i\in\mathcal{M}$, then we need to consider the following two situations:\\
    (i) If $i\notin\Omega_{*}^{c}$, we have
    \begin{equation*}
        R_{T+1,i}^{\Omega_{2n-1}^{c}}=\beta\sum\limits_{j\in\mathcal{M}}p_{ij}^{0}R_{T,j}^{\Omega_{2n-1}^{c}}, R_{T+1,i}^{\Omega_{*}^{c}}=\beta\sum\limits_{j\in\mathcal{M}}p_{ij}^{0}R_{T,j}^{\Omega_{*}^{c}}.
    \end{equation*}
    Thus
    \begin{equation*}
        \vert R_{T+1,i}^{\Omega_{2n-1}^{c}}-R_{T+1,i}^{\Omega_{*}^{c}}\vert\leq\beta\sum\limits_{j\in\mathcal{M}}p_{ij}^{0}\vert R_{T,j}^{\Omega_{2n-1}^{c}}-R_{T,j}^{\Omega_{*}^{c}}\vert.
    \end{equation*}
    For any $\varepsilon>0$, since $p_{i*}^{0}+\sum\limits_{j=1}^{+\infty}p_{ij}^{0}=1$, $\exists J$ s.t. $\sum\limits_{j=J+1}^{+\infty}p_{ij}^{0}\leq\frac{\varepsilon(1-\beta)}{4\beta C}$. For $j=*,1,\cdots,J$, by the induction hypothesis, $\exists N(j)$ s.t. when $n>N(j)$, $\vert R_{T,j}^{\Omega_{2n-1}^{c}}-R_{T,j}^{\Omega_{*}^{c}}\vert\leq\frac{\varepsilon}{2\beta}$. Let $N=\max\limits_{j\in\{*,1,\cdots,J\}}N(j)$, then when $n>N$, we have
    \begin{align*}
        \vert R_{T+1,i}^{\Omega_{2n-1}^{c}}-R_{T+1,i}^{\Omega_{*}^{c}}\vert&\leq\beta\sum\limits_{j\in\mathcal{M}}p_{ij}^{0}\vert R_{T,j}^{\Omega_{2n-1}^{c}}-R_{T,j}^{\Omega_{*}^{c}}\vert\\
        &=\beta\sum\limits_{j\in\{*,1,\cdots,J\}}p_{ij}^{0}\vert R_{T,j}^{\Omega_{2n-1}^{c}}-R_{T,j}^{\Omega_{*}^{c}}\vert+\beta\sum\limits_{j=J+1}^{+\infty}p_{ij}^{0}\vert R_{T,j}^{\Omega_{2n-1}^{c}}-R_{T,j}^{\Omega_{*}^{c}}\vert\\
        &\leq\beta\cdot\left(\sum\limits_{j\in\{*,1,\cdots,J\}}p_{ij}^{0}\right)\cdot\frac{\varepsilon}{2\beta}+\beta\cdot\left(\sum\limits_{j=J+1}^{+\infty}p_{ij}^{0}\right)\cdot\frac{2C}{1-\beta}\\
        &\leq\beta\cdot\frac{\varepsilon}{2\beta}+\beta\cdot 2C\cdot\frac{\varepsilon}{4\beta C}=\varepsilon.
    \end{align*}
    Therefore, $\lim\limits_{n\rightarrow+\infty}R_{T+1,i}^{\Omega_{2n-1}^{c}}=R_{T+1,i}^{\Omega_{*}^{c}}$.\\
    (ii) If $i\in\Omega_{*}^{c}$, then $\exists N$ s.t. $i\in\Omega_{n}^{c}$ when $n>N$. In this case, we have
    \begin{equation*}
        R_{T+1,i}^{\Omega_{2n-1}^{c}}=R_{i}+\beta\sum\limits_{j\in\mathcal{M}}p_{ij}^{1}R_{T,j}^{\Omega_{2n-1}^{c}}, R_{T+1,i}^{\Omega_{*}^{c}}=R_{i}+\beta\sum\limits_{j\in\mathcal{M}}p_{ij}^{1}R_{T,j}^{\Omega_{*}^{c}}.
    \end{equation*}
    Thus
    \begin{equation*}
        \vert R_{T+1,i}^{\Omega_{2n-1}^{c}}-R_{T+1,i}^{\Omega_{*}^{c}}\vert\leq\beta\sum\limits_{j\in\mathcal{M}}p_{ij}^{1}\vert R_{T,j}^{\Omega_{2n-1}^{c}}-R_{T,j}^{\Omega_{*}^{c}}\vert.
    \end{equation*}
    Using the discussion similar to (i) one can derive that $\lim\limits_{n\rightarrow+\infty}R_{T+1,i}^{\Omega_{2n-1}^{c}}=R_{T+1,i}^{\Omega_{*}^{c}}$. By induction, we know that for any $T$, $\lim\limits_{n\rightarrow+\infty}R_{T,i}^{\Omega_{2n-1}^{c}}=R_{T,i}^{\Omega_{*}^{c}}$. \\
    Next we will prove that $\lim\limits_{n\rightarrow+\infty}R_{i}^{\Omega_{2n-1}^{c}}=R_{i}^{\Omega_{*}^{c}}$. $\forall\varepsilon>0$, let $T=\lfloor\frac{\log\varepsilon+\log(1-\beta)-\log C}{\log\beta}\rfloor$, then when $T_{2}>T_{1}>T$, for every $n$, we have
    \begin{equation*}
        R_{T_{2},i}^{\Omega_{2n-1}^{c}}=R_{T_{1},i}^{\Omega_{2n-1}^{c}}+\beta^{T_{1}}\sum\limits_{j\in\mathcal{M}}p_{T_{1},ij}^{\Omega_{2n-1}^{c}}R_{T_{2}-T_{1},j}^{\Omega_{2n-1}^{c}}.
    \end{equation*}
    Thus
    \begin{align*}
        \vert R_{T_{2},i}^{\Omega_{2n-1}^{c}}-R_{T_{1},i}^{\Omega_{2n-1}^{c}}\vert&=\beta^{T_{1}}\vert\sum\limits_{j\in\mathcal{M}}p_{T_{1},ij}^{\Omega_{2n-1}^{c}}R_{T_{2}-T_{1},j}^{\Omega_{2n-1}^{c}}\vert\\
        &\leq\frac{\beta^{T}C}{1-\beta}<\varepsilon.
    \end{align*}
    By Theorem 7.11 in \citet{rudin1976principles}, $\lim\limits_{n\rightarrow+\infty}\lim\limits_{T\rightarrow+\infty}R_{T,i}^{\Omega_{2n-1}^{c}}=\lim\limits_{T\rightarrow+\infty}\lim\limits_{n\rightarrow+\infty}R_{T,i}^{\Omega_{2n-1}^{c}}$. By Lemma~\ref{lem:RConverge},
    \begin{align*}
        \lim\limits_{n\rightarrow+\infty}R_{i}^{\Omega_{2n-1}^{c}}&=\lim\limits_{n\rightarrow+\infty}\lim\limits_{T\rightarrow+\infty}R_{T,i}^{\Omega_{2n-1}^{c}}\\
        &=\lim\limits_{T\rightarrow+\infty}\lim\limits_{n\rightarrow+\infty}R_{T,i}^{\Omega_{2n-1}^{c}}\\
        &=\lim\limits_{T\rightarrow+\infty}R_{T,i}^{\Omega_{*}^{c}}\\
        &=R_{i}^{\Omega_{*}^{c}},\quad\forall i\in\mathcal{M}.
    \end{align*}
    Using similar arguments, it can be shown that $\lim\limits_{n\rightarrow+\infty}R_{i}^{\Omega_{2n}^{c}}=R_{i}^{\Omega_{*}^{-c}}, ~\forall i\in\mathcal{M}$.
\endproof

\begin{lemma}\label{prop:converge}
$\forall i\in\mathcal{M}, ~\lim\limits_{n\rightarrow+\infty}W_{i}^{\Omega_{2n+1}}=W_{i}^{\Omega_{*}},~ \lim\limits_{n\rightarrow+\infty}W_{i}^{\Omega_{2n}}=W_{i}^{\Omega_{*}^{-}}, ~\lim\limits_{n\rightarrow+\infty}A_{i}^{\Omega_{2n+1}}=A_{i}^{\Omega_{*}}$, and $\lim\limits_{n\rightarrow+\infty}A_{i}^{\Omega_{2n}}=A_{i}^{\Omega_{*}^{-}}$.
\end{lemma}
\proof
    Here we only prove that $\lim\limits_{n\rightarrow+\infty}W_{i}^{\Omega_{2n+1}}=W_{i}^{\Omega_{*}}$, other identities can be proved similarly. By definition,
    \begin{equation*}
        W_{i}^{\Omega_{2n+1}}=R_{i}+\beta\sum\limits_{j\in\mathcal{M}}(p_{ij}^{1}-p_{ij}^{0})R_{j}^{\Omega_{2n+1}^{c}}, W_{i}^{\Omega_{*}}=R_{*}+\beta\sum\limits_{j\in\mathcal{M}}(p_{ij}^{1}-p_{ij}^{0})R_{j}^{\Omega_{*}^{c}}.
    \end{equation*}
    $\forall\varepsilon>0$, there $\exists J$, s.t. $\sum\limits_{j=J+1}^{+\infty}(p_{ij}^{1}+p_{ij}^{0})\leq\frac{\varepsilon(1-\beta)}{4\beta C}$. For $j=*,1,\cdots,J$, since $\lim\limits_{n\rightarrow+\infty}R_{j}^{\Omega_{2n+1}^{c}}=R_{j}^{\Omega_{*}^{c}}$, there $\exists N(j)$ s.t. when $n>N(j)$, $\vert R_{j}^{\Omega_{2n+1}^{c}}-R_{j}^{\Omega_{*}^{c}}\vert<\frac{\varepsilon}{4\beta}$. Let $N=\max\limits_{j\in\{*,1,\cdots,J\}}N(j)$, then when $n>N$, we have
    \begin{align*}
        \vert W_{i}^{\Omega_{2n+1}}-W_{i}^{\Omega_{*}}\vert&=\vert \beta\sum\limits_{j\in\mathcal{M}}(p_{ij}^{1}-p_{ij}^{0})(R_{j}^{\Omega_{2n+1}^{c}}-R_{j}^{\Omega_{*}^{c}})\vert\\
        &\leq\beta\sum\limits_{j\in\mathcal{M}}\vert p_{ij}^{1}-p_{ij}^{0}\vert\vert R_{j}^{\Omega_{2n+1}^{c}}-R_{j}^{\Omega_{*}^{c}}\vert\\
        &\leq\beta\sum\limits_{j\in\{*,1,\cdots,J\}}(p_{ij}^{1}+p_{ij}^{0})\vert R_{j}^{\Omega_{2n+1}^{c}}-R_{j}^{\Omega_{*}^{c}}\vert+\beta\sum\limits_{j=J+1}^{+\infty}(p_{ij}^{1}+p_{ij}^{0})\vert R_{j}^{\Omega_{2n+1}^{c}}-R_{j}^{\Omega_{*}^{c}}\vert\\
        &\leq\beta\cdot\frac{\varepsilon}{4\beta}\cdot\sum\limits_{j\in\{*,1,\cdots,J\}}(p_{ij}^{1}+p_{ij}^{0})+\beta\cdot\frac{2C}{1-\beta}\cdot\sum\limits_{j=J+1}^{+\infty}(p_{ij}^{1}+p_{ij}^{0})\\
        &\leq\frac{\varepsilon}{4}\cdot 2+\frac{2\beta C}{1\beta}\cdot\frac{\varepsilon(1-\beta)}{4\beta C}\\
        &=\frac{\varepsilon}{2}+\frac{\varepsilon}{2}=\varepsilon.
    \end{align*}
    Thus we have $\lim\limits_{n\rightarrow+\infty}W_{i}^{\Omega_{2n+1}}=W_{i}^{\Omega_{*}}$.
\endproof

\begin{proof}[Proof of Theorem~\ref{prop:W-expansion}]
We will separately consider three cases where $j$ is odd, $j=*$ and $j$ is even.\\
(i) If $j$ is odd, we can assume $j=2N-1$, then
\begin{align*}
    W_{j}^{\Omega_{j}}+\sum\limits_{i:\lambda_{i}>\lambda_{j}}(W_{j}^{\Omega_{i}}-W_{j}^{\Omega_{i}^{-}})&=W_{2N-1}^{\Omega_{2N-1}}+\sum\limits_{n=1}^{N}(W_{2N-1}^{\Omega_{2n-1}}-W_{2N-1}^{\Omega_{2n+1}})\\
    &=W_{2N-1}^{\Omega_{1}}=R_{2N-1}.
\end{align*}
(ii) If $j=*$, then by Lemma~\ref{prop:converge},
\begin{align*}
    W_{*}^{\Omega_{*}}+\sum\limits_{i:\lambda_{i}>\lambda_{*}}(W_{*}^{\Omega_{i}}-W_{*}^{\Omega_{i}^{-}})&=W_{*}^{\Omega_{*}}+\sum\limits_{n=1}^{+\infty}(W_{*}^{\Omega_{2n-1}}-W_{*}^{\Omega_{2n+1}})\\
    &=W_{*}^{\Omega_{*}}+\lim\limits_{N\rightarrow+\infty}\sum\limits_{n=1}^{N}(W_{*}^{\Omega_{2n-1}}-W_{*}^{\Omega_{2n+1}})\\
    &=W_{*}^{\Omega_{*}}+W_{*}^{\Omega_{1}}-\lim\limits_{N\rightarrow+\infty}W_{*}^{\Omega_{2N+1}}\\
    &=R_{*}.
\end{align*}
(iii) If $j$ is even, we can assume $j=2N$, then by Lemma~\ref{prop:converge}, we have:
\begin{figure}[htbp]
\centering
\begin{tikzpicture}
  \draw[->] (0,0) -- (8,0);
  \foreach \x in {1,2,3,4,5,6,7}
    \draw (\x,0) -- (\x,0.1);
  \foreach \x/\label in {1/$\lambda_{2N}$,2/$\lambda_{2N+2}$,3/$\cdots$,4/$\lambda_{*}$,5/$\cdots$,6/$\lambda_{3}$,7/$\lambda_{1}$}
    \node at (\x,-0.3) {\label};
    \draw[<-] (4.5,0.5) -- (6.5,0.5);
    \draw[->] (1.5,0.5) -- (3.5,0.5);
\end{tikzpicture}
\end{figure}
\begin{align*}
    \sum\limits_{i:\lambda_{i}>\lambda_{j}}(W_{j}^{\Omega_{i}}-W_{j}^{\Omega_{i}^{-}})&=\sum\limits_{n=1}^{+\infty}(W_{j}^{\Omega_{2n-1}}-W_{j}^{\Omega_{2n-1}^{-}})+\sum\limits_{n=N+1}^{+\infty}(W_{j}^{\Omega_{2n}}-W_{j}^{\Omega_{2n}^{-}})+W_{j}^{\Omega_{*}}-W_{j}^{\Omega_{*}^{-}}\\
    &=\sum\limits_{n=1}^{+\infty}(W_{j}^{\Omega_{2n-1}}-W_{j}^{\Omega_{2n+1}})+\sum\limits_{n=N+1}^{+\infty}(W_{j}^{\Omega_{2n}}-W_{j}^{\Omega_{2n-2}})+W_{j}^{\Omega_{*}}-W_{j}^{\Omega_{*}^{-}}\\
    &= W_{j}^{\Omega_{1}}-\lim\limits_{n\rightarrow+\infty}W_{j}^{\Omega_{2n+1}}-W_{j}^{\Omega_{j}}+\lim\limits_{n\rightarrow+\infty}W_{j}^{\Omega_{2n}}+W_{j}^{\Omega_{*}}-W_{j}^{\Omega_{*}^{-}}\\
    &= R_{j}-W_{j}^{\Omega_{j}}.
\end{align*}
It shows that $W_{j}^{\Omega_{j}}=R_{j}+\sum\limits_{i:\lambda_{i}>\lambda_{j}}(W_{j}^{\Omega_{i}}-W_{j}^{\Omega_{i}^{-}})$ and Theorem~\ref{prop:W-expansion} is proved.
\end{proof}

Theorem~\ref{prop:W-expansion} is a natural generalization of the key equation in finite-state problems to countable-state problems. Based on Theorem~\ref{prop:W-expansion}, we can establish the second relaxed complementary condition:
\begin{lemma}\label{prop:complementary1}
    $\langle x^{W},A^{*}\lambda^{*}-R\rangle=0$.
\end{lemma}
\proof
     By Lemma~\ref{prop:AWconnection} and Theorem~\ref{prop:W-expansion}, we have
    \begin{align*}
        A_{\boldsymbol{\omega_{j}}}^{\Omega_{\boldsymbol{\omega_{j}}}}\lambda_{\boldsymbol{\omega_{j}}}&=W_{\boldsymbol{\omega_{j}}}^{\Omega_{\boldsymbol{\omega_{j}}}}\\
        &=R_{\boldsymbol{\omega_{j}}}-\sum\limits_{\boldsymbol{\omega_{i}}:\lambda_{\boldsymbol{\omega_{i}}}>\lambda_{\boldsymbol{\omega_{j}}}}(W_{\boldsymbol{\omega_{j}}}^{\Omega_{\boldsymbol{\omega_{i}}}}-W_{\boldsymbol{\omega_{j}}}^{\Omega_{\boldsymbol{\omega_{i}}}^{-}})\\
        &=R_{\boldsymbol{\omega_{j}}}-\sum\limits_{\boldsymbol{\omega_{i}}:\lambda_{\boldsymbol{\omega_{i}}}>\lambda_{\boldsymbol{\omega_{j}}}}(A_{\boldsymbol{\omega_{j}}}^{\Omega_{\boldsymbol{\omega_{i}}}}-A_{\boldsymbol{\omega_{j}}}^{\Omega_{\boldsymbol{\omega_{i}}}^{-}})\lambda_{\boldsymbol{\omega_{i}}}.
    \end{align*}
    $\forall x^{1}\in X$, following from Proposition~4.3 in \citet{frostig1999four}, we have
    \begin{align*}
        &\langle x^{1},A^{*}\lambda^{*}\rangle=\langle Ax^{1},\lambda^{*}\rangle\\
        =&\sum_{\boldsymbol{\omega_{i}}\in\Omega(\omega_{0})}\lambda_{\boldsymbol{\omega_{i}}}\left[Ax^{1}(\Omega_{\boldsymbol{\omega_{i}}})-Ax^{1}(\Omega_{\boldsymbol{\omega_{i}}}^{-})\right]\\
        =&\sum_{\boldsymbol{\omega_{i}}\in\Omega(\boldsymbol{\omega_{0}})}\lambda_{\boldsymbol{\omega_{i}}}\left[\sum_{\boldsymbol{\omega_{j}}\in \Omega_{\boldsymbol{\omega_{i}}}}A_{\boldsymbol{\omega_{j}}}^{\Omega_{\boldsymbol{\omega_{i}}}}x^{1}_{\boldsymbol{\omega_{j}}}-\sum_{\boldsymbol{\omega_{j}}\in\Omega_{\boldsymbol{\omega_{i}}}^{-}}A_{\boldsymbol{\omega_{j}}}^{\Omega_{\boldsymbol{\omega_{i}}}^{-}}x_{\boldsymbol{\omega_{j}}}^{1}\right]\\
        =&\sum_{\boldsymbol{\omega_{i}}\in\Omega(\boldsymbol{\omega_{0}})}\lambda_{\boldsymbol{\omega_{i}}}A_{\boldsymbol{\omega_{i}}}^{\Omega_{\boldsymbol{\omega_{i}}}}x_{\boldsymbol{\omega_{i}}}^{1}-\sum_{\boldsymbol{\omega_{i}}\in\Omega(\boldsymbol{\omega_{0}})}\lambda_{\boldsymbol{\omega_{i}}}\sum_{\boldsymbol{\omega_{j}}:\lambda_{\boldsymbol{\omega_{j}}}<\lambda_{\boldsymbol{\omega_{i}}}}(A_{\boldsymbol{\omega_{j}}}^{\Omega_{\boldsymbol{\omega_{i}}}^{-}}-A_{\boldsymbol{\omega_{j}}}^{\Omega_{\boldsymbol{\omega_{i}}}})x_{\boldsymbol{\omega_{j}}}^{1}\\
        =&\sum_{\boldsymbol{\omega_{i}}\in\Omega(\boldsymbol{\omega_{0}})}\lambda_{\boldsymbol{\omega_{i}}}A_{\boldsymbol{\omega_{i}}}^{\Omega_{\boldsymbol{\omega_{i}}}}x_{\boldsymbol{\omega_{i}}}^{1}-\sum_{\boldsymbol{\omega_{j}}\in\Omega(\boldsymbol{\omega_{0}})}x_{\boldsymbol{\omega_{j}}}^{1}\sum_{\boldsymbol{\omega_{i}}:\lambda_{\boldsymbol{\omega_{i}}}>\lambda_{\boldsymbol{\omega_{j}}}}\lambda_{\boldsymbol{\omega_{i}}}(A_{\boldsymbol{\omega_{j}}}^{\Omega_{\boldsymbol{\omega_{i}}}^{-}}-A_{\boldsymbol{\omega_{j}}}^{\Omega_{\boldsymbol{\omega_{i}}}})\\
        =&\sum_{\boldsymbol{\omega_{j}}\in\Omega(\boldsymbol{\omega_{0}})}x_{\boldsymbol{\omega_{j}}}^{1}\left[\lambda_{\boldsymbol{\omega_{j}}}A_{\boldsymbol{\omega_{j}}}^{\Omega_{\boldsymbol{\omega_{j}}}}-\sum_{\boldsymbol{\omega_{i}}:\lambda_{\boldsymbol{\omega_{i}}}>\lambda_{\boldsymbol{\omega_{j}}}}\lambda_{\boldsymbol{\omega_{i}}}(A_{\boldsymbol{\omega_{j}}}^{\Omega_{\boldsymbol{\omega_{i}}}^{-}}-A_{\boldsymbol{\omega_{j}}}^{\Omega_{\boldsymbol{\omega_{i}}}})\right]\\
        =&\sum_{\boldsymbol{\omega_{j}}\in\Omega(\boldsymbol{\omega_{0}})}x_{\boldsymbol{\omega_{j}}}^{1}R_{\boldsymbol{\omega_{j}}}\\
        =&\langle x^{1},R\rangle.
    \end{align*}
    Thus $\langle x^{W},A^{*}\lambda^{*}-R\rangle=0$
\endproof

From the proof of Lemma~\ref{prop:complementary1}, we know that $\lambda^{*}$ satisfies $A^{*}\lambda^{*}-\boldsymbol{R}\in P_{X^{*}}$. Therefore, the remaining condition in the weak duality theorem is $-\lambda^{*}\in P_{Y^{*}}$. To prove this statement, we need to show that, $\forall y\in P_{Y}, \langle y,~\lambda^{*}\rangle\leq 0$. In fact,
\begin{align*}
    &\langle y,\lambda^{*}\rangle\\
    =&\sum\limits_{i\in\mathcal{M}}\lambda_{i}\left[y(\Omega_{i})-y(\Omega_{i}^{-})\right]\\
    =&\sum\limits_{n=1}^{+\infty}\lambda_{2n-1}\left[y(\Omega_{2n-1})-y(\Omega_{2n-1}^{-})\right]+\lambda_{*}\left[y(\Omega_{*})-y(\Omega_{*}^{-})\right]+\sum\limits_{n=1}^{+\infty}\lambda_{2n}\left[y(\Omega_{2n})-y(\Omega_{2n}^{-})\right]\\
    =&\sum\limits_{n=1}^{+\infty}\lambda_{2n-1}\left[y(\Omega_{2n-1})-y(\Omega_{2n+1})\right]+\lambda_{*}\left[y(\Omega_{*})-y(\Omega_{*}^{-})\right]+\sum\limits_{n=2}^{+\infty}\lambda_{2n}\left[y(\Omega_{2n})-y(\Omega_{2n-2})\right]\\
    &+\lambda_{2}y(\Omega_{2})\\
    =&\lambda_{1}y(\Omega_{1})+\sum\limits_{n=2}^{+\infty}(\lambda_{2n-1}-\lambda_{2n-3})y(\lambda_{2n-1})-\lim\limits_{n\rightarrow+\infty}\lambda_{2n-1}y(\Omega_{2n+1})+\lambda_{*}y(\Omega_{*})-\lambda_{*}y(\Omega_{*}^{-})\\
    &+\sum\limits_{n=1}^{+\infty}(\lambda_{2n}-\lambda_{2n+2})y(\Omega_{2n})+\lim\limits_{n\rightarrow+\infty}\lambda_{2n+2}y(\Omega_{2n+2}).
\end{align*}
Since $\lambda_{1}y(\Omega_{1})=0, ~y(\Omega)\geq 0, ~\lambda_{2n-1}\leq\lambda_{2n-3}, \lambda_{2n}\leq\lambda_{2n+2}$, we just need to prove that $\lambda_{*}y(\Omega_{*})\leq\lim\limits_{n\rightarrow+\infty}\lambda_{2n-1}y(\Omega_{2n+1})$ and $\lim\limits_{n\rightarrow+\infty}\lambda_{2n}y(\Omega_{2n})\leq\lambda_{*}y(\Omega_{*}^{-})$.
\begin{lemma}\label{prop:liminq}
For $y=Ax^{1}$ and $y=b$, we have
\begin{align*}
    &y(\Omega_{*})\leq\lim\limits_{n\rightarrow+\infty}y(\Omega_{2n+1}),\\
    &\lim\limits_{n\rightarrow+\infty}y(\Omega_{2n})\leq y(\Omega_{*}^{-}).
\end{align*}
\end{lemma}
\proof It is easy to see that $\Omega_{*}=\lim\limits_{n\rightarrow+\infty}\Omega_{2n-1}$ and $\Omega_{*}^{-}=\lim\limits_{n\rightarrow+\infty}\Omega_{2n}$. Now we will prove that $Ax^{1}(\Omega_{*})\leq\lim\limits_{n\rightarrow+\infty}Ax^{1}(\Omega_{2n+1})$ and $\lim\limits_{n\rightarrow+\infty}Ax^{1}(\Omega_{2n})\leq Ax^{1}(\Omega_{*}^{-})$. First we show that
\begin{equation*}
    \lim\limits_{n\rightarrow+\infty}\sum\limits_{i\in\Omega_{*}}A_{i}^{\Omega_{2n+1}}x_{i}^{1}=\sum\limits_{i\in\Omega_{*}}\lim\limits_{n\rightarrow+\infty}A_{i}^{\Omega_{2n+1}}x_{i}^{1}=\sum\limits_{i\in\Omega_{*}}A_{i}^{\Omega_{*}}x_{i}^{1}.
\end{equation*}
For any $\varepsilon>0$, $\sum\limits_{i\in\mathcal{M}}\vert x_{i}^{1}\vert<+\infty$ shows that $\sum\limits_{i=1}^{+\infty}\vert x_{2i}^{1}\vert<+\infty$, and $\exists$ constant $B$ s.t. $\vert x_{i}^{1}\vert\leq B$ holds for every $i\in\mathcal{M}$. Then $\exists I$, s.t. $\sum\limits_{i=I+1}^{+\infty}\vert x_{2n}^{1}\vert<\frac{(1-\beta)\varepsilon}{4(1+\beta)}$. Since $\lim\limits_{n\rightarrow+\infty}A_{i}^{\Omega_{2n-1}}=A_{i}^{\Omega_{*}}$, for any $1\leq i\leq I$, $\exists N(i)$ s.t. $\vert A_{2i}^{\Omega_{2n-1}}-A_{2i}^{\Omega_{*}}\vert<\frac{\varepsilon}{2(I+1)B}$ when $n>N(i)$. Similarly, $\exists N(*)$ s.t. $\vert A_{*}^{\Omega_{2n-1}}-A_{*}^{\Omega_{*}}\vert<\frac{\varepsilon}{2(I+1)B}$ when $n>N(*)$. Let $N=\max\limits_{i\in\{*,1,\cdots,I\}}N(i)$, then when $n>N$, we have
\begin{align*}
    &\vert \sum\limits_{i\in\Omega_{*}}A_{i}^{\Omega_{2n-1}}x_{i}^{1}-\sum\limits_{i\in\Omega_{*}}A_{i}^{\Omega_{*}}x_{i}^{1}\vert\\
    \leq&\sum\limits_{i\in\Omega_{*}}\vert A_{i}^{\Omega_{2n-1}}-A_{i}^{\Omega_{*}}\vert\vert x_{i}^{1}\vert\\
    =&\vert A_{*}^{\Omega_{2n-1}}-A_{*}^{\Omega_{*}}\vert\vert x_{*}^{1}\vert+\sum\limits_{i=1}^{I}\vert A_{2i}^{\Omega_{2n-1}}-A_{2i}^{\Omega_{*}}\vert\vert x_{2i}^{1}\vert+\sum\limits_{i=I+1}^{+\infty}\vert A_{2i}^{\Omega_{2n-1}}-A_{2i}^{\Omega_{*}}\vert\vert x_{2i}^{1}\vert\\
    \leq&\frac{\varepsilon}{2}+\frac{2(1+\beta)}{1-\beta}\cdot\frac{1-\beta}{4(1+\beta)}\cdot\varepsilon\\
    =&\frac{\varepsilon}{2}+\frac{\varepsilon}{2}=\varepsilon.
\end{align*}
So we obtain $\lim\limits_{n\rightarrow+\infty}\sum\limits_{i\in\Omega_{*}}A_{i}^{\Omega_{2n-1}}x_{i}^{1}=\sum\limits_{i\in\Omega_{*}}A_{i}^{\Omega_{*}}x_{i}^{1}$. Thus we have
\begin{align*}
    &\lim\limits_{n\rightarrow+\infty}Ax^{1}(\Omega_{2n-1})=\lim\limits_{n\rightarrow+\infty}\sum\limits_{i\in\Omega_{2n-1}}A_{i}^{\Omega_{2n-1}}x_{i}^{1}\\
    \geq&\lim\limits_{n\rightarrow+\infty}\sum\limits_{i\in\Omega_{*}}A_{i}^{\Omega_{2n-1}}x_{i}^{1}=\sum\limits_{i\in\Omega_{*}}A_{i}^{\Omega_{*}}x_{i}^{1}=Ax(\Omega_{*}).
\end{align*}
A similar process yields $\lim\limits_{n\rightarrow+\infty}Ax^{1}(\Omega_{2n})\leq Ax^{1}(\Omega_{*}^{-})$.
Next we prove
\begin{align*}
    &b(\Omega_{*})\leq\lim\limits_{n\rightarrow+\infty}b(\Omega_{2n-1}),\\
    &\lim\limits_{n\rightarrow+\infty}b(\Omega_{2n})\leq b(\Omega_{*}^{-}).
\end{align*}
In fact, by Theorem~\ref{prop:DecompLaw}, if the initial state is $\boldsymbol{\omega_{j}}\in\Omega(\boldsymbol{\omega_{0}})$, then for any $\Omega\subset\Omega(\boldsymbol{\omega_{0}})$, $b(\{0\}\cup\Omega)=\frac{1}{1-\beta}-T_{j}^{\Omega^{c}}$. As for $b(\Omega)=\inf\limits_{\pi}\{\sum\limits_{\boldsymbol{\omega}}A_{\boldsymbol{\omega}}^{\Omega}x_{\boldsymbol{\omega}}^{1}(\pi)\}$, when we always active the auxiliary arm, then $\sum\limits_{\boldsymbol{\omega}}A_{\boldsymbol{\omega}}^{\Omega}x_{\boldsymbol{\omega}}^{1}=0$. Combining this fact with $b(\Omega)\geq 0$, we have $b(\Omega)=0$. Next, we only need to consider whether state~$0$ is in $\Omega_{*}$ or not:\\
(i) If $0\in\Omega_{*}$, then $0\in\Omega_{2n-1}$ holds for every $n$. In this case,
\begin{equation*}
    b(\Omega_{*})=\frac{1}{1-\beta}-T_{j}^{\Omega_{*}^{c}},b(\Omega_{2n-1})=\frac{1}{1-\beta}-T_{j}^{\Omega_{2n-1}^{c}}.
\end{equation*}
Since $\lim\limits_{n\rightarrow+\infty}T_{j}^{\Omega_{2n-1}^{c}}=T_{j}^{\Omega_{*}^{c}}$, we have $b(\Omega_{*})=\lim\limits_{n\rightarrow+\infty}b(\Omega_{2n-1})$.\\
(ii) If $0\notin\Omega_{*}$, then $\exists N$, when $n>N$, $0\notin\Omega_{2n-1}$. In this case, $b(\Omega_{*})=b(\Omega_{2n-1})=0$.
In summary, we obtain $b(\Omega_{*})\leq\lim\limits_{n\rightarrow+\infty}b(\Omega_{2n+1})$. Similar arguments show that $\lim\limits_{n\rightarrow+\infty}b(\Omega_{2n})\leq b(\Omega_{*}^{-})$.
\endproof

Based on the previous discussion and Lemma~\ref{prop:liminq}, we arrive at:
\begin{lemma}\label{prop:feasiblelambda}
$-\lambda^{*}\in P_{Y^{*}}$
\end{lemma}
In summary, when the problem satisfies PCL-indexability and Assumption~\ref{assump:FiniteLimitPoint} holds, we have the following main theorem:
\begin{theorem}[PCL-indexability$\Rightarrow$Whittle indexability]\label{thm:PCLWhittle}
If the countable-state RMAB problem considered above is PCL-indexable and Assumption~\ref{assump:FiniteLimitPoint} holds, then the optimal policy for the two-arm problem is the priority policy given by $\{\lambda_{\boldsymbol{\omega_{i}}}\}_{\boldsymbol{\omega_{i}}\in\Omega(\boldsymbol{\omega_{0}})}$, i.e., for any $\boldsymbol{\omega_{i}},\boldsymbol{\omega_{j}}\in\Omega(\boldsymbol{\omega_{0}})$, if $\lambda_{\boldsymbol{\omega_{i}}}>\lambda_{\boldsymbol{\omega_{j}}}$, then the policy gives higher priority to state $\boldsymbol{\omega_{i}}$ than to state $\boldsymbol{\omega_{j}}$. Furthermore, the problem is Whittle indexable, and $\lambda_{\boldsymbol{\omega_{i}}}$ is exactly the Whittle index of the belief state $\boldsymbol{\omega_{i}}$.
\end{theorem}
\proof
From the weak duality theorem and above discussion, we know that the optimal policy is the priority policy given by $\{\lambda_{\boldsymbol{\omega_{i}}}\}$. It follows from $A_{0}^{\{0\}\cup\Omega}=1$ and $W_{0}^{\{0\}\cup\Omega}=\lambda$ that the priority index for state 0 is $\lambda$. Therefore, for state $\boldsymbol{\omega_{i}}$, when the system's subsidy $\lambda=\lambda_{\boldsymbol{\omega_{i}}}$, state $\boldsymbol{\omega}$ is precisely at the critical point for taking active or passive action under the optimal policy. Therefore, this two-arm problem is Whittle indexable, and the Whittle index for state $\boldsymbol{\omega_{i}}$ is $\lambda_{\boldsymbol{\omega_{i}}}$.
\endproof

\section{The Application of PCL in Approximate-State-Space Problems}\label{sec:approximatePCL}

In this section, we approximate the system state space based on the concept of $T$-step state space $\Omega(T\vert\boldsymbol{\omega})$ defined in Section~\ref{sec:CountableStateSpace} for effective applications of our PCL framework. 

\subsection{$T$-Step Approximate State Space}\label{sec:approxState}
If the repeated states are also taken into account, then $\vert\Omega(T\vert\boldsymbol{\omega})\vert=\frac{(H+1)^{T+1}-1}{H}$. In this case, the number of states will increase exponentially with the number of time steps~$T$. However, in practical situations, the phenomenon of state repetitions often occurs due to the long time horizon of the underlying stochastic processes in consideration. In other words, along the process of belief state updates, the updated state may have already appeared within a small number of steps from the start or have a short Euclidean distance to a previous state. In such cases, we can assume that the states encountered within a finite number of steps approximate the true belief state space reasonably well. In this sense, the cardinality of the $T$-step state space can be significantly reduced. Formally, we recursively define the $T$-step {\em approximate} state space as follows.
\begin{definition}[$T$-step approximate state space]\label{def:TStepApproxStateSpace}
    For any $\varepsilon>0$, belief state $\boldsymbol{\omega}$ and belief state space $\Omega$, we say $d(\boldsymbol{\omega},\Omega)\geq\varepsilon$, if $\forall\boldsymbol{\omega}'\in\Omega$, we have $\vert\vert\boldsymbol{\omega}-\boldsymbol{\omega}'\vert\vert\geq\varepsilon$. Define
    \begin{equation*}
        \tilde{\Omega}(T,\varepsilon\vert\boldsymbol{\omega})=\begin{cases}
        \{\mathcal{B}_{h_{1}}\cdots\mathcal{B}_{h_{T}}\boldsymbol{\omega}:0\leq h_{1},\cdots,h_{T}\leq H\text{ and}\\d(\mathcal{B}_{h_{1}}\cdots\mathcal{B}_{h_{T}}\boldsymbol{\omega},\tilde{\Omega}(T-1,\varepsilon\vert\boldsymbol{\omega}))\geq\varepsilon\}\cup\tilde{\Omega}(T-1,\varepsilon\vert\boldsymbol{\omega}),& T\geq2\\
        \{\mathcal{B}_{h}\boldsymbol{\omega}:0\leq h\leq H\text{ and }d(\mathcal{B}_{h}\boldsymbol{\omega},\{\boldsymbol{\omega}\})\geq\varepsilon\}\cup\{\boldsymbol{\omega}\},& T=1
        \end{cases}.
    \end{equation*}
    We call $\tilde{\Omega}(T\vert\boldsymbol{\omega},\epsilon)$ the $T$-step approximate state space under the initial belief state $\boldsymbol{\omega}$.
\end{definition}
Under the definition of $T$-step approximate state space, numerical calculations show that $\vert\tilde{\Omega}(T,\varepsilon\vert\boldsymbol{\omega})\vert$ is likely to be much smaller than $\vert\Omega(T\vert\boldsymbol{\omega})\vert$, which will greatly reduce the actual number of states that need to be considered for the problem. The following examples illustrate the above concept.
\begin{example}[Opportunistic Spectrum Access]
    Consider a cognitive radio system with a set of channels and a secondary user. Each channel has 2 states: busy ($S(t)=0$) or idle ($S(t)=1$), which evolves with time as a Markov chain. At each discrete time~$t$, the secondary user chooses one channel to sense and observe its state. Observations are error-prone, i.e., a busy channel may be sensed as idle, and vice versa. If the user senses the channel to be idle ($O(t)=1$) and the channel is indeed idle ($S(t)=1$), a unit reward will be accrued. Set the transition matrix of the Markov chain and observation error as follows.
   \[
\mathbf{P}=\begin{pmatrix}
0.8 & 0.2\\
0.2 & 0.8
\end{pmatrix},\qquad
\mathbf{O}=\begin{pmatrix}
0.8 & 0.2\\
0.2 & 0.8
\end{pmatrix},\qquad
\boldsymbol{\omega}_0=\boldsymbol{\pi}=(0.5,0.5),
\]
where $\boldsymbol{\pi}$ is the stationary distribution under $\mathbf{P}$.
For $T=6$, the $\varepsilon$-approximate belief space sizes are
$|\tilde{\Omega}(6,10^{-2}\mid\boldsymbol{\omega}_0)|=13$ and
$|\tilde{\Omega}(6,10^{-3}\mid\boldsymbol{\omega}_0)|=14$,
whereas $|\Omega(6\mid\boldsymbol{\omega}_0)|=1{,}093$.
For $T=10$, we further compute that
$\bigl|\tilde{\Omega}(10,10^{-2}\mid\boldsymbol{\omega}_0)\bigr|=15$ and
$\bigl|\tilde{\Omega}(10,10^{-3}\mid\boldsymbol{\omega}_0)\bigr|=21$,
whereas $\bigl|\Omega(10\mid\boldsymbol{\omega}_0)\bigr|=88{,}573$.
\end{example}
\begin{example}[Opportunistic Spectrum Access - Extended]
    Now we consider a 3-state channel model: low SNIR (signal to noise and interference ratio), \ie $S(t)=0$, middle SNIR ($S(t)=1$), and high SNIR ($S(t)=2$). The user attempts to send more data if the sensed channel is observed to have a higher SNIR. Set the transition matrix of the Markov chain and observation error as follows.
    \[
\mathbf{P}=\begin{pmatrix}
0.85 & 0.10 & 0.05\\
0.10 & 0.80 & 0.10\\
0.02 & 0.08 & 0.90
\end{pmatrix},\qquad
\mathbf{O}=\begin{pmatrix}
0.88 & 0.07 & 0.05\\
0.15 & 0.70 & 0.15\\
0.04 & 0.06 & 0.90
\end{pmatrix},\qquad
\boldsymbol{\omega}_0=\boldsymbol{\pi}=(0.26,0.30,0.44),
\]
where $\boldsymbol{\pi}$ is the stationary distribution under $\mathbf{P}$.
For $T=6$, we obtain
$|\tilde{\Omega}(6,10^{-2}\mid\boldsymbol{\omega}_0)|=24$ and
$|\tilde{\Omega}(6,10^{-3}\mid\boldsymbol{\omega}_0)|=25$,
whereas $|\Omega(6\mid\boldsymbol{\omega}_0)|=5{,}461$.
For $T=10$, we obtain
$|\tilde{\Omega}(10,10^{-2}\mid\boldsymbol{\omega}_0)|=36$ and
$|\tilde{\Omega}(10,10^{-3}\mid\boldsymbol{\omega}_0)|=41$,
whereas $|\Omega(10\mid\boldsymbol{\omega}_0)|=1{,}398{,}101$.
\end{example}
\begin{example}[Cybersecurity Monitoring and Detection]
    Consider a computer network where each time the administrator chooses a section (arm) in the network to monitor and detect the anomaly if any. The arm has 4 states: attacked ($S(t)=0$), malfunctional ($S(t)=1$), overloaded ($S(t)=2$), and normal ($S(t)=3$). Set the transition matrix of the Markov chain and observation error as follows.
    \[
\mathbf{P}=\begin{pmatrix}
0.75 & 0.12 & 0.08 & 0.05\\
0.08 & 0.72 & 0.12 & 0.08\\
0.05 & 0.10 & 0.70 & 0.15\\
0.03 & 0.04 & 0.08 & 0.85
\end{pmatrix},\qquad
\mathbf{O}=\begin{pmatrix}
0.88 & 0.07 & 0.04 & 0.01\\
0.10 & 0.80 & 0.07 & 0.03\\
0.05 & 0.10 & 0.78 & 0.07\\
0.02 & 0.04 & 0.07 & 0.87
\end{pmatrix},\qquad
\boldsymbol{\omega}_0=\boldsymbol{\pi}=(0.16,0.21,0.23,0.40),
\]
where $\boldsymbol{\pi}$ is the stationary distribution under $\mathbf{P}$.
For $T=6$, we obtain
$\bigl|\tilde{\Omega}(6,10^{-2}\mid\boldsymbol{\omega}_0)\bigr|=31$ and
$\bigl|\tilde{\Omega}(6,10^{-3}\mid\boldsymbol{\omega}_0)\bigr|=35$,
whereas $\bigl|\Omega(6\mid\boldsymbol{\omega}_0)\bigr|=19{,}531$.
For $T=10$, we obtain
$\bigl|\tilde{\Omega}(10,10^{-2}\mid\boldsymbol{\omega}_0)\bigr|=45$ and
$\bigl|\tilde{\Omega}(10,10^{-3}\mid\boldsymbol{\omega}_0)\bigr|=51$,
whereas $\bigl|\Omega(10\mid\boldsymbol{\omega}_0)\bigr|=12{,}207{,}031$.
\end{example}
%%%%%%%%%%%%%%%%
\begin{lemma}\label{obs1}
For any given $\varepsilon>0$, $\exists T$, such that $\forall t>T$ and $\boldsymbol{\omega}_{t}\in\Omega(t\vert\boldsymbol{\omega})$, $\exists\boldsymbol{\omega}_{T}\in\tilde{\Omega}(T\vert\boldsymbol{\omega})$, such that $\vert\vert\boldsymbol{\omega}_{t}-\boldsymbol{\omega}_{T}\vert\vert<\varepsilon$, where $\vert\vert\cdot\vert\vert$ is the Euclidean norm.
\end{lemma}
\proof
Given an arbitrary initial belief state $\boldsymbol{\omega}$, the belief state space grows as a tree in the compact simplex formed by all possible belief states. Therefore, we can cover the simplex by finitely many open balls with diameter equal to~$\epsilon$ (Theorem~2.41 in~\citet{rudin1976principles}). For each ball~$i$ containing at least one belief state that can be reached by the tree expansion, record the first time $T_i$ that the tree reaches a belief state in the ball. The maximum value of all $\{T_i\}$ gives the desired~$T$.
\endproof
%%%%%%%%%%%%%%%%%%%
\begin{lemma}\label{obs2}
For any given $\varepsilon>0$, $\lim\limits_{T\rightarrow+\infty}\vert\tilde{\Omega}(T,\varepsilon\vert\boldsymbol{\omega})\vert<+\infty$.
\end{lemma}
\proof
By Lemma~\ref{obs1}, there exists a~$T_0$ such that all belief states after that have a distance smaller than~$\epsilon$ to $\Omega(T_0\vert\boldsymbol{\omega})$ and $\lim\limits_{T\rightarrow+\infty}\vert\tilde{\Omega}(T,\varepsilon\vert\boldsymbol{\omega})\vert=\vert\Omega(T_0\vert\boldsymbol{\omega})\vert$ which is finite.
\endproof
By Lemmas~\ref{obs1} and~\ref{obs2}, for a sufficiently large number~$T$ of decision epochs, the belief states can be well approximated by a finite set of belief states within $\tilde{\Omega}(T\vert\boldsymbol{\omega})$. In practice, we have a numerical way of choosing the best~$\epsilon$ and~$T$ inspired by the idea of avoiding overfitting in neural network training. Specifically, we can randomly choose~$0<\epsilon<1$ and~$T\ge1$ at the start and conduct the Monte-Carlo Simulation over a period longer than $T$ to observe the performance. In the next round, we reduce $\epsilon$ by half and observe the performance. We repeatedly refine $\epsilon$ until the performance is no longer improving. Then we do the same procedure by doubling~$T$ in each round. In sum, we alternatively adjust~$\epsilon$ and~$T$ until the performance has no significant improvement and obtain the best~$\epsilon$ and~$T$. The detailed procedure is given in Algorithm~\ref{algor:optimalParams}.
%}

\begin{algorithm}[H]
  \caption{Adaptive selections of $\epsilon$ and $T$}\label{algor:optimalParams}
  \begin{algorithmic}[1]
    \State \textbf{Input:} initial $\epsilon \in (0,1)$, $T \ge 1$, performance tolerance $\delta_{\text{perf}} > 0$, simulation horizon $H > T$.
    \State \textbf{Definition:} let $J(\epsilon,T)$ denote the Monte--Carlo estimate of long-run performance over horizon $H$ under parameters $(\epsilon,T)$.
    \State Compute the baseline performance $J_{\text{best}} \gets J(\epsilon,T)$ \Comment{best performance observed so far}
    \Repeat
      \State Step~1. Refine $\epsilon$ with fixed $T$
      \State $\epsilon_{\text{cand}} \gets \epsilon/2$ \Comment{candidate $\epsilon$}
      \State $J_{\epsilon} \gets J(\epsilon_{\text{cand}}, T)$
      \If{$J_{\epsilon} - J_{\text{best}} > \delta_{\text{perf}}$}
        \State $\epsilon \gets \epsilon_{\text{cand}}$, $J_{\text{best}} \gets J_{\epsilon}$, repeat Step~1.
      \EndIf
      \State Step~2. Refine $T$ with fixed $\epsilon$
      \State $T_{\text{cand}} \gets 2T$ \Comment{candidate $T$}
      \State $J_{T} \gets J(\epsilon, T_{\text{cand}})$
      \If{$J_{T} - J_{\text{best}} > \delta_{\text{perf}}$}
        \State $T \gets T_{\text{cand}}$, $J_{\text{best}} \gets J_{T}$, repeat Step~2
      \EndIf
    \Until{Neither Step~1 nor Step~2 updates $(\epsilon, T)$}
    \State \textbf{Output:} selected $(\epsilon,T)$ as the adaptive choice for this problem instance.
  \end{algorithmic}
\end{algorithm}

\subsection{AG Algorithm for Approximate State Space}\label{sec:ApproxWhittleIndexAlg}
Now we are ready to incorporate the approximate state space $\tilde{\Omega}(T,\varepsilon\vert\boldsymbol{\omega}_{0})$ into the AG algorithms introduced in \citet{nino2001restless} based on Theorem~\ref{thm:PCLWhittle} and Lemmas~\ref{obs1} and~\ref{obs2}. Specifically, for the state update, if the resulting belief state $\boldsymbol{\omega}$ does not fall within $\tilde{\Omega}(T,\varepsilon\vert\boldsymbol{\omega}_{0})$, then we look for the belief state $\boldsymbol{\omega'}$ in $\tilde{\Omega}(T,\varepsilon\vert\boldsymbol{\omega_{0}})$ that has the closest Euclidean distance to $\boldsymbol{\omega}$, and use the Whittle index $\lambda_{\boldsymbol{\omega'}}$ of $\boldsymbol{\omega'}$ as the Whittle index of $\boldsymbol{\omega}$ or adopt a linear interpolation since~$\varepsilon$ is small. The detailed algorithm for computing the approximate Whittle index for our RMAB with general observation models is given below.
\newline
\begin{breakablealgorithm}
    % \captionsetup{justification=raggedright,singlelinecheck=false}
	\caption{Approximate Whittle Index Algorithm}
	\begin{algorithmic}[1]
        \Require
    	$\beta\in(0,1), T\ge1, N\ge2, 1\le K<N, H, \varepsilon$
        \Require
    	Initial belief state $\omega^{(n)}(0), \textbf{P}^{(n)}, \textbf{E}^{(n)}, \textbf{R}^{(n)}, n=1,...,N$
		\For{$n=1,..., N$}
		\State Let $\Omega^{(n)}(0,\varepsilon\vert\boldsymbol{\omega}^{(n)}(0))=\{\boldsymbol{\omega}^{(n)}(0)\}$ and $\Omega^{(n)}(-1,\varepsilon\vert\boldsymbol{\omega}^{(n)}(0))=\emptyset$
		\For{$t=1,...,T$}
		\State Let $D_{t-1}=\Omega^{(n)}(t-1,\varepsilon\vert\boldsymbol{\omega}^{(n)}(0))\backslash\Omega^{(n)}(t-2,\varepsilon\vert\boldsymbol{\omega}^{(n)}(0))$
            \State Let $\Omega^{(n)}(t,\varepsilon\vert\boldsymbol{\omega}^{(n)}(0))=\Omega^{(n)}(t-1,\varepsilon\vert\boldsymbol{\omega}^{(n)}(0))$
		\For{$\omega\in D_{t-1}$}
		\For{$h=1,...,H$}
              Compute $\boldsymbol{\omega'}=\mathcal{B}_{h}\boldsymbol{\omega}$
            \For {$\boldsymbol{\omega''}\in\Omega^{(n)}(t,\varepsilon\vert\boldsymbol{\omega}^{(n)}(0))$}
		    \If{$\vert\vert\omega'-\omega''\vert\vert>\varepsilon$}
		    \State Append $\omega'$ to $\Omega^{(n)}(t,\varepsilon\vert\boldsymbol{\omega}^{(n)}(0))$
                \State \textbf{break}
		    \EndIf
		\EndFor
		\EndFor
		\EndFor
            \EndFor
		\State Compute $P_{(n)}^{1},P_{(n)}^{0},\{R_{\boldsymbol{\omega}_{i}}\vert\boldsymbol{\omega}_{i}\in\Omega_{1}^{(n)}\}$
		\State Let $\Omega_{1}^{(n)}=\Omega^{(n)}(T,\varepsilon\vert\boldsymbol{\omega}^{(n)}(0)), k=1,$ compute the sequence $\{A_{\boldsymbol{\omega}_{i}}^{\Omega_{1}^{(n)}}\vert\boldsymbol{\omega}_{i}\in\Omega_{1}^{(n)}\}$
		\State Compute $\gamma_{\boldsymbol{\omega}_{i}}^{\Omega_{1}^{(n)}}=\frac{R_{\boldsymbol{\omega}_{i}}}{A_{\boldsymbol{\omega}_{i}}^{\Omega_{1}^{(n)}}},\boldsymbol{\omega}_{i}\in\Omega_{1}^{(n)};$
		\State Let $\pi_{1}=\arg\max\{\gamma_{\boldsymbol{\omega}_{i}}^{\Omega_{1}^{(n)}}\vert\boldsymbol{\omega}_{i}\in\Omega_{1}^{(n)}\}, \gamma_{\boldsymbol{\omega}_{\pi_{1}}}^{(n)}=\gamma_{\boldsymbol{\omega}_{\pi_{1}}}^{\Omega_{1}^{(n)}}, k=1$
		\While{$\vert\Omega_{k}^{(n)}\vert>1$}
		\State $k=k+1, \Omega_{k}^{(n)}=\Omega_{k-1}^{(n)}\backslash\{\boldsymbol{\omega}_{\pi_{k-1}}\}$
		\State Compute $\gamma_{\boldsymbol{\omega}_{i}}^{\Omega_{k}^{(n)}}=\gamma_{\boldsymbol{\omega}_{i}}^{\Omega_{k-1}^{(n)}}+\left(\frac{A_{\boldsymbol{\omega}_{i}}^{\Omega_{k-1}^{(n)}}}{A_{\boldsymbol{\omega}_{i}}^{\Omega_{k}^{(n)}}}-1\right)(\gamma_{\boldsymbol{\omega}_{i}}^{\Omega_{k-1}^{(n)}}-\gamma_{\boldsymbol{\omega}_{\pi_{k}}}^{\Omega_{k-1}^{(n)}}), \boldsymbol{\omega}_{i}\in\Omega_{k}^{(n)}$
		\State Let $\pi_{k}=\arg\max\{\gamma_{\boldsymbol{\omega}_{i}}^{\Omega_{k}^{(n)}}\vert\boldsymbol{\omega}_{i}\in\Omega_{k}^{(n)}\}, \gamma_{\boldsymbol{\omega}_{\pi_{k}}}^{(n)}=\gamma_{\boldsymbol{\omega}_{\pi_{k}}}^{\Omega_{k}^{(n)}}$
		\EndWhile
            \State $K=k$
            \If{$\gamma_{\boldsymbol{\omega}_{\pi_{1}}}^{(n)}\geq\cdots\geq\gamma_{\boldsymbol{\omega}_{\pi_{K}}}^{(n)}$}
            \State $FAIL^{(n)}:=0$
            \Else
            \State $FAIL^{(n)}:=1$
            \EndIf
		\EndFor
        \If{For any $1\leq n\leq N, FAIL^{(n)}=0$}
        \State $FAIL:=0$
        \Else
        \State $FAIL:=1$
        \EndIf
         \State \Return $\Omega^{(n)}(T,\varepsilon\vert\boldsymbol{\omega}^{(n)}(0)),\{\gamma_{\boldsymbol{\omega}}^{(n)}\vert\boldsymbol{\omega}\in\Omega_{1}^{(n)}\}, FAIL$, where $1\leq n\leq N$
	\end{algorithmic}
\end{breakablealgorithm}

\vspace{1em}
\noindent\textbf{Remark.} In the adaptive-greedy part of this algorithm, we need to input the parameter $\boldsymbol{A}$ of every arm. For the approximate state space, we can solve for all parameters $A_{\boldsymbol{\omega}}^{\Omega}$ using the dynamic programming equations of $T_{\boldsymbol{\omega}}^{\Omega}$. The output FAIL value is used to judge whether the problem is PCL-indexable w.r.t. $\tilde{\Omega}$. If FAIL=0, the problem is PCL-indexable, otherwise not.  If the approximate state space over any finite horizon fails the AG algorithm, then we can conclude that the original RMAB with infinite states does not satisfy PCL-indexability since the state combination subset violates the partial conservation law.

\subsection{Numerical Results}\label{sec:NumericalResults}

%In this subsection, we demonstrate the strong performance of our algorithm by numerical simulations. For the calculation of approxiamte state space, we set the maximum number of iterations to 6, and the error limit $\varepsilon$ to $1e-3$. For each example, we set the number of arms to 6, and the number of arms activated each time to 1, and other parameters to meet the PCL-indexability. We mainly consider the unit-time-return performance of PCL/Whittle index policy and myopic policy when the RMAB system runs for 200 time slots, using 10000 Monte-Carlo simulation experiments. Figures \ref{fig:1}-\ref{fig:6} show that our PCL/Whittle index policy achieves a superior performance (specific parameter settings are shown in Table% 

In this subsection, we present the robust performance of our algorithm through extensive numerical simulations. To approximate the state space, we set the maximum number of iterations to~6 and the error limit~$\varepsilon$ to $10^{-3}$.  We expanded our simulation framework to evaluate systems of varying scales, ranging from 3x3 to 6x6 configurations (for the transition matrix of arm states). For each arm state size, we conducted simulations with 30, 40, and 50 arms, while maintaining the arm selection and system parameters adhered to the PCL-indexability. Our primary focus is on comparing the unit-time-return performance of the PCL/Whittle index policy against a myopic policy over 200 decision epochs, utilizing 10,000 Monte Carlo simulation experiments for each configuration.  The performance gap, quantified as the percentage of improvement from the Whittle index policy over the myopic policy, is summarized in Table~\ref{tbl1}.  As shown in Fig.~\ref{fig:3x3_system}-\ref{fig:6x6_system}, the PCL/Whittle index policy consistently achieves significantly superior performance compared to the myopic policy.

\begin{table}[htbp]
\caption{Performance Improvement of Whittle Index Policy over Myopic Policy}\label{tbl1}
\footnotesize % Add this line to reduce font size
\begin{tabular*}{\textwidth}{@{\extracolsep{\fill}}ccccc@{}}
\toprule
Arm State Size & Number of Arms & Myopic Policy's Average Reward & Whittle Index Policy's Average Reward & Performance Gain \\
\midrule
3x3 & 30 & 2.3618 & 2.4591 & 4.12\% \\
3x3 & 40 & 2.3740 & 2.5889 & 9.05\% \\
3x3 & 50 & 2.3587 & 2.4784 & 5.07\% \\
\midrule
4x4 & 30 & 2.5201 & 2.7764 & 10.17\% \\
4x4 & 40 & 2.6595 & 2.7900 & 4.91\% \\
4x4 & 50 & 2.4355 & 2.7620 & 13.41\% \\
\midrule
5x5 & 30 & 2.9509 & 3.1787 & 7.72\% \\
5x5 & 40 & 2.4096 & 2.5853 & 7.29\% \\
5x5 & 50 & 2.3188 & 2.5419 & 9.62\% \\
\midrule
6x6 & 30 & 2.2340 & 2.4113 & 7.93\% \\
6x6 & 40 & 2.2554 & 2.4409 & 8.22\% \\
6x6 & 50 & 2.2977 & 2.4729 & 7.63\% \\
\bottomrule
\end{tabular*}
\end{table}

% Figure 1-3: 3x3 System
\begin{figure}[htbp]
    \centering
    \subfigure[30 arms]{\includegraphics[width=0.3\textwidth]{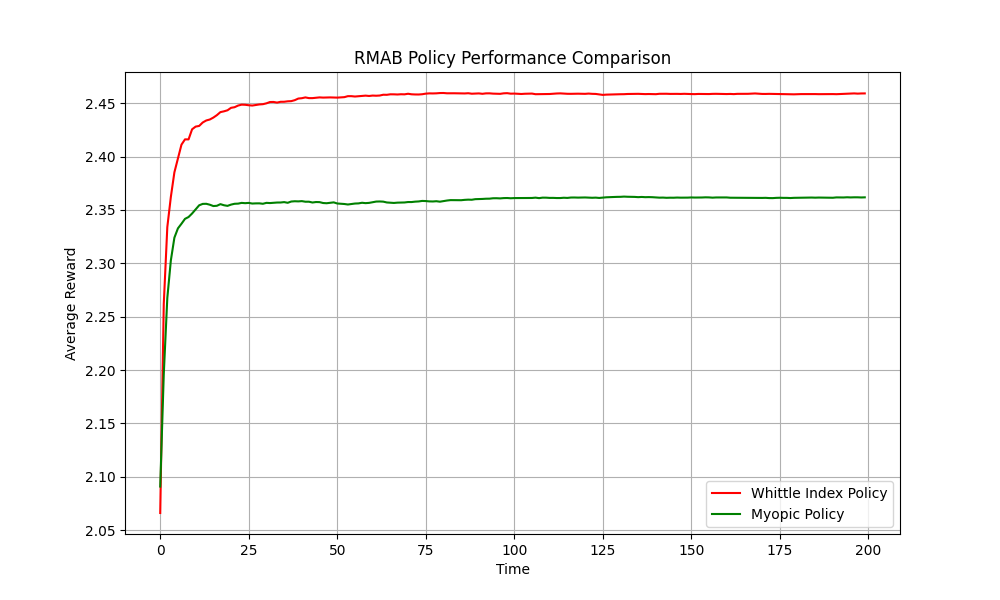}\label{fig:3x3_30arms}}
    \subfigure[40 arms]{\includegraphics[width=0.3\textwidth]{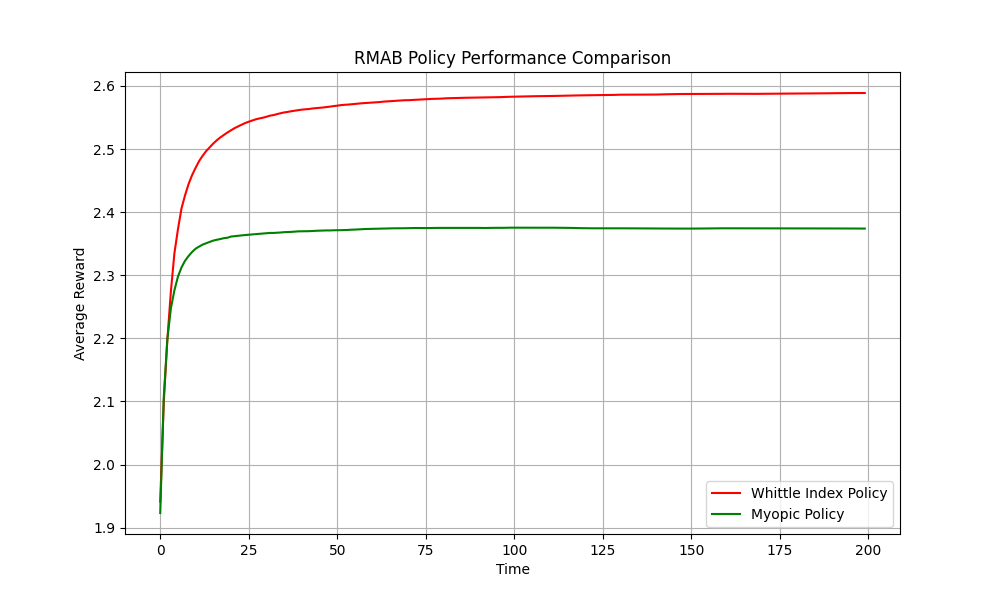}\label{fig:3x3_40arms}}
    \subfigure[50 arms]{\includegraphics[width=0.3\textwidth]{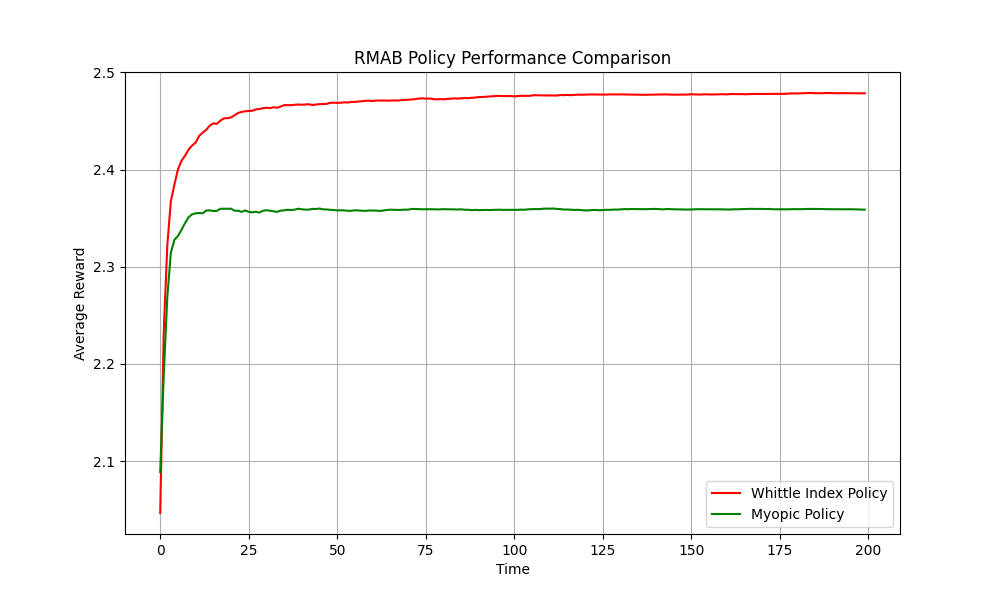}\label{fig:3x3_50arms}}
    \caption{Performance comparison for 3x3 system with varying number of arms}
    \label{fig:3x3_system}
\end{figure}

% Figure 4-6: 4x4 System
\begin{figure}[htbp]
    \centering
    \subfigure[30 arms]{\includegraphics[width=0.3\textwidth]{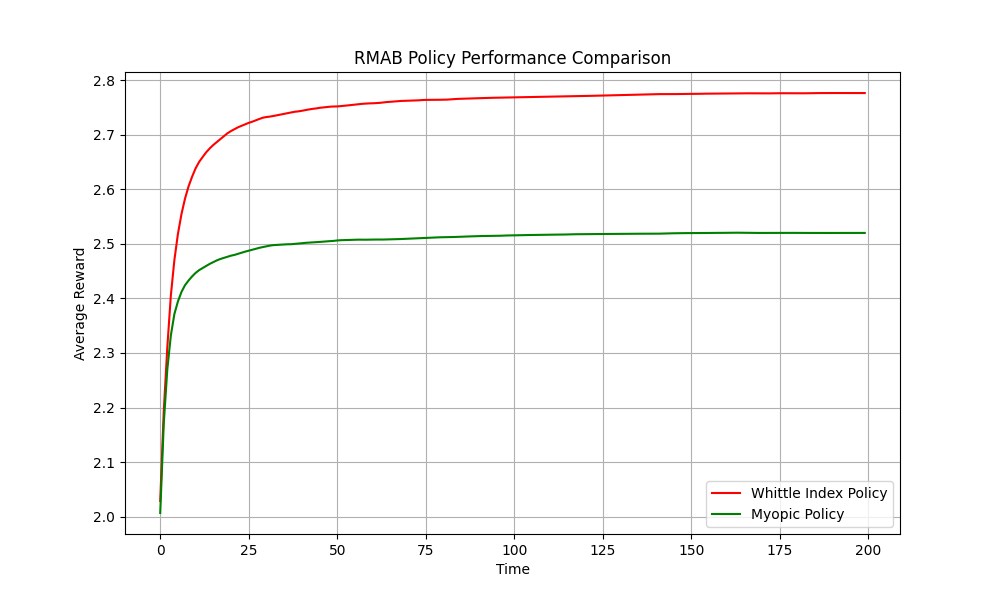}\label{fig:4x4_30arms}}
    \subfigure[40 arms]{\includegraphics[width=0.3\textwidth]{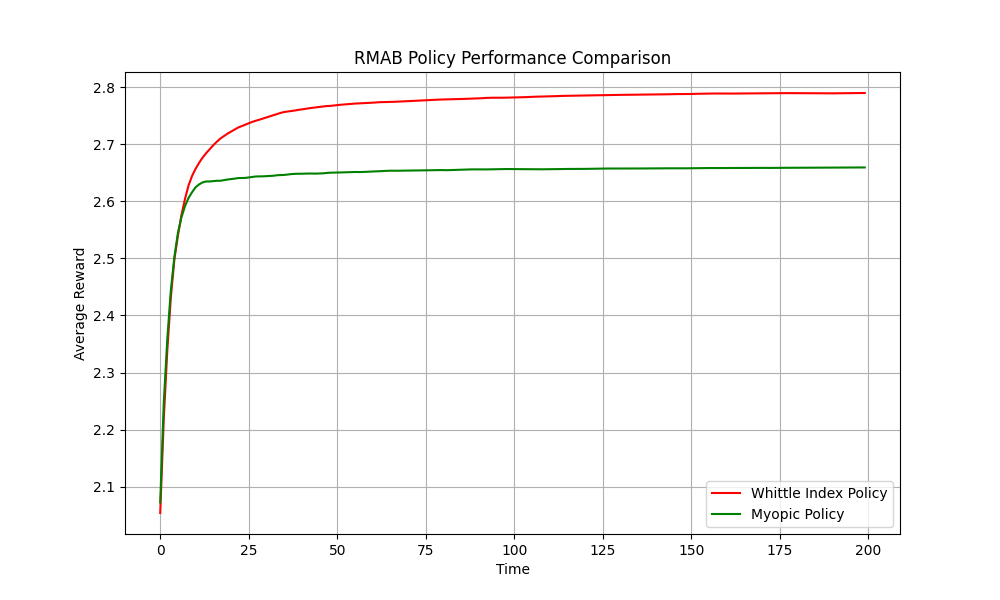}\label{fig:4x4_40arms}}
    \subfigure[50 arms]{\includegraphics[width=0.3\textwidth]{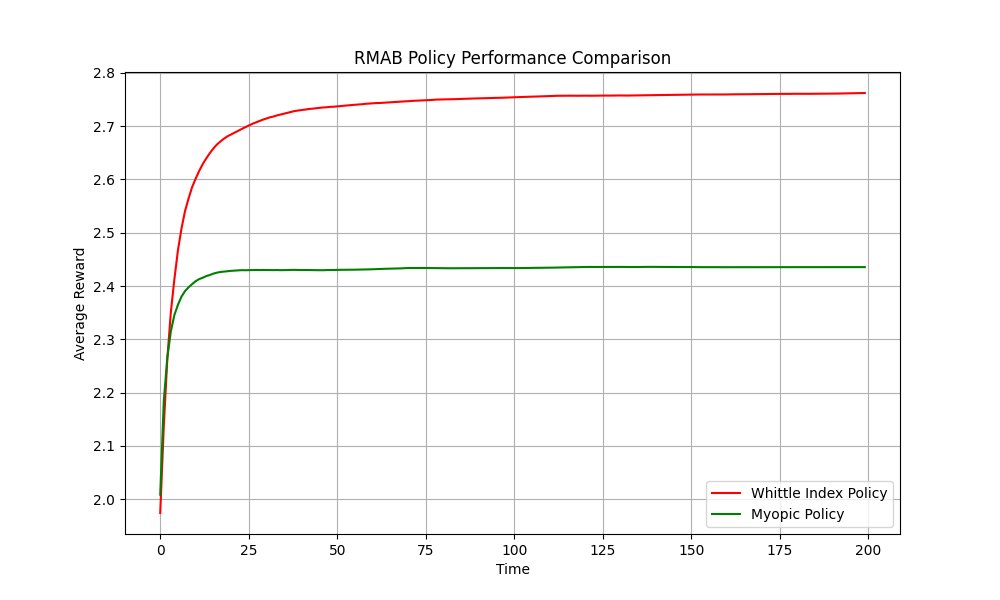}\label{fig:4x4_50arms}}
    \caption{Performance comparison for 4x4 system with varying number of arms}
    \label{fig:4x4_system}
\end{figure}

% Figure 7-9: 5x5 System
\begin{figure}[htbp]
    \centering
    \subfigure[30 arms]{\includegraphics[width=0.3\textwidth]{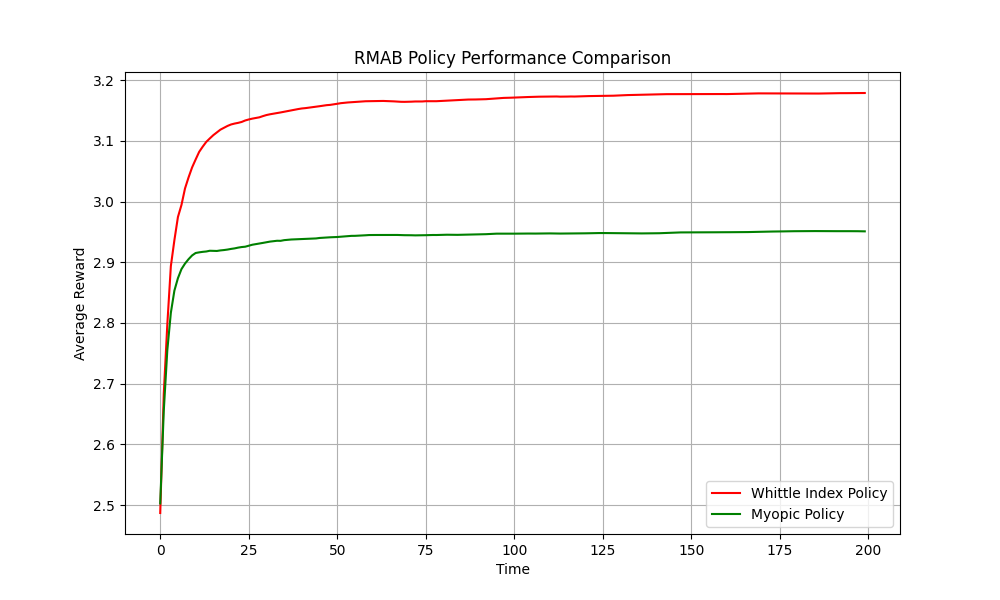}\label{fig:5x5_30arms}}
    \subfigure[40 arms]{\includegraphics[width=0.3\textwidth]{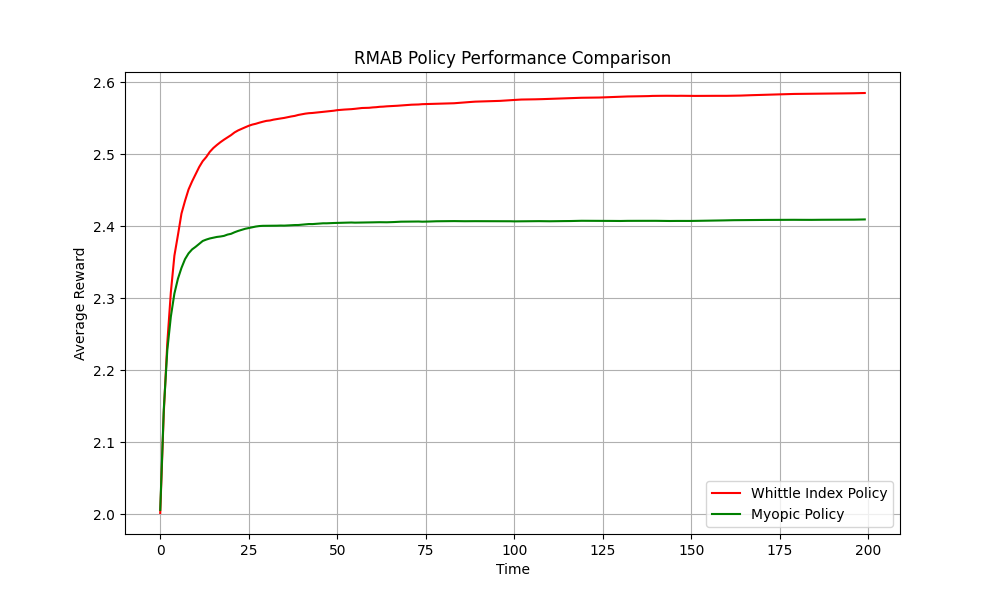}\label{fig:5x5_40arms}}
    \subfigure[50 arms]{\includegraphics[width=0.3\textwidth]{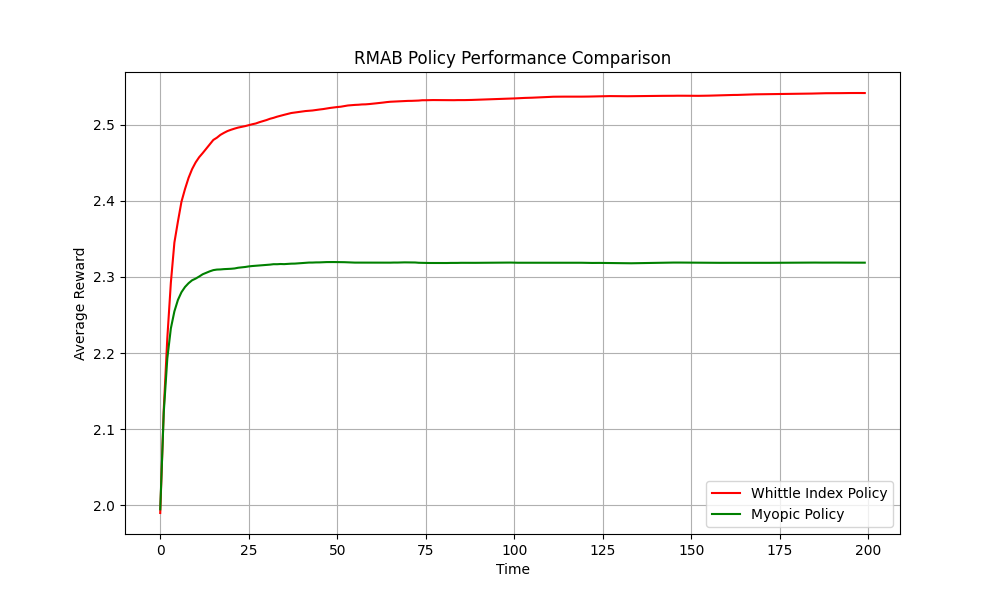}\label{fig:5x5_50arms}}
    \caption{Performance comparison for 5x5 system with varying number of arms}
    \label{fig:5x5_system}
\end{figure}

% Figure 10-12: 6x6 System
\begin{figure}[htbp]
    \centering
    \subfigure[30 arms]{\includegraphics[width=0.3\textwidth]{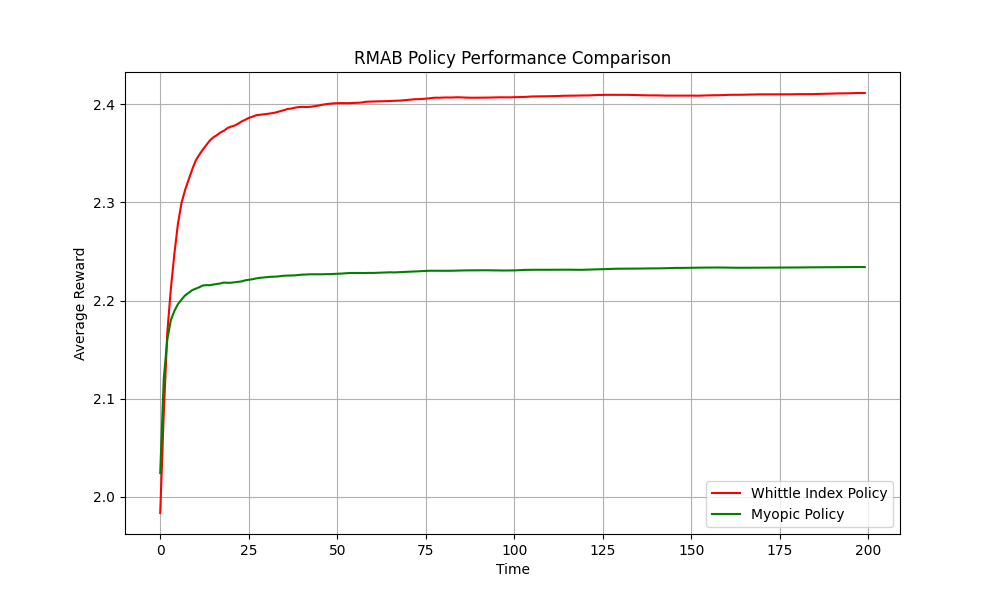}\label{fig:6x6_30arms}}
    \subfigure[40 arms]{\includegraphics[width=0.3\textwidth]{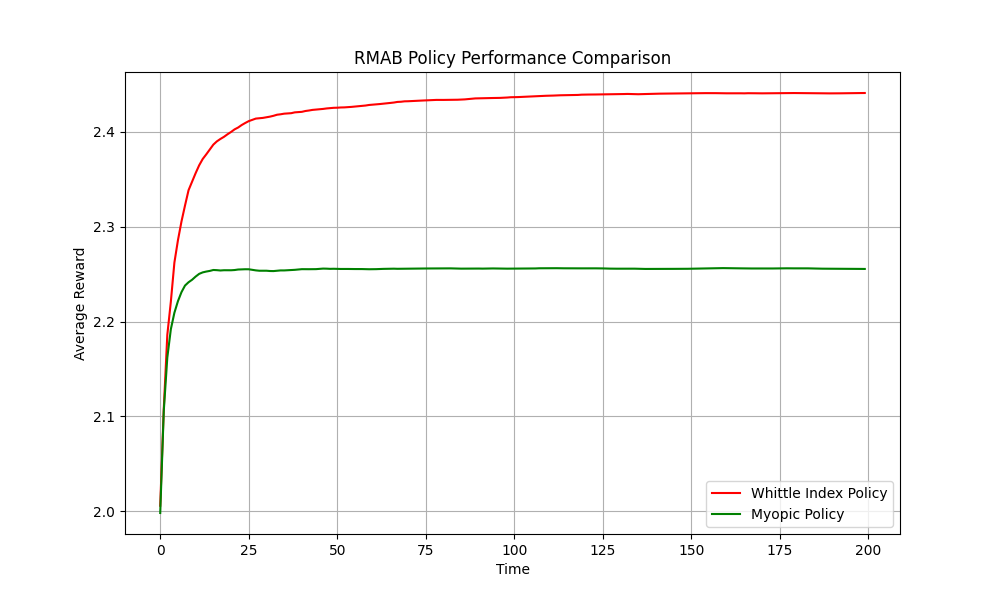}\label{fig:6x6_40arms}}
    \subfigure[50 arms]{\includegraphics[width=0.3\textwidth]{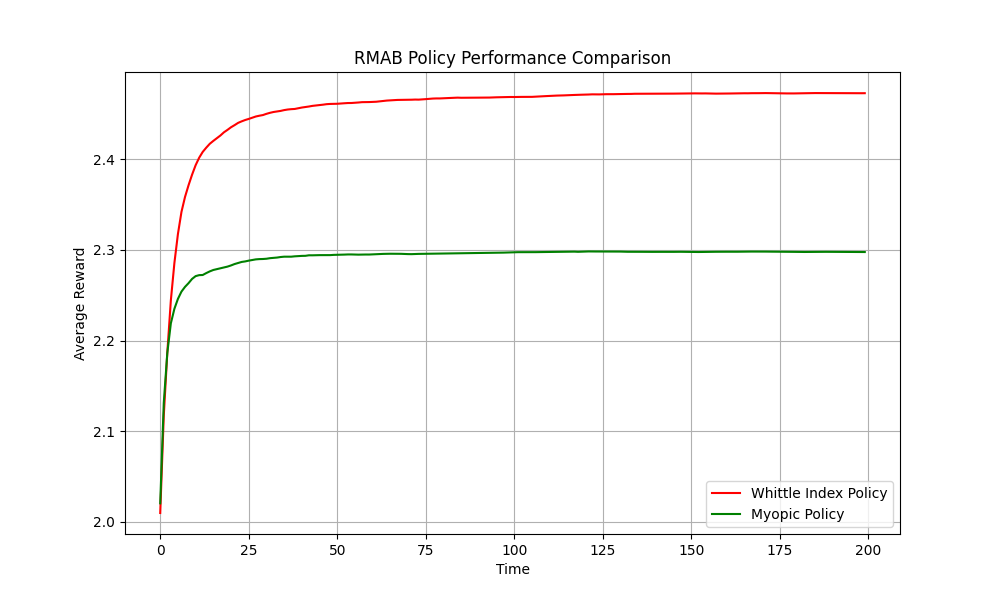}\label{fig:6x6_50arms}}
    \caption{Performance comparison for 6x6 system with varying number of arms}
    \label{fig:6x6_system}
\end{figure}

Next, we compare our Whittle index policy with more baselines to demonstrate the strong performance of the former.

\noindent\emph{(a.1) Rollout and VFA baselines:}
We now compare our Whittle index policy with (i) a one-step rollout policy based on the myopic baseline [1], and (ii) value-function approximation policies (QMDP and FIB)~[2], across four discount factors $\beta\in\{0.7,0.8,0.9,0.9999\}$.
Figure~\ref{fig:rollout-vfa} plots the average reward over time for all five policies.

\begin{figure}[h!]
  \centering
  \includegraphics[width=0.48\textwidth]{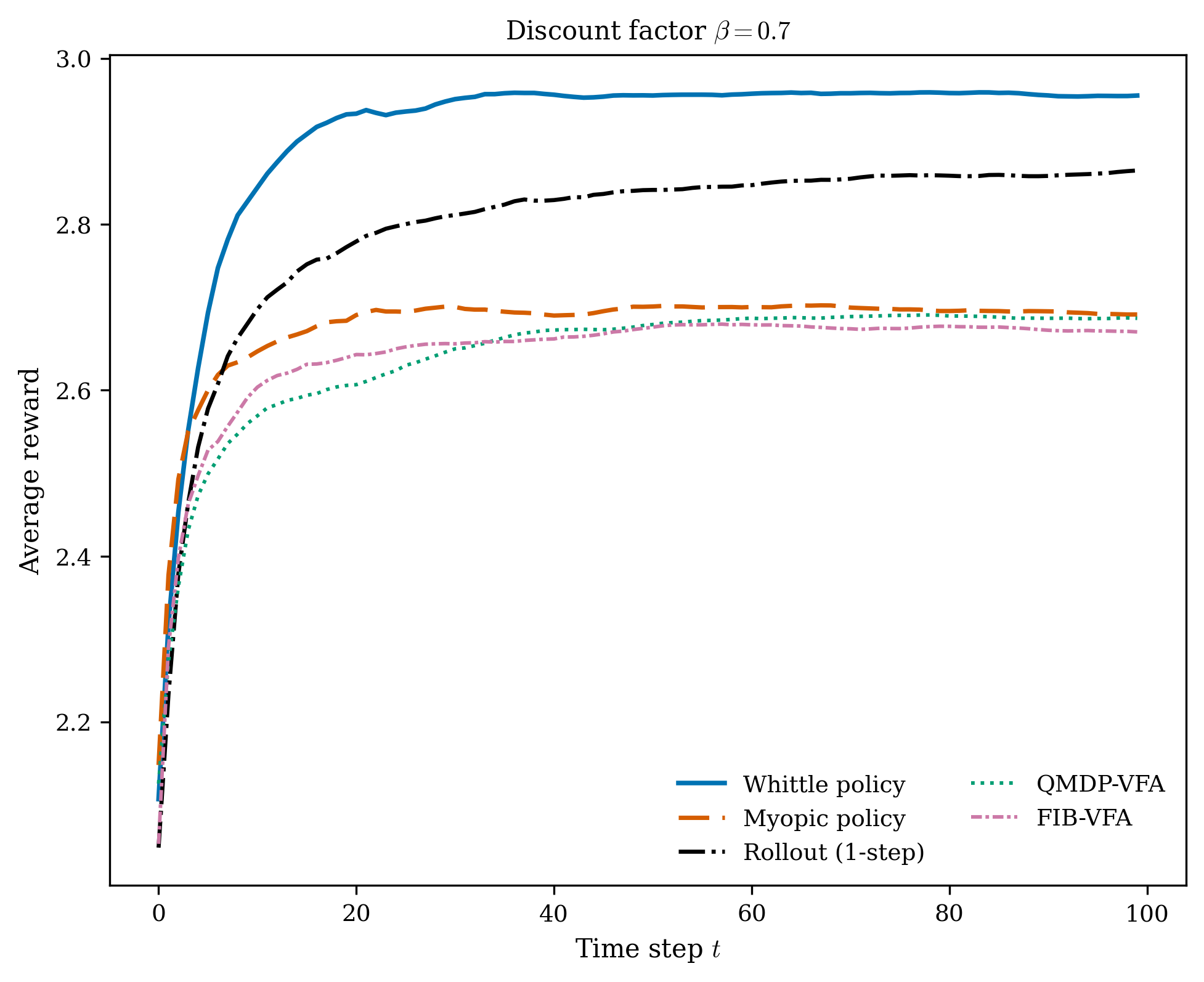}%
  \hfill
  \includegraphics[width=0.48\textwidth]{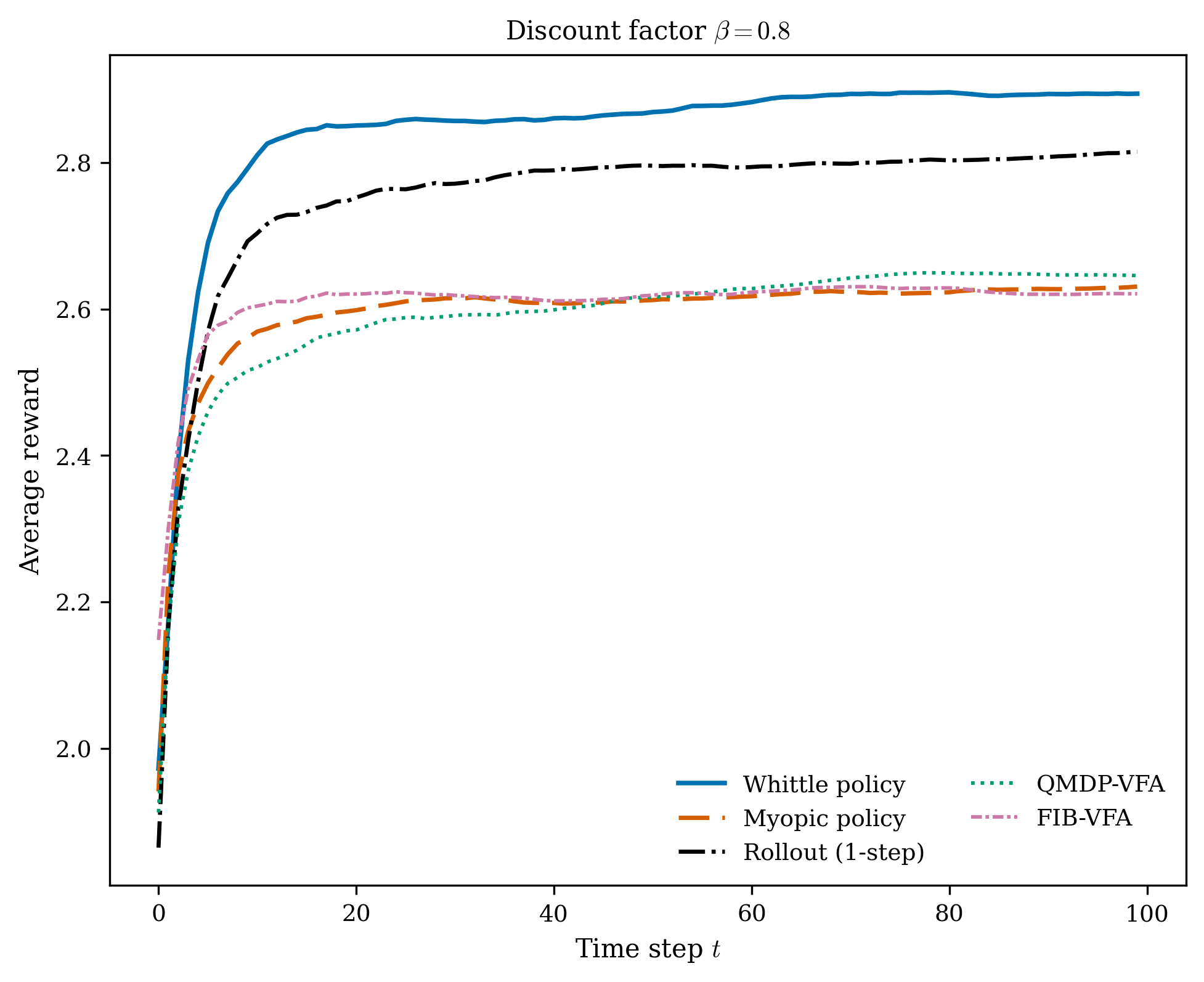}\\[0.4em]
  \includegraphics[width=0.48\textwidth]{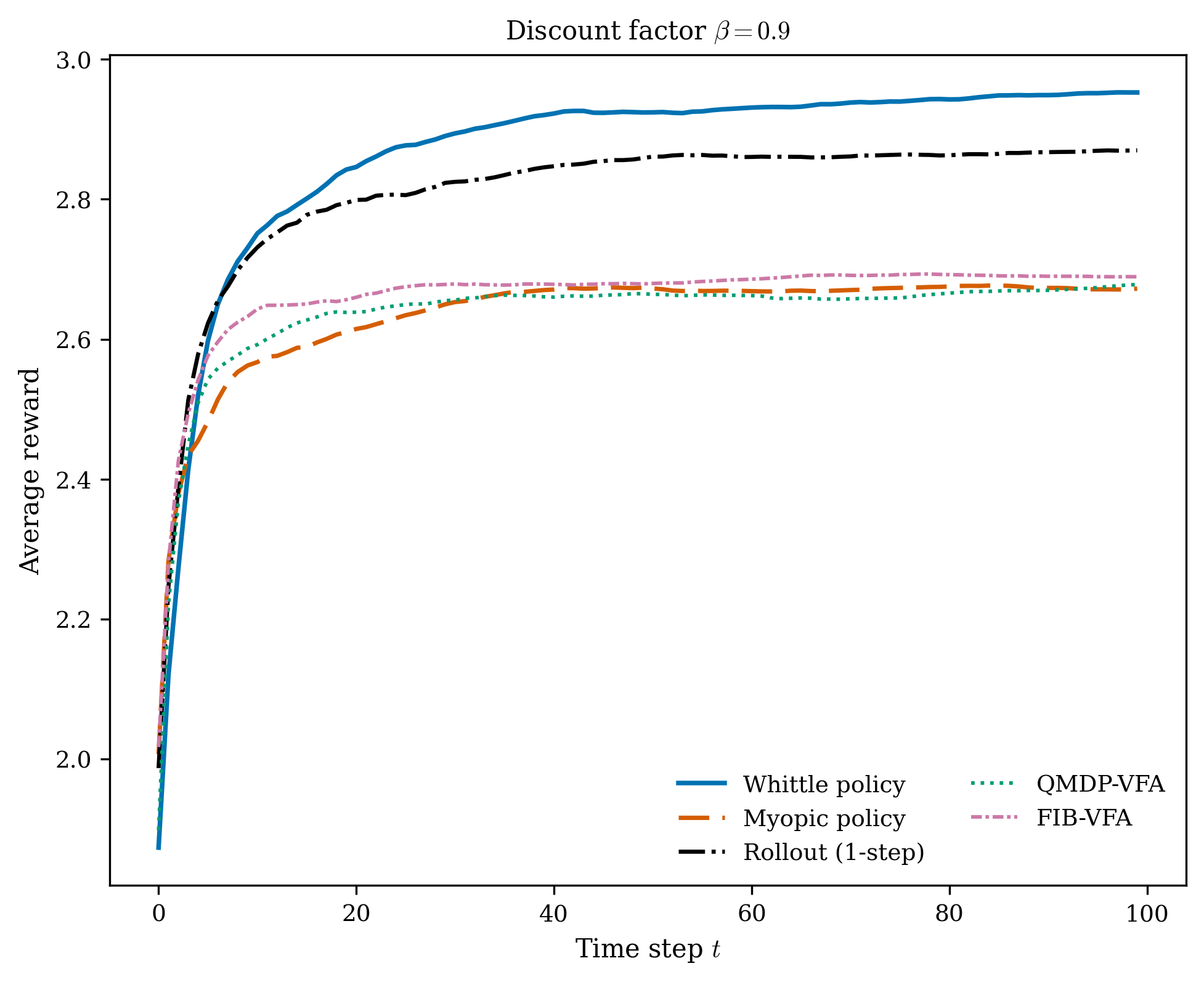}%
  \hfill
  \includegraphics[width=0.48\textwidth]{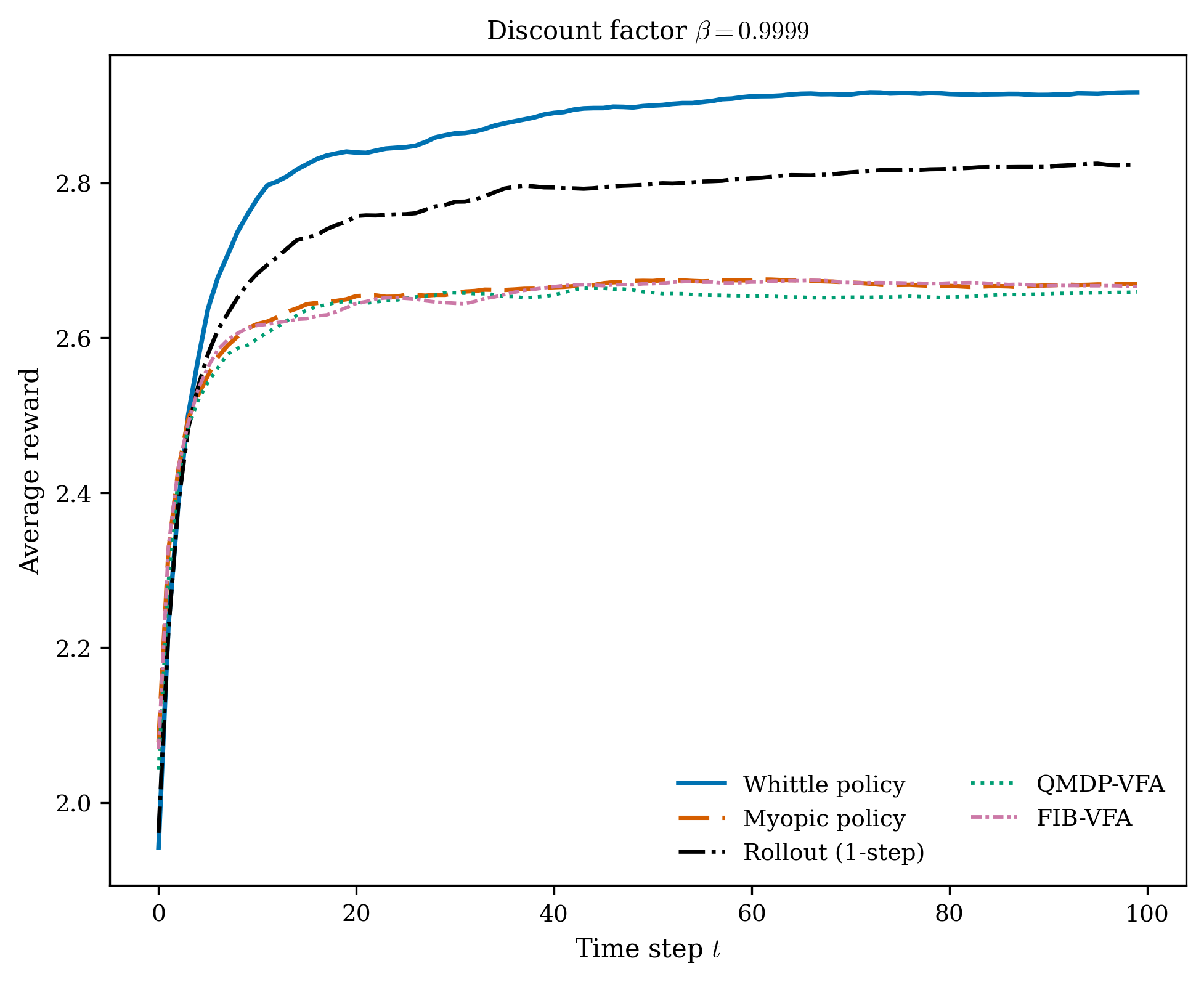}
  \caption{Performance of Whittle index, myopic, rollout, QMDP-VFA, and FIB-VFA policies.}
  \label{fig:rollout-vfa}
\end{figure}

Across all four discount factors, the Whittle index policy consistently outperforms the rollout baseline.
The two VFA policies closely track the myopic policy in both performance and runtime (see Figure~\ref{fig:runtime-vs-N}); their curves are almost indistinguishable from myopic in Figure~\ref{fig:rollout-vfa}.
Rollout is far more computationally expensive than Whittle, without providing a clear performance advantage,and the two VFA baselines perform nearly identically to the myopic policy, we therefore omit them from the main (large-scale) experiments to maintain figure clarity. The myopic policy alone serves as a consolidated proxy for these simple baselines.

\smallskip
\noindent\emph{Runtime versus number of arms and time horizon:}
On the experiment above, we measure the computation time per simulation run as a function of number of arms~$N$ or time horizon~$T$, for all five policies (myopic, Whittle, rollout, QMDP-VFA, and FIB-VFA).
Figure~\ref{fig:runtime-vs-N} shows the resulting curves with a logarithmic $y$-axis.

\begin{figure}[h!]
  \centering
  \includegraphics[width=0.48\textwidth]{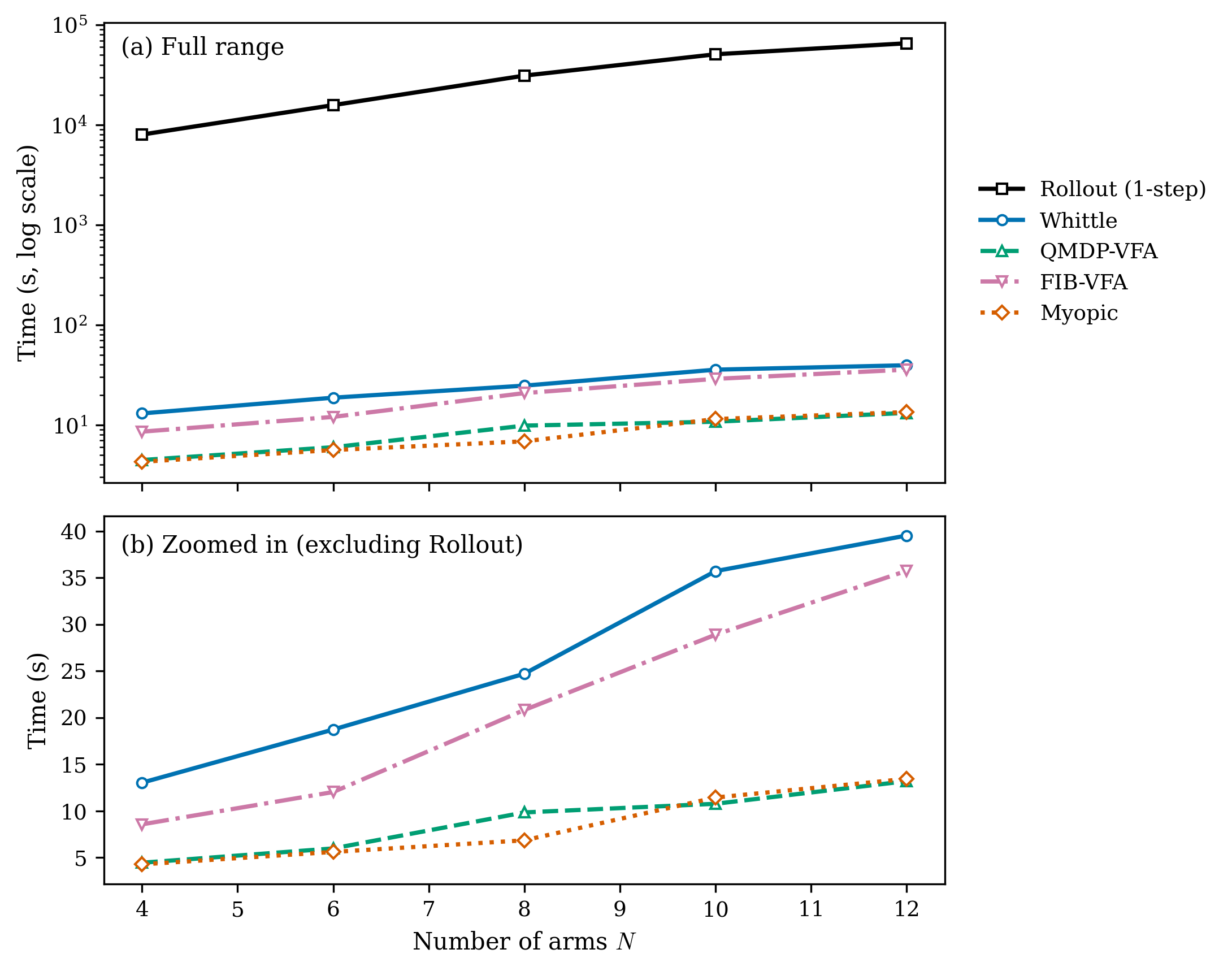}
  \hfill
  \includegraphics[width=0.48\linewidth]{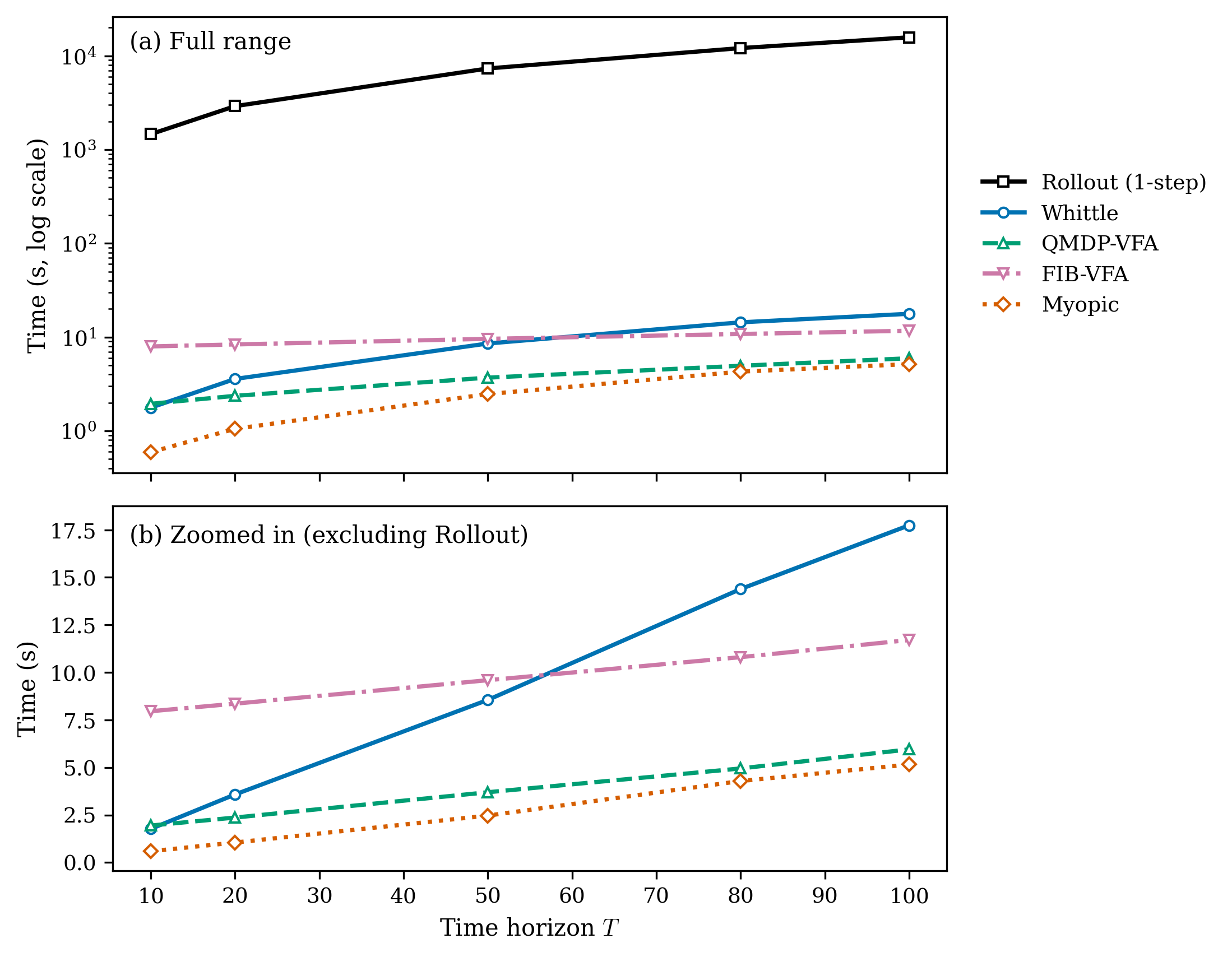}
  \caption{Computation time versus $N$ and $T$.}
  \label{fig:runtime-vs-N}
\end{figure}

The myopic, Whittle index, and VFA policies all exhibit roughly linear growth in~$N$ and~$T$, whereas the rollout baseline becomes several orders of magnitude slower as~$N$ or~$T$ increases.
This confirms that our  Whittle index computation scales well in terms of~$N$ and~$T$, while rollout is not practical beyond small instances.

\noindent\emph{(a.2) Exact dynamic programming on tiny instances:}
To compare against the exact dynamic-programming, we consider a tiny two-state, four-arm instance with time horizon~$7$, for which dynamic programming is feasible.
We generate four independent instances, for each one, compute both the optimal (dynamic-programming) policy and the Whittle policy.
Figure~\ref{fig:optimal-small} plots the average discounted reward over time for the two policies, based on $500$ Monte--Carlo runs per instance.

\begin{figure}[h!]
  \centering
  \includegraphics[width=0.48\textwidth]{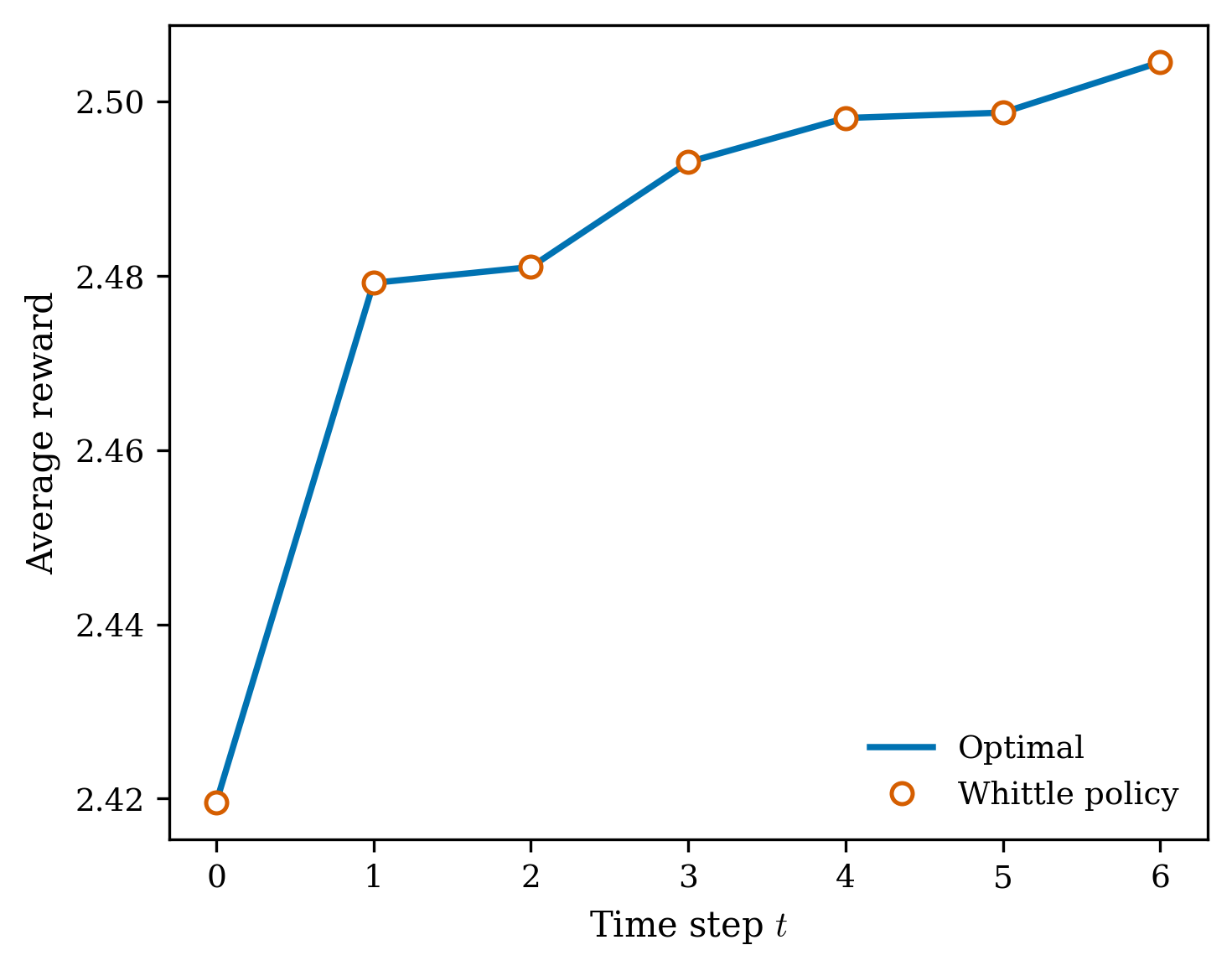}%
  \hfill
  \includegraphics[width=0.48\textwidth]{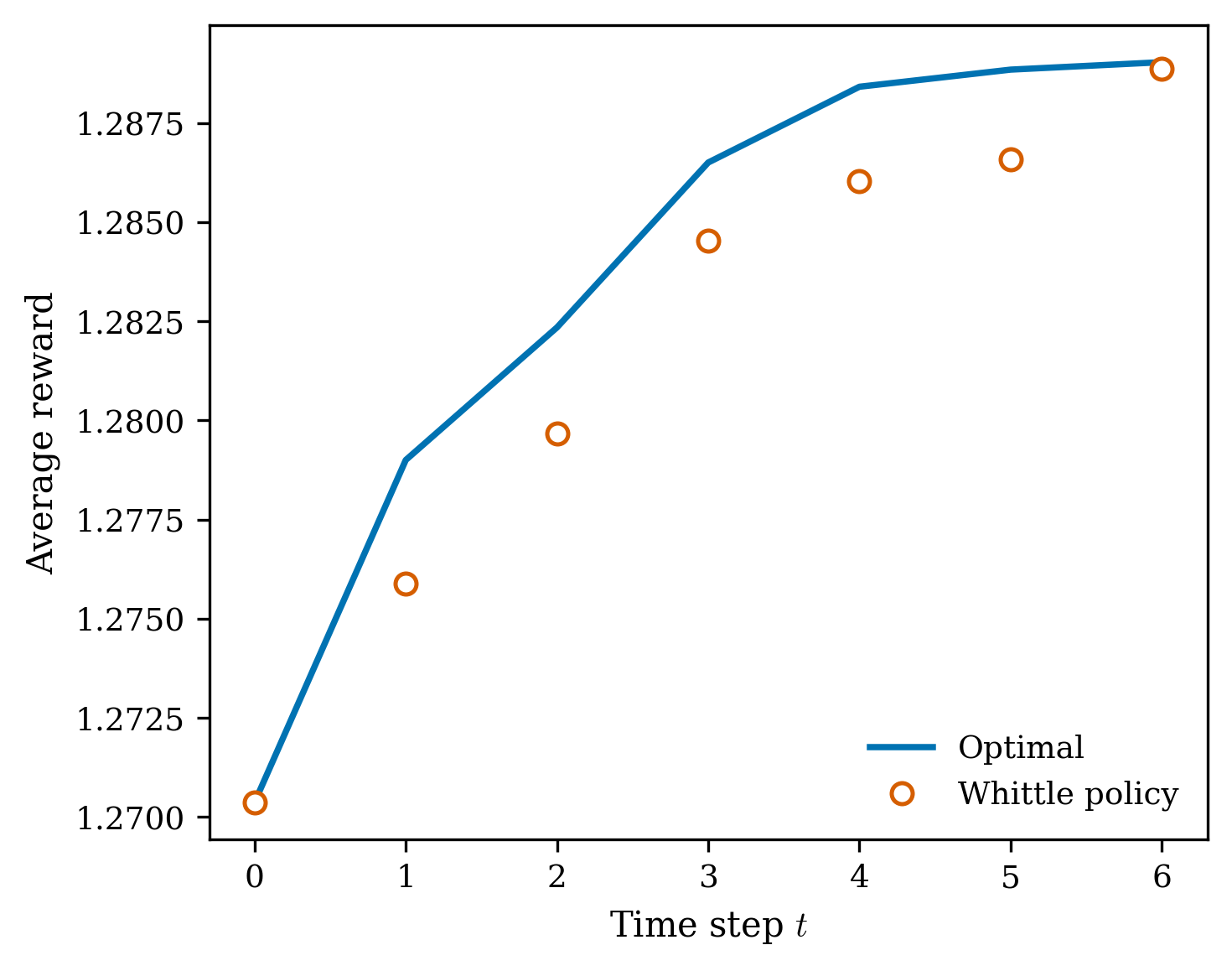}\\[0.4em]
  \includegraphics[width=0.48\textwidth]{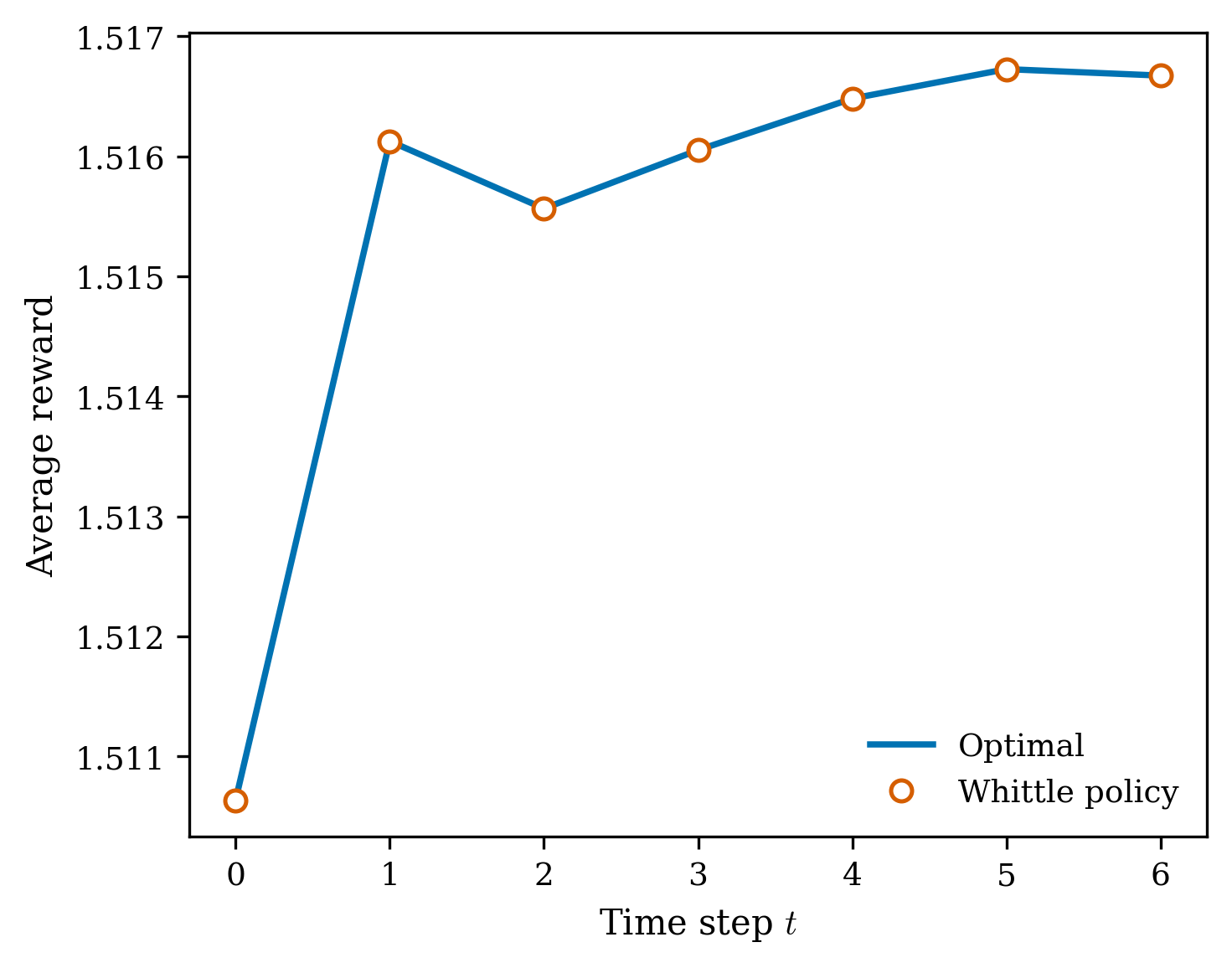}%
  \hfill
  \includegraphics[width=0.48\textwidth]{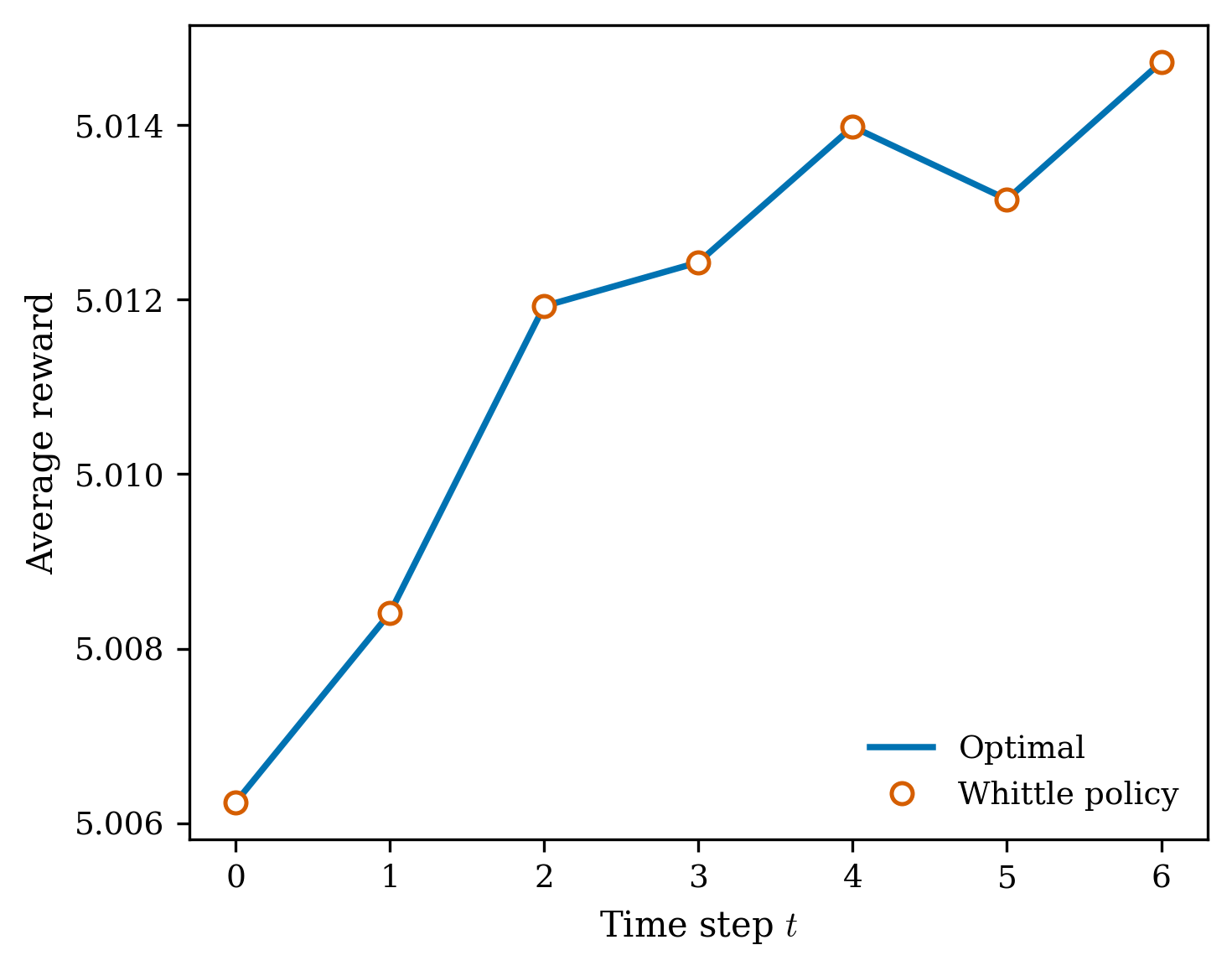}
  \caption{Performance of the optimal and Whittle index policies.}
  \label{fig:optimal-small}
\end{figure}

On three of the four instances (panels~1, 3, and~4), the Whittle and optimal curves coincide at all time steps up to numerical precision.
On the remaining instance (panel~2), the Whittle policy starts slightly below the optimal policy in the very early steps (with a maximum percentage gap of about $0.24\%$ in average reward), but quickly catches up; by the final time step the Whittle policy lags the optimal one by only $0.01\%$.

Thus, whenever dynamic programming is computationally feasible, the Whittle policy essentially recovers the optimal performance trajectory.

However, computing the exact dynamic-programming policy is already very expensive on this tiny example:
with 500 Monte--Carlo runs, the average runtime is $61.58$\,s and the peak resident set size (RSS) is $1768.40$\,MB.
By contrast, our Whittle policy on the larger three-state, six-arm benchmark runs in a fraction of a second per arm and uses under $100$\,MB of memory (see part~(b) below).
This confirms that dynamic programming is useful as a sanity check on small instances, but it cannot serve as a scalable baseline for large systems.

\noindent\emph{(b) Runtime and memory}

\smallskip
\noindent\emph{Memory usage:}
We monitor peak RSS on a three-state, six-arm instance with time horizon $100$ and $500$ Monte--Carlo runs.
Table~\ref{tab:memory-main} reports the measured peak RSS for each policy. All five policies have essentially the same peak memory footprint.

\begin{table}[h!]
  \centering
  \caption{Peak resident set size.}
  \label{tab:memory-main}
  \begin{tabular}{lc}
    \toprule
    Policy            & Peak RSS (MB) \\
    \midrule
    Myopic            & 92.75 \\
    Whittle           & 93.14 \\
    Rollout (1-step)  & 93.72 \\
    QMDP-VFA          & 94.05 \\
    FIB-VFA           & 94.10 \\
    \bottomrule
  \end{tabular}
\end{table}

\smallskip
\noindent\emph{Substantiate scalability:}
We consider larger instances with $K\times K$ transition and observation matrices for $K\in\{6,7,8,9\}$ and 100 arms.
For each $K$, we run $1000$ Monte--Carlo simulations and compare the Whittle and myopic policies.
Figure~\ref{fig:large-state} shows the average reward over time for each instance, and Table~\ref{tab:large-state} summarises the final performance gap and the runtime. Across these larger state spaces, the Whittle index policy consistently improves over the myopic policy by $6$--$9\%$ in the long-run reward.The runtime increase with $K$, as expected.

\begin{figure}[h!]
  \centering
  \includegraphics[width=0.48\textwidth]{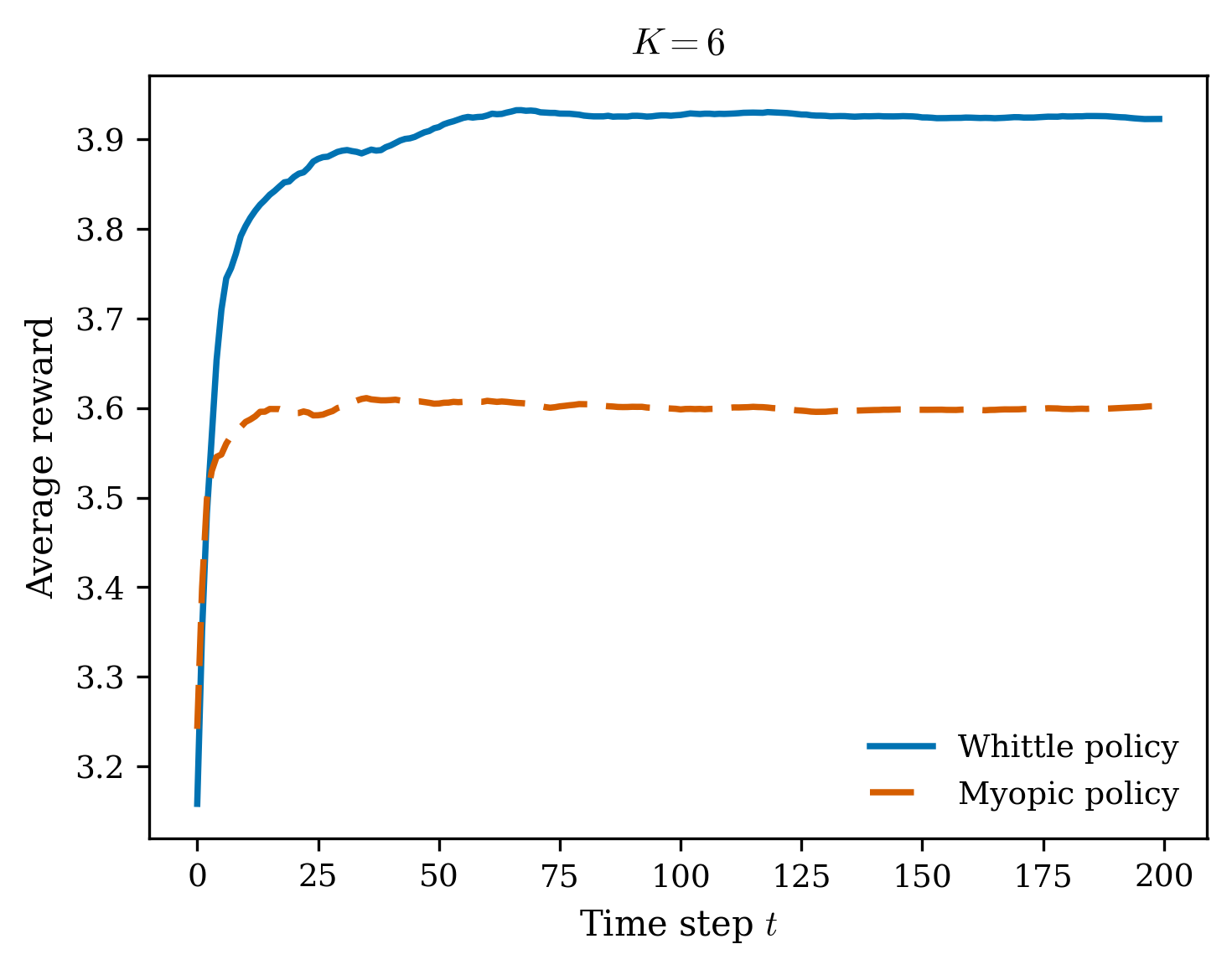}%
  \hfill
  \includegraphics[width=0.48\textwidth]{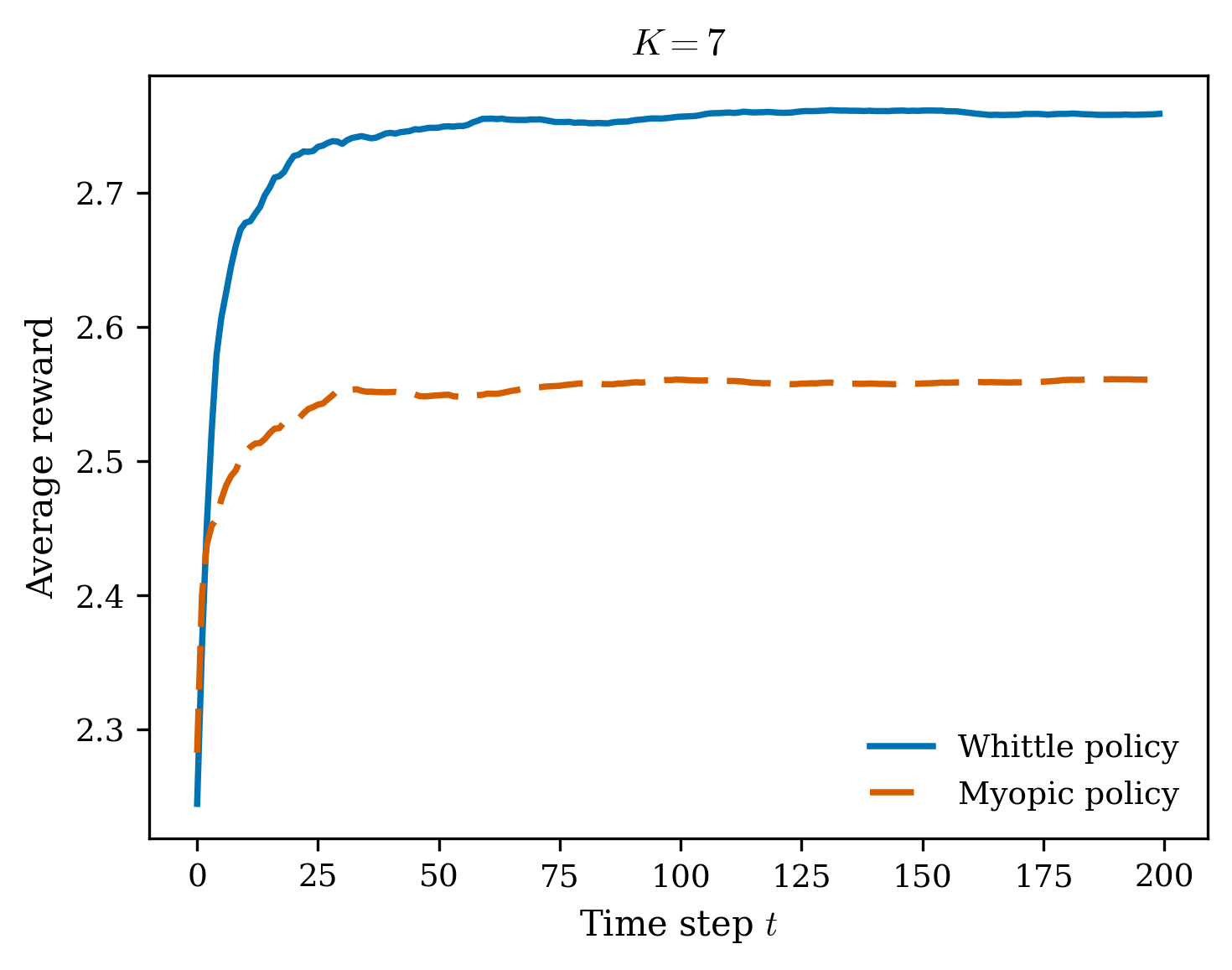}\\[0.4em]
  \includegraphics[width=0.48\textwidth]{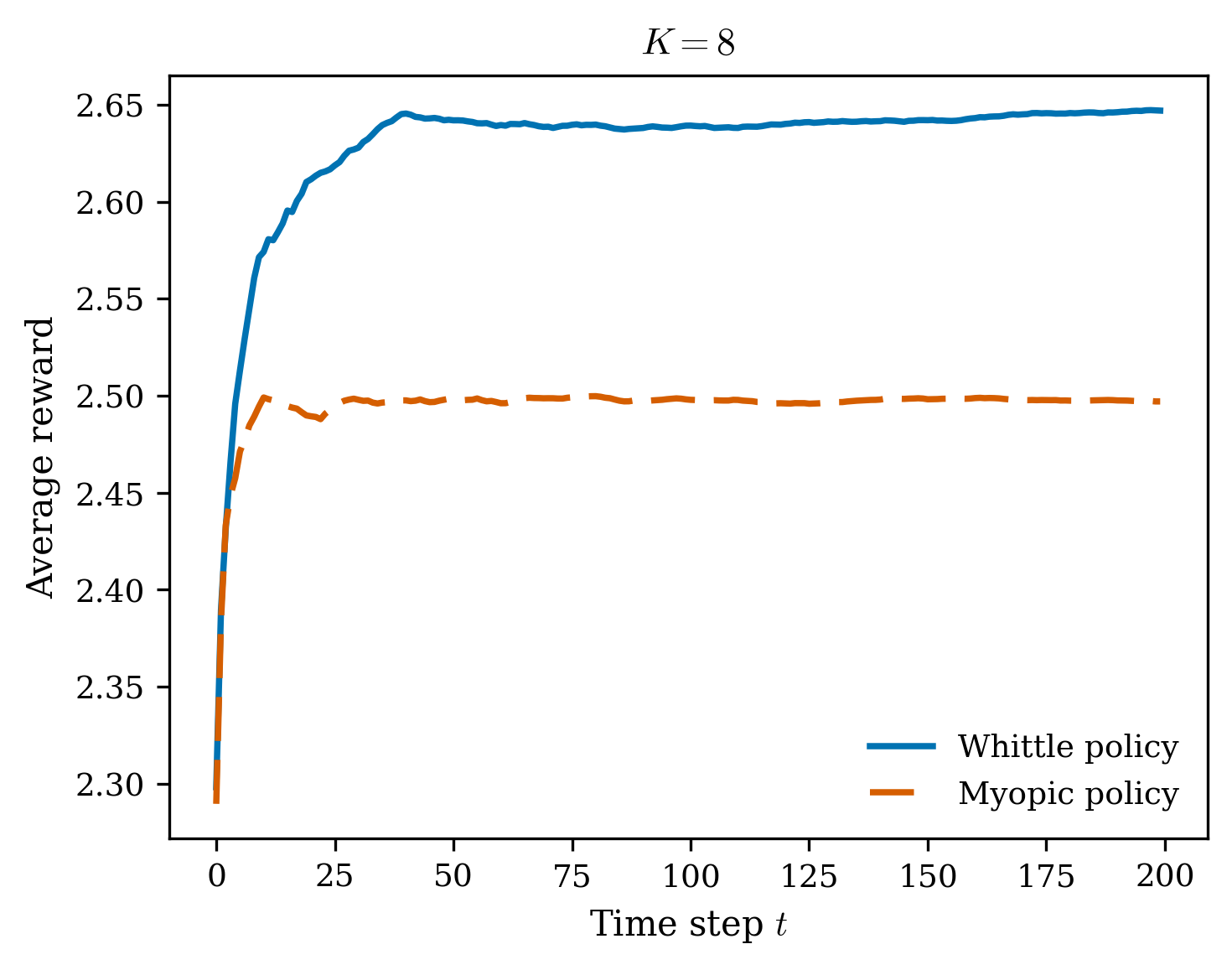}%
  \hfill
  \includegraphics[width=0.48\textwidth]{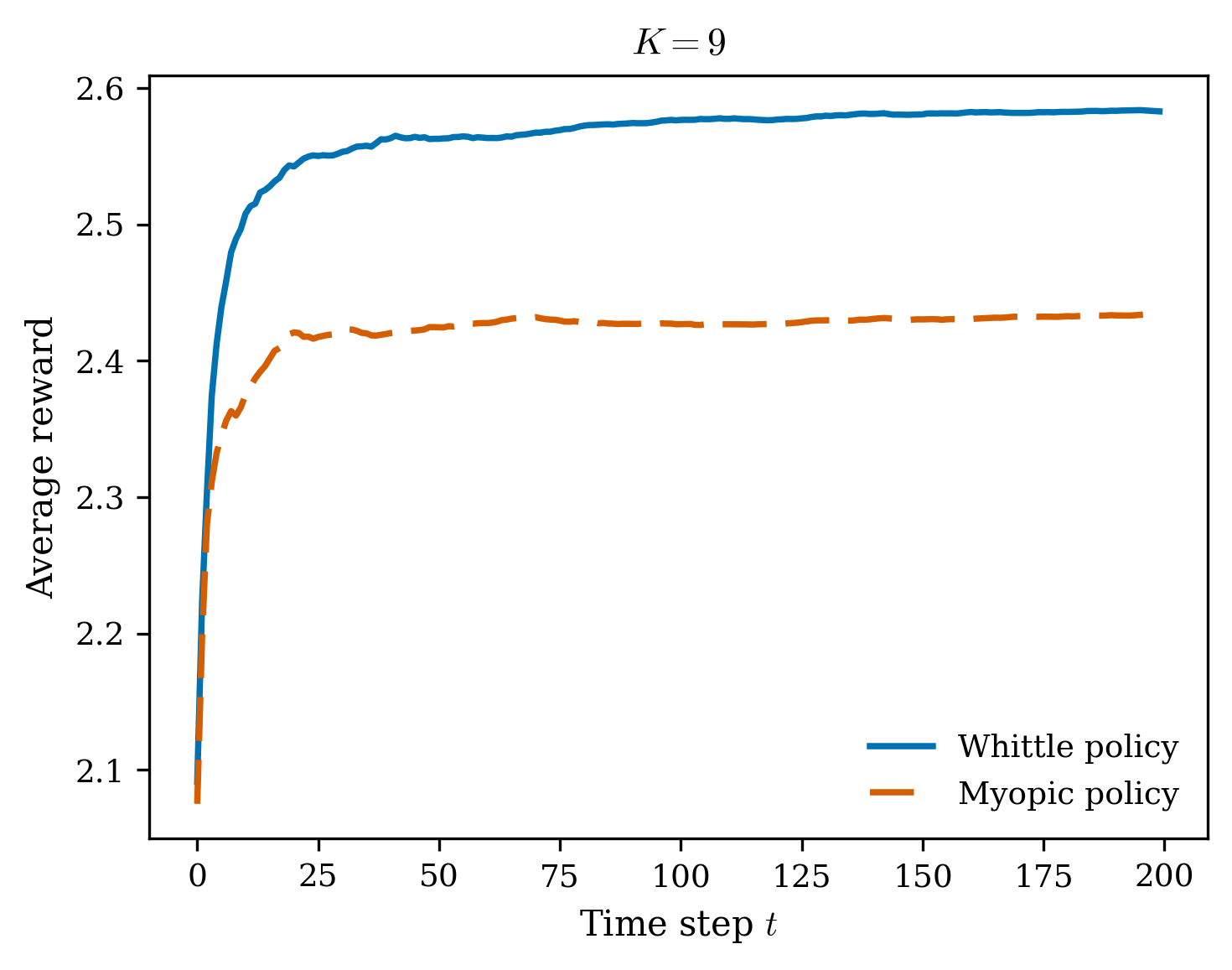}
  \caption{Performance of Whittle index and myopic policies on larger $K\times K$ instances.}
  \label{fig:large-state}
\end{figure}

\begin{table}[h!]
  \centering
  \caption{Performance gap and runtime on larger $K\times K$ instances.}
  \label{tab:large-state}
  \begin{tabular}{ccc}
    \toprule
    State dimension $K$ & Whittle improvement over myopic (\%) & runtime (s) \\
    \midrule
     6 & 8.90 & 1820.96 \\
     7 & 7.75 & 2139.01 \\
     8 & 6.01 & 2862.12 \\
     9 & 6.12 & 2993.83 \\
    \bottomrule
  \end{tabular}
\end{table}

Now we study the performance of our Whittle index policy under four observation-error structures (symmetric, state-independent, asymmetric, and random matrices). For each structure, we record the average index computation time per arm, and then run the Monte--Carlo simulation to compare the long-run performance of the Whittle index policy with the Myopic policy.

\begin{figure}[h!]
  \centering
  \includegraphics[width=0.48\textwidth]{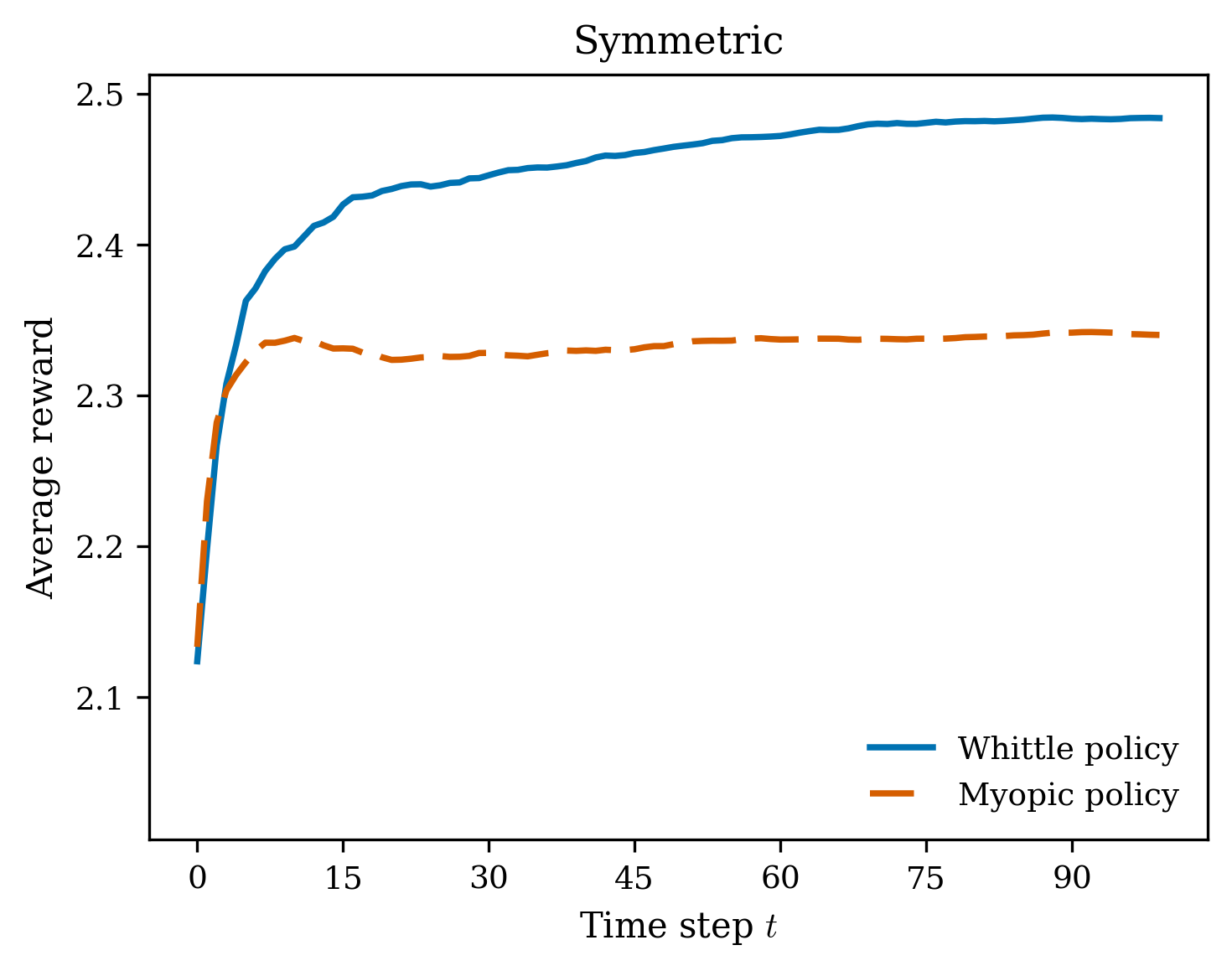}%
  \hfill
  \includegraphics[width=0.48\textwidth]{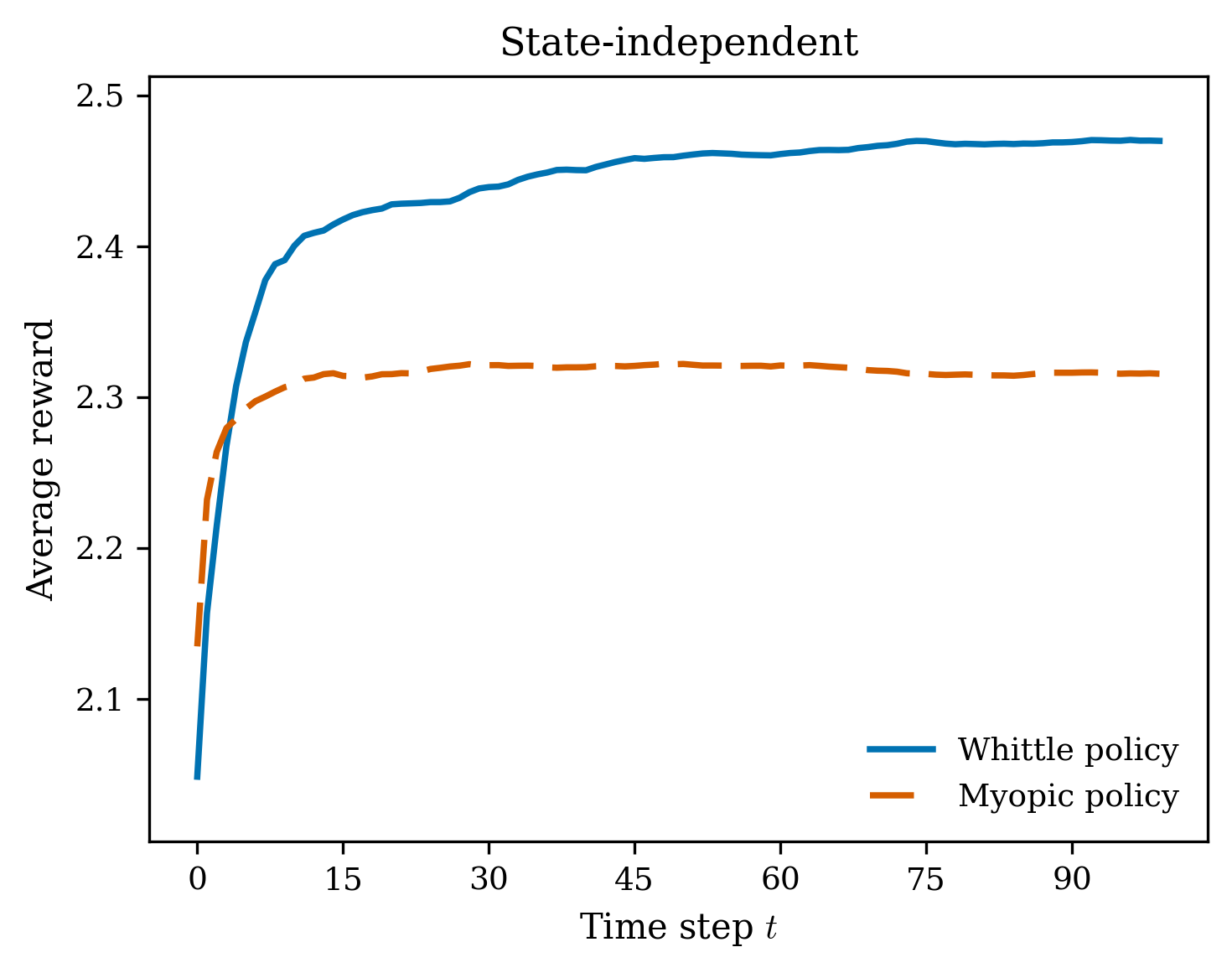}\\[0.4em]
  \includegraphics[width=0.48\textwidth]{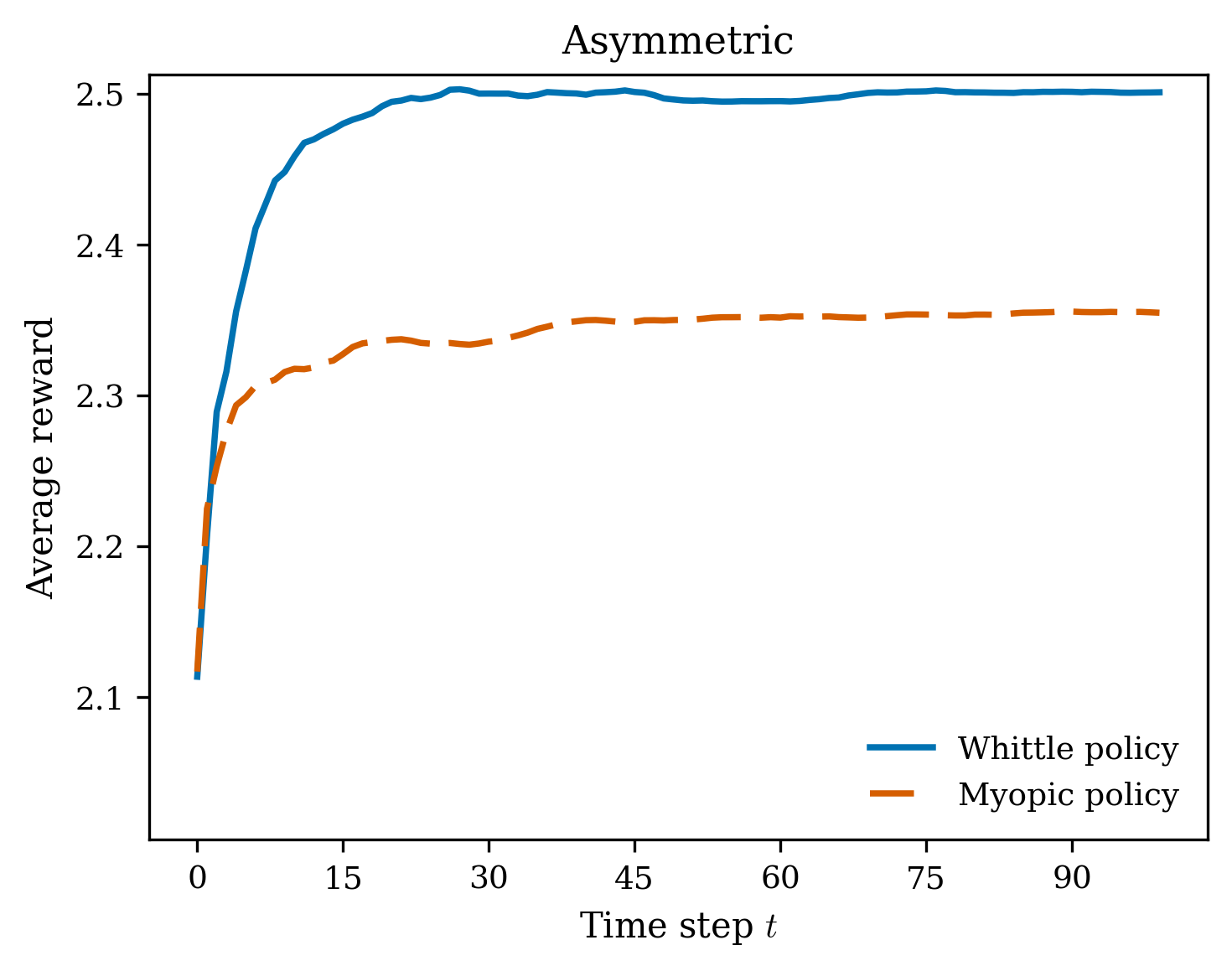}%
  \hfill
  \includegraphics[width=0.48\textwidth]{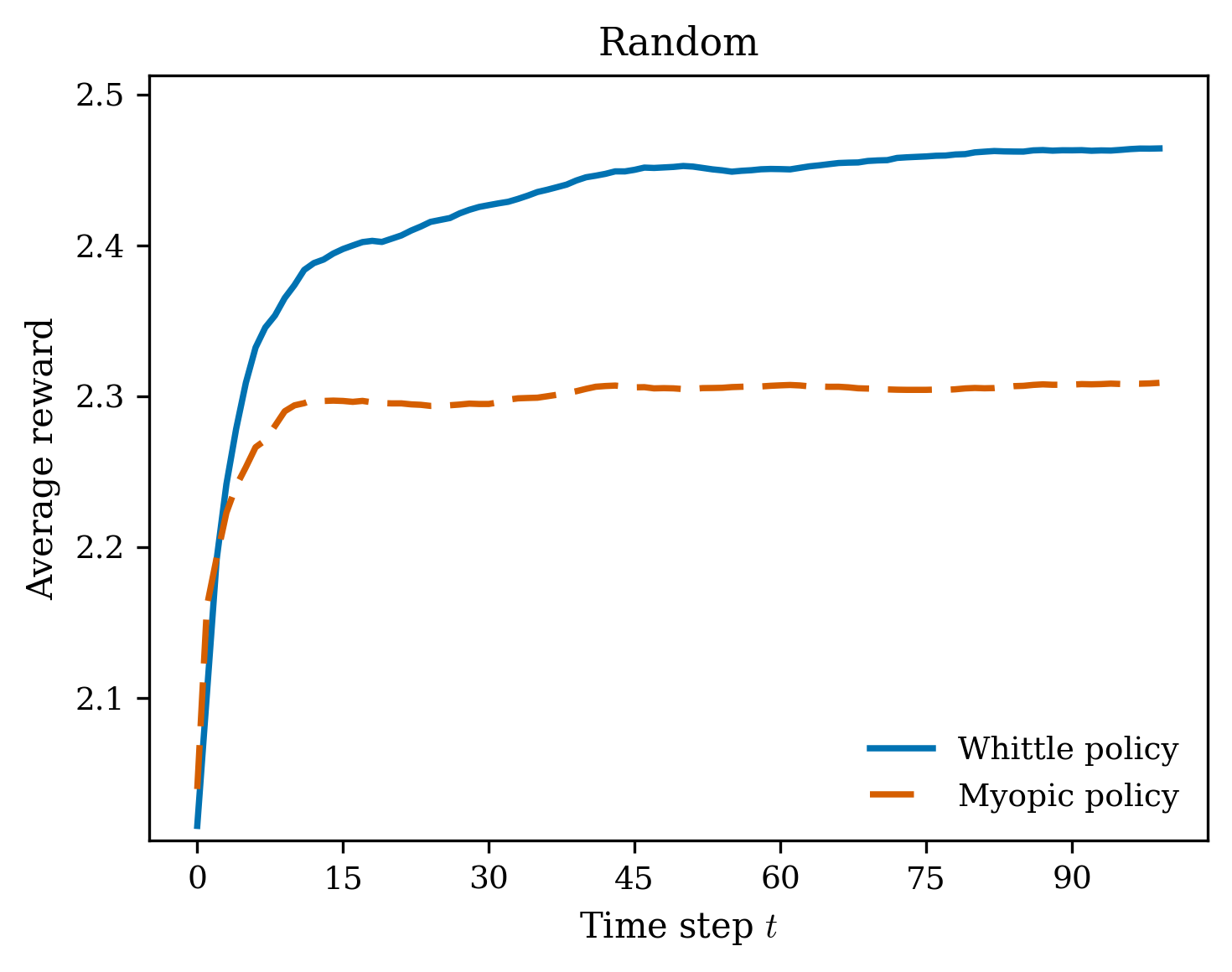}
  \caption{Whittle versus Myopic performance to four observation-error structures.}
  \label{fig:obs-noise-sensitivity}
\end{figure}

The index computation times are very similar across all four cases: the average time per arm ranges from $0.2278$\,s to $0.2332$\,s. In terms of policy performance, the Whittle index policy consistently dominates the myopic policy under all four observation models; Table~\ref{tab:obs-noise-sensitivity} reports the improvement and computation times. These experiments show that (i) the index computation cost is essentially unchanged across a wide range of observation-error structures, and (ii) the Whittle index policy performs better than the myopic policy under any type of observation matrices.

\begin{table}[h!]
  \centering
  \caption{Performance under different observation matrices}\label{tab:obs-noise-sensitivity}
  \begin{tabular}{lcccc}
    \toprule
    Error structure      & Improvement (\%) & Computation time (s) \\
    \midrule
    Symmetric           & 6.17 & 0.2332 \\
    State-independent    & 6.67 & 0.2317 \\
    Asymmetric          & 6.21 & 0.2290 \\
    Random              & 6.69 & 0.2278 \\
    \bottomrule
  \end{tabular}
\end{table}

Finally, we link the choice of $(T,\epsilon)$ to both the index error and the long-run performance, using a three-state, six-arm instance with fixed transition, observation, and reward matrices and $1000$ times Monte Carlo simulations.

\noindent\emph{(a) Index error versus $(T,\epsilon)$.}
We first fix the parameters $(\bar T,\bar\epsilon)$ and use the approximate Whittle index $\bar W = W(\bar T,\bar\epsilon)$ as an almost exact index, with
\[
\bar T = 50,\qquad \bar\epsilon = 10^{-6}.
\]
We then vary $T$ and $\epsilon$ and record the absolute error
$\bigl|W(T,\epsilon)-\bar W\bigr|$. The two plots are shown in Figure~\ref{fig:ablation}.

\begin{figure}[h!]
  \centering
  \includegraphics[width=0.48\textwidth]{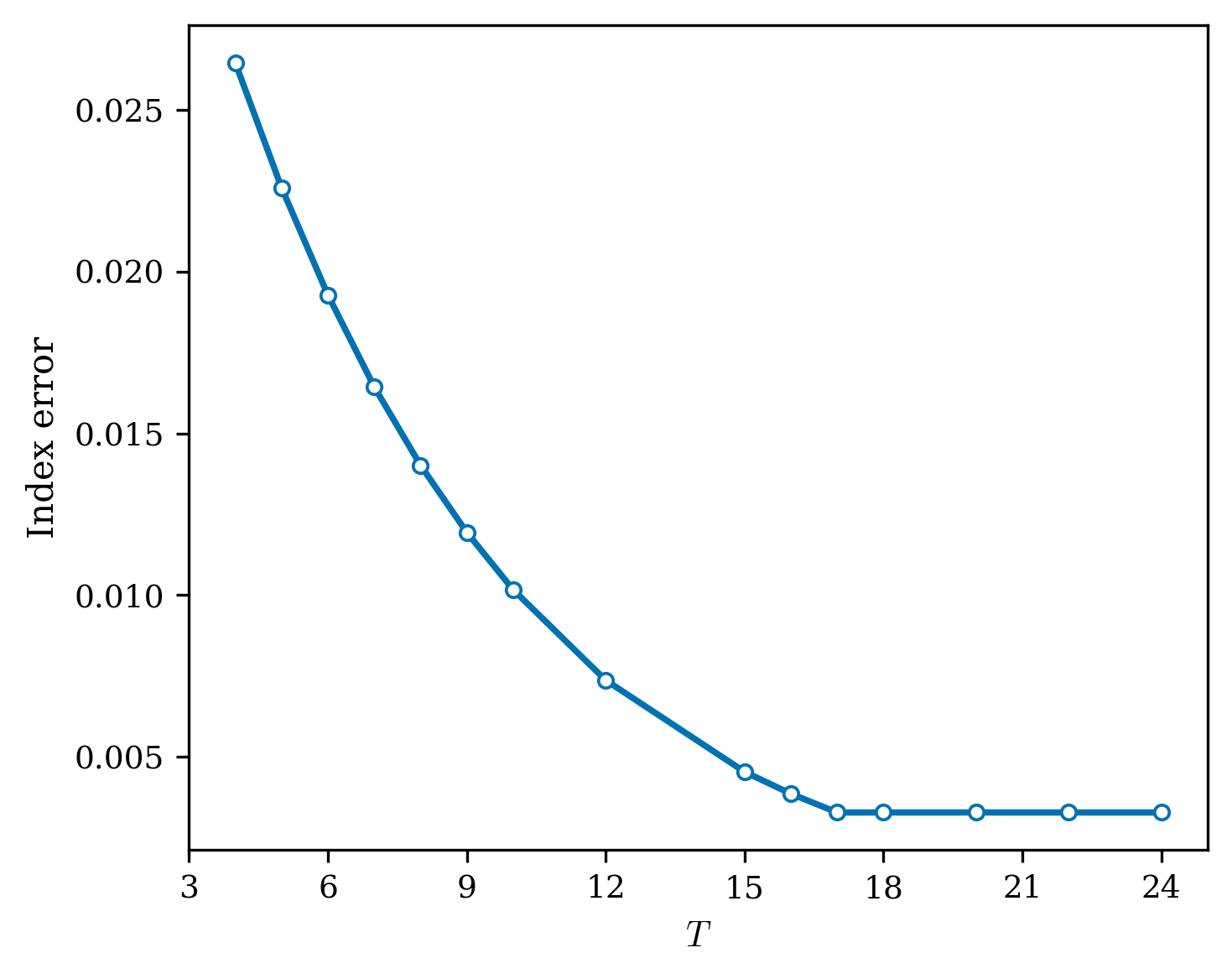}%
  \hfill
  \includegraphics[width=0.48\textwidth]{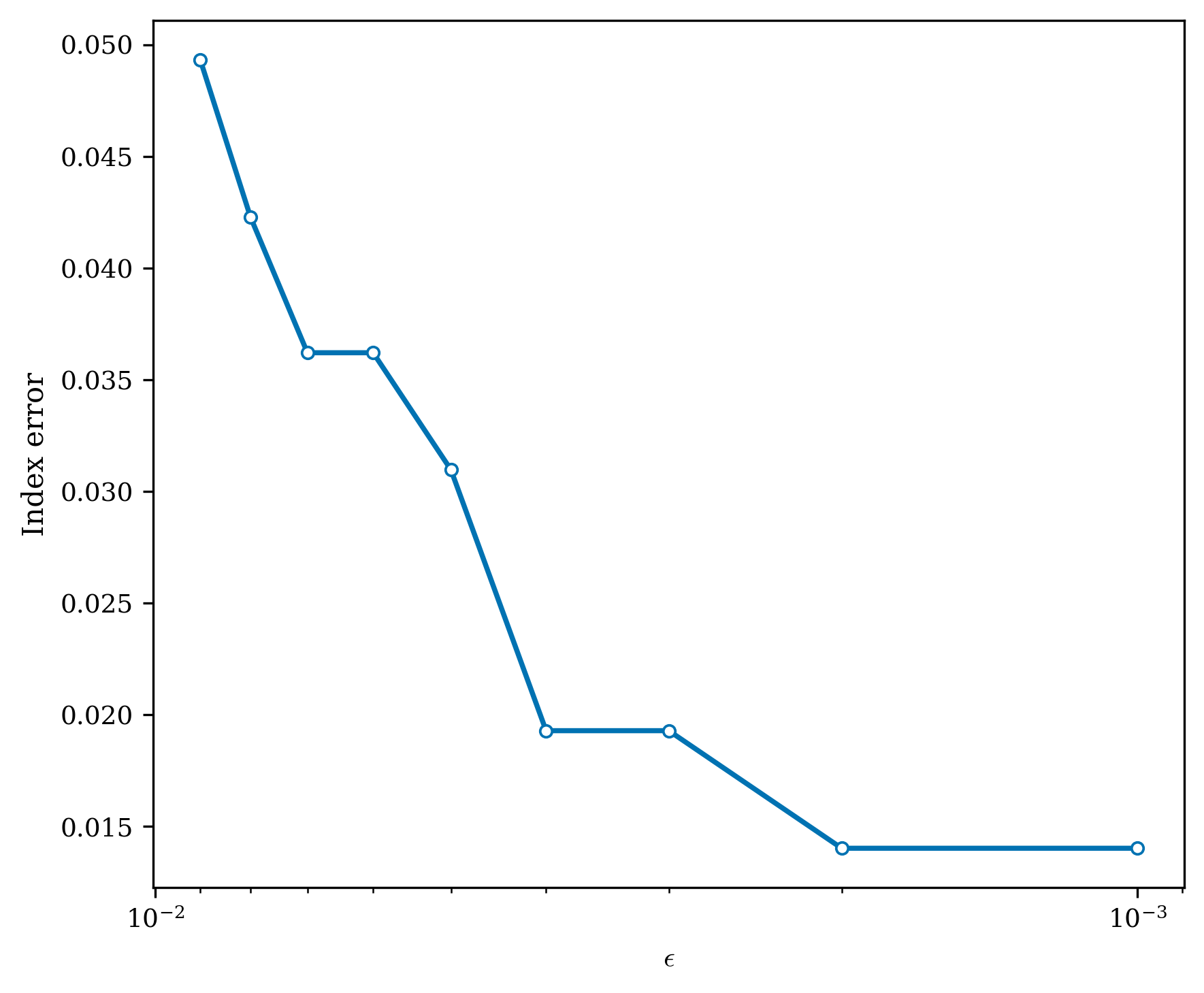}
  \caption{Left: index error versus $T$ (with $\epsilon=10^{-3}$). Right: index error versus $\epsilon$ (with $T=8$).}
  \label{fig:ablation}
\end{figure}

\begin{itemize}
  \item \emph{$T$:}
  The index error decreases from $2.65\times 10^{-2}$ at $T=4$ to $1.02\times 10^{-2}$ at $T=10$, and then levels off at $3.28\times 10^{-3}$ for all $T\ge 17$.

  \item \emph{$\epsilon$:}
  As we decrease $\epsilon$ from $10^{-2}$ to $2\times 10^{-3}$, the error drops from $4.93\times 10^{-2}$ to $1.93\times 10^{-2}$; further tightening $\epsilon$ to $10^{-3}$ and $5\times 10^{-4}$ only reduce the error slightly to about $1.40\times 10^{-2}$.
\end{itemize}

These results show that the $T$ is the main factor controlling the index error: once $T$ is moderately large, the additional gain is small.
Reducing $\epsilon$ from a coarse value ($\approx 10^{-2}$) to a moderately small value (at the order of $\times 10^{-3}$) removes most of the error, and further reductions have limited effect.
Since the error curve in $\epsilon$ already flattens out around $\epsilon= 10^{-3}$, we use this value in the more complex experiments and focus on the sensitivity to~$T$.

\noindent\emph{(b) Impact on policy performance.}
To check the impact of~$T$ on the policy performance, we keep the same three-state, six-arm instance and compare the long-run performance of the Whittle index policy (with different approximation depths~$T$) to the myopic policy, under the same observation and reward model and with $\epsilon= 10^{-3}$.
The results are shown in Figure~\ref{fig:depth-perf}.

\begin{figure}[h!]
  \centering
  \includegraphics[width=0.48\textwidth]{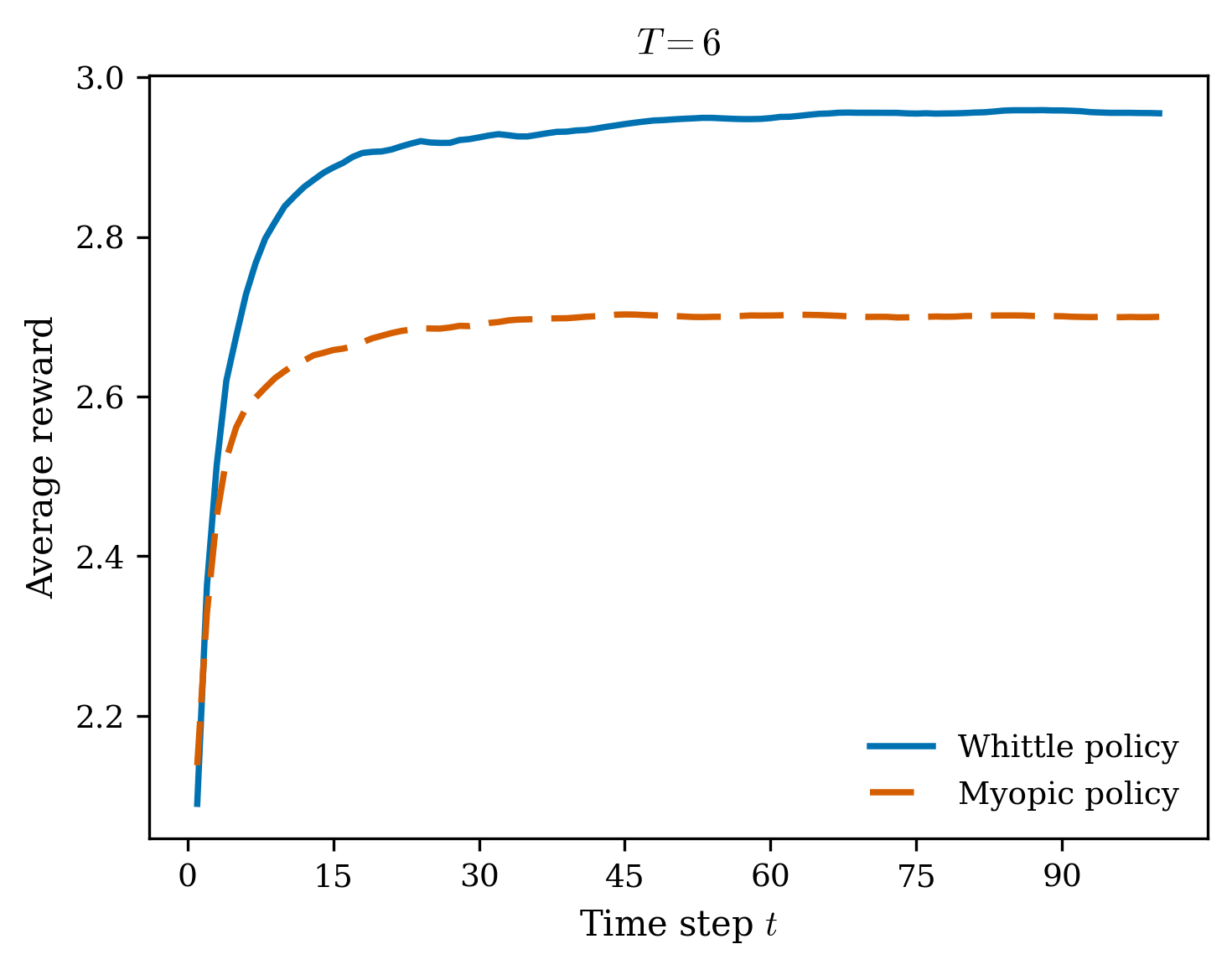}%
  \hfill
  \includegraphics[width=0.48\textwidth]{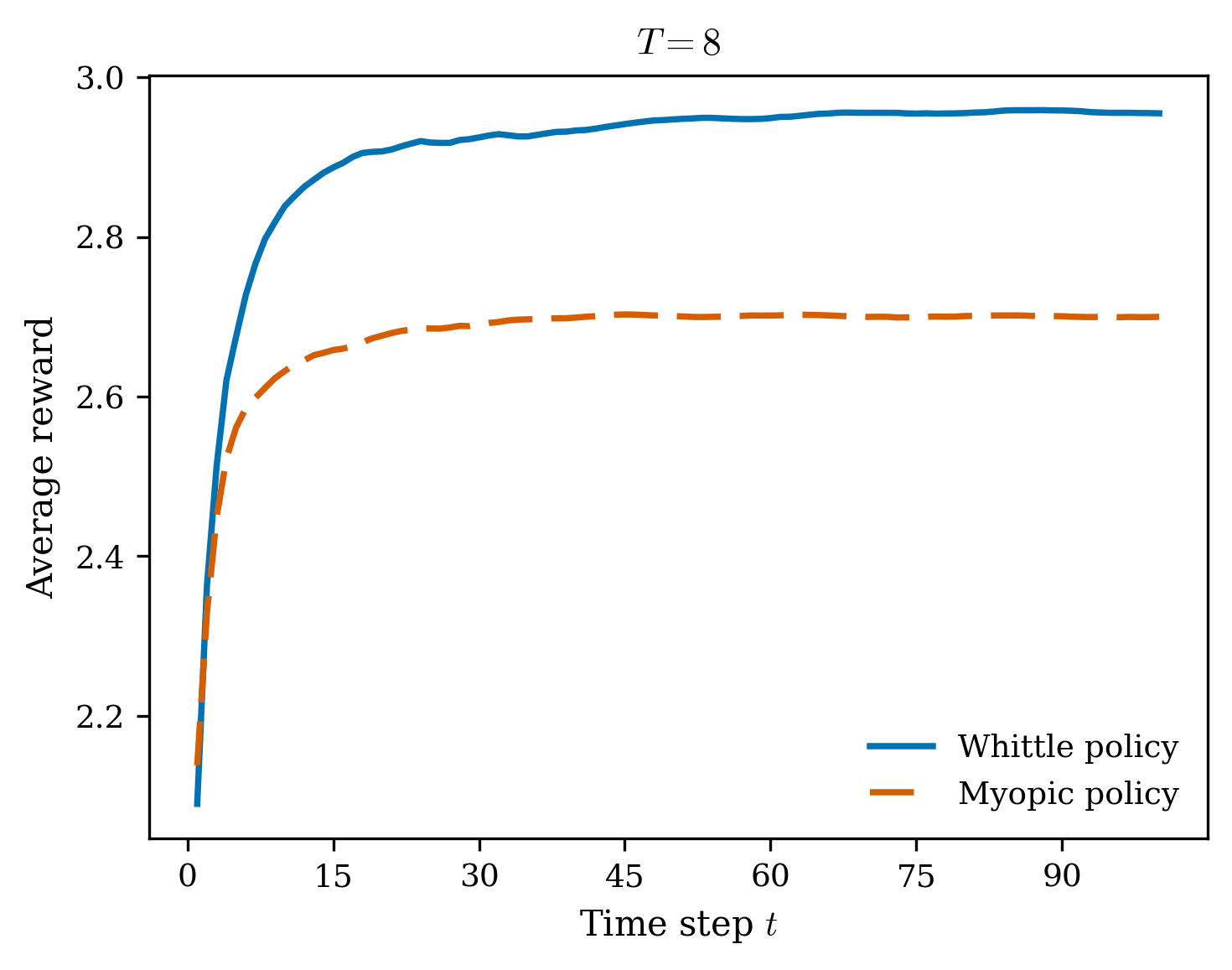}\\[0.4em]
  \includegraphics[width=0.48\textwidth]{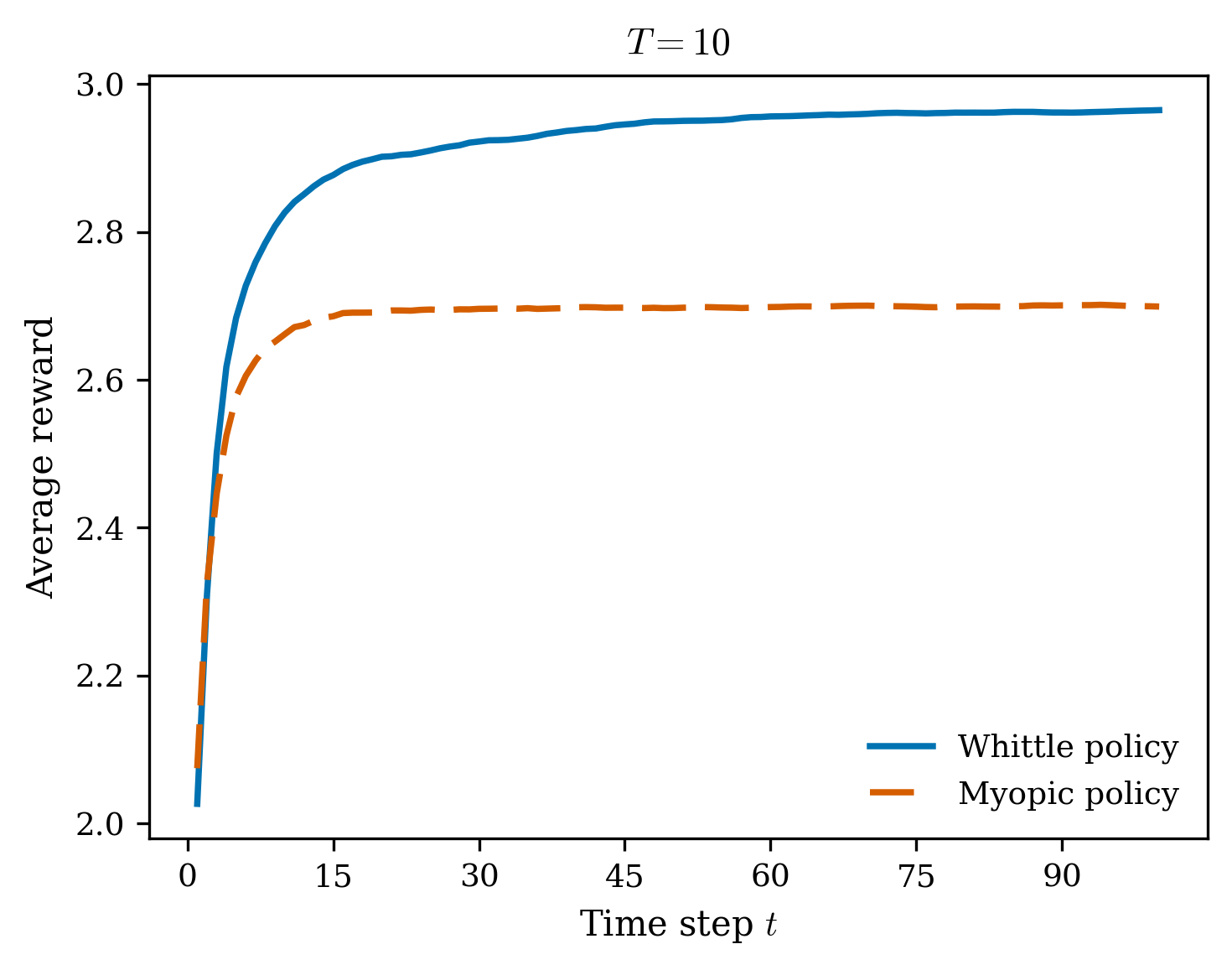}%
  \hfill
  \includegraphics[width=0.48\textwidth]{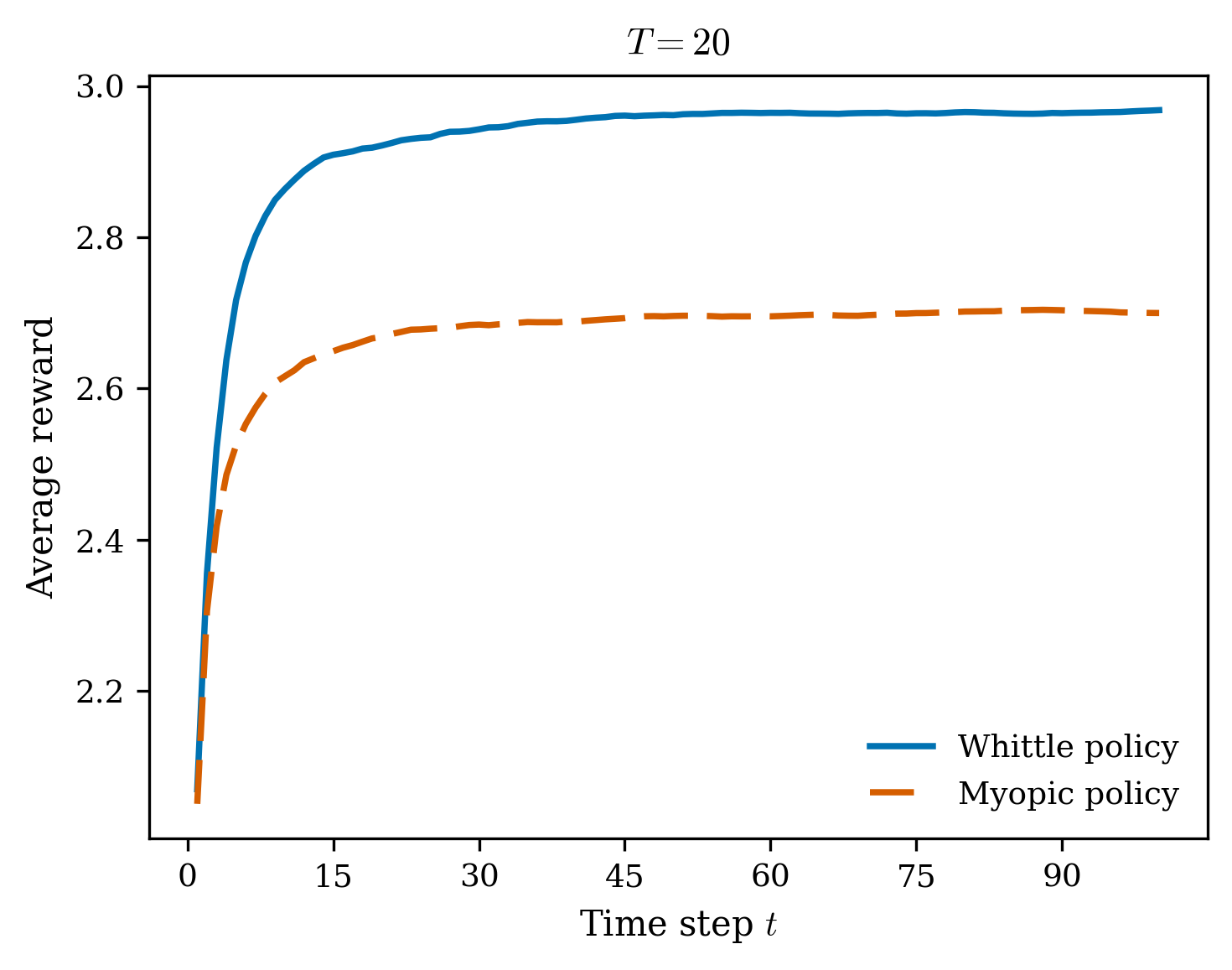}
  \caption{Performance of Whittle index and the myopic policies for different $T\in\{6,8,10,20\}$.}
  \label{fig:depth-perf}
\end{figure}

At the end of the simulation, the approximate Whittle index policy consistently dominates the myopic policy: the relative improvement in long-run average reward is about $9\%\text{--}10\%$, indicating that varying $T$ in this range has little effect on the final performance. Together with the index-error analysis above, the results indicate that setting $T$ around $6$ to $8$ achieves an effective balance between computational cost and policy performance for this problem instance.

\section{Conclusion and Future Work}

In this paper, we generalized the PCL framework for RMAB with finite state spaces to countable state spaces and propose an efficient algorithm to test the PCL-indexability and compute the Whittle index function for the partially observable RMAB with general observation models. As the number of belief states increases rapidly with time, the classic dynamic programming analysis for the RMAB problem suffers from the curse of dimensionality. Based on the achievable region and conservation law analysis studied in the literature for various RMAB problems, this paper is the first to extend the previous theoretical results in this direction to high-dimensional infinite-state scenarios by taking advantages of the theory of infinite-dimensional linear programming. Under a loose regularity assumption, we prove that the PCL-indexability is a sufficient condition for the Whittle indexability. Last, by introducing the concept of approximate state space, we build an algorithm to calculate the approximate Whittle index of each state using a finite-state adaptive greedy algorithm. Numerical results verified the strong performance of the Whittle index policy in our model. 

Although we have quite a few theoretical analyses and have confirmed the effectiveness of the Whittle index algorithm, some problems remain to be solved in the future. An important one is that we only analyzed the optimality of the Whittle index under the condition that the problem satisfies PCL-indexability with respect to a specific set family $\tilde{\Omega}$, but we have not yet explored other possible sets that might also allow PCL to function effectively. Furthermore, the strict calculation of Whittle indices for the entire belief state space remains an open problem (even though only their relative order is required for optimal policy implementation). Future work also includes the establishment of sufficient conditions for PCL-indexability and the extension of fluid analysis to the high-dimensional belief state space \citep{bertsimas2016fluid,verloop2016rmab}. In summary, there are still many opportunities for further exploration of problems with infinite state spaces.

\vspace{1em}

\noindent{\bf Acknowledgements} {We are grateful to Mr. Chengzhong Zhang for the help during the initial phase of this project and anonymous reviewers for the very valuable comments that helped improve this paper.}

\section{}\label{}


\bibliographystyle{cas-model2-names}

\bibliography{cas-refs}



\end{document}